\documentclass{article}

\usepackage{microtype}
\usepackage{graphicx}
\usepackage{subcaption}
\usepackage{booktabs} 

\usepackage{hyperref}

\usepackage[preprint]{icml2026}

\usepackage{amsmath}
\usepackage{amssymb}
\usepackage{mathtools}
\usepackage{amsthm}
\usepackage{makecell}
\usepackage{cancel}

\usepackage[capitalize,noabbrev]{cleveref}

\theoremstyle{plain}
\newtheorem{theorem}{Theorem}[section]

\newtheorem{lemma}[theorem]{Lemma}

\theoremstyle{definition}

\theoremstyle{remark}

\usepackage[textsize=tiny]{todonotes}

\icmltitlerunning{Incorruptible Neural Networks: Training Models that can Generalize to Large Internal Perturbations}

\begin{document}

\twocolumn[
  \icmltitle{Incorruptible Neural Networks: Training Models that can Generalize to Large Internal Perturbations}



  \icmlsetsymbol{equal}{*}

  \begin{icmlauthorlist}
    \icmlauthor{Philip Jacobson}{ca}
    \icmlauthor{Ben Feinberg}{abq}
    \icmlauthor{Suhas Kumar}{ca}
    \icmlauthor{Sapan Agarwal}{ca}
    \icmlauthor{T. Patrick Xiao}{equal,abq}
    \icmlauthor{Christopher Bennett}{equal,abq}
  \end{icmlauthorlist}

  \icmlaffiliation{ca}{Sandia National Laboratories, 7011 East Ave. Livermore, CA 94550, United States of America}
  \icmlaffiliation{abq}{Sandia National Laboratories, 1515 Eubank SE. Albuquerque, NM 87185, United States of America}

  \icmlcorrespondingauthor{Philip Jacobson}{pljacob@sandia.gov}

  \icmlkeywords{Machine Learning, ICML}

  \vskip 0.3in
]



\printAffiliationsAndNotice{\icmlEqualContribution}

\begin{abstract}
Flat regions of the neural network loss landscape have long been hypothesized to correlate with better generalization properties. A closely related but distinct problem is training models that are robust to internal perturbations to their weights, which may be an important need for future low-power hardware platforms.
In this paper, we explore the usage of two methods, sharpness-aware minimization (SAM) and random-weight perturbation (RWP), to find minima robust to a variety of random corruptions to weights. We consider the problem from two angles: generalization (how do we reduce the noise-robust generalization gap) and optimization (how do we maximize performance from optimizers when subject to strong perturbations). First, we establish, both theoretically and empirically, that an over-regularized RWP training objective is optimal for noise-robust generalization. For small-magnitude noise, we find that SAM's adversarial objective further improves performance over any RWP configuration, but performs poorly for large-magnitude noise. We link the cause of this to a vanishing-gradient effect, caused by unevenness in the loss landscape, affecting both SAM and RWP. Lastly, we demonstrate that dynamically adjusting the perturbation strength to match the evolution of the loss landscape improves optimizing for these perturbed objectives.
\end{abstract}

\section{Introduction}
Optimizing deep neural network models in order to locate flat minima in the loss landscape has long been a problem of interest to researchers, driven by the theory that flat minima correlate with better generalization on unseen data \cite{Keskar2016,DR17,jiang2019fantasticgeneralizationmeasures}. From this line of thought have emerged several modified methods for optimization that are explicitly designed to locate flat minima. The most prominent of these, sharpness-aware minimization (SAM), uses a one-step adversarial (i.e. \textit{worst-case}) perturbation along the direction of the gradient to ensure low loss over a finite volume of the loss landscape, and has been demonstrated to successfully improve generalization in a wide array of settings \cite{foret2021sharpnessaware}. A lesser-known alternative, known as random-weight perturbation (RWP), has likewise been studied as a potential mechanism for finding flat minima, but in contrast to SAM, has failed to show consistent generalization benefit \cite{bisla2022lowpassfilteringsgdrecovering,li2024revisiting}.

While flat minima are primarily studied to improve generalization, they are also relevant to finding minima that are robust to \textit{weight-space} perturbations. Although high-fidelity digital hardware has minimized this issue, the growing demand for compute has driven interest in analog neural network accelerators as energy-efficient alternatives to digital processors \cite{Sebastian2020}. Of particular concern to analog in-memory computing (AIMC) solutions are irreducible analog hardware errors, which induce accuracy-degrading effects. These errors arise from many sources, such as circuit noise, process variations, and component drift. This paper focuses particularly on errors that perturb the physical conductances of memory devices, which manifest numerically as random perturbations to the model's parameters. Previous analog computing literature has identified this problem and proposed solutions for specific hardware configurations, relying on empirically-measured device error profiles to induce robustness for an exact deployment scenario \cite{gokmen2019,Rasch2023}. However, these works all adopt the same assumption that the optimal training objective should exactly match the intended inference noise setting. We pose two natural questions in response: first, will optimizing toward the expected test-time noise distribution necessarily lead to the best-generalizing minimum? Secondly, we question whether or not random perturbations, in any form, are ideal for finding noise-robust minima, or if a more deliberate perturbation (as in the case of SAM) is in fact superior? 


We structure this paper as a study of the two key aspects underlying noise robustness: generalization, which we define as a model's ability to perform well on unseen data when exposed to weight noise, and optimization, specifically in the context of SAM and RWP. To understand noise-robust generalization, we derive a PAC-Bayes generalization bound that predicts the existence of tighter generalization bounds when training noise magnitude is \textit{larger} than test noise magnitude. Empirical experiments with RWP confirm that \textit{over-regularizing} the training objective consistently converges to better generalizing minima, consistent with the theoretical model. Extending our empirical studies to SAM, we observe that for small-strength noise, SAM finds minima with superior generalization to those of any RWP configuration. However, during training, large perturbations create a vanishing gradient effect, hindering the convergence of SAM, and to a lesser degree, RWP. We find that this effect can be mitigated through the use of dynamic perturbation schedules, which allow for training perturbations to better respect the evolving loss landscape.

In summary, our key contributions are as follows:

\begin{itemize}
    \item We derive a novel generalization bound predicting that over-regularized RWP converges to more noise-robust minima, and empirically show this to be true.
    \item We show that SAM converges to more noise-robust minima for small-strength noise, but suffers from a vanishing gradient effect for larger perturbations.
    \item We demonstrate that gradually ramping perturbation strength during training to match the evolution of the training trajectory effectively mitigates the vanishing-gradient effect, improving convergence of both methods.
    \item We validate our results using simulations of analog hardware.
\end{itemize}

\section{Related Work}
\subsection{Corruption-Robust Neural Networks}
Many previous works have studied the robustness of neural networks to input corruptions, either adversarial \cite{goodfellow2015explainingharnessingadversarialexamples,Mustafa_2019_ICCV,Yan2018} or random (common corruptions) \cite{rusak2020,fang2023,kar20223d,mintun2021,Guo_2023_CVPR,PRIME2022}. However, few works have previously considered the setting of random noise corruptions to model weights. In the realm of AIMC, several prior works have examined the resiliency of neural network inference accuracy to various sources of device noise \cite{yang2022,gokmen2019,kariyappa2021,Rasch2023,xiao2023}. The most comprehensive of these, Rasch et. al., applies regularization through injecting hardware noise emulating the target deployment platform during training. However, the methods these works propose are either application-specific or focus on hardware mitigations, and neglect to develop a general framework for understanding noise robustness.

\subsection{Sharpness-Aware Minimization}
SAM has recently emerged as a popular approach for finding flat, well-generalizing minima in the training landscape \cite{foret2021sharpnessaware}. In contrast to standard SGD-based optimization methods, SAM first perturbs the model using a gradient ascent step, approximating a local maximization, followed by a descent step from the perturbed point. Since its inception, an array of modified variants aiming to improve SAM's generalization ability \cite{RSAM,li2024friendly,kwon2021asam,kim2022fishersaminformationgeometry} or training efficiency \cite{du2022efficientsharpnessawareminimizationimproved,jiang2023adaptivepolicyemploysharpnessaware,liu2022efficientscalablesharpnessawareminimization,Mi2022,SAMPA} have been proposed. 

Several works have closely examined SAM's properties to better understand its empirically observed success. Andriushchenko et. al. studies the convergence of SAM through theoretically demonstrating the implicit biases of SAM when applied to a diagonal linear network \cite{andriushchenko2022understandingsharpnessawareminimization}. Wen et. al. more precisely demonstrates the measure of sharpness that the minibatch variation of SAM regularizes \cite{wen2023doessharpnessawareminimizationminimize}. Khanh et. al. provide a more comprehensive analysis of the convergence properties of SAM \cite{khanh2024}. Baek et. al. studies SAM's ability to induce robustness to label noise, demonstrating a positive effect from regularization of the network Jacobian \cite{baek2024samrobustlabelnoise}.

\subsection{Random Weight Perturbation}
Randomly perturbing weights during optimization has long been known as a straightforward approach to regularizing neural networks during training \cite{an1996}. Neelankantan et. al. demonstrate that adding noise to gradients during training improves generalization, achieving a similar effect to residual connections in deep networks \cite{neelakantan2015addinggradientnoiseimproves}. Zhou et. al. empirically show that weight noise can guide the optimization route out from spurious local minima \cite{zhou19d}. Bisla et. al. propose low-pass filtering the loss function (practically implemented through sampling weight noise from a Gaussian distribution) as a means to smoothing the loss landscape and hence improve generalization \cite{bisla2022lowpassfilteringsgdrecovering}. Li et. al. propose a modified form of RWP which introduces filter-wise perturbations based on the historical gradient \cite{li2024revisiting}. M\"{o}llenhoff et. al. establish a theoretical connection between SAM and RWP, in which SAM can be recovered through an optimal relaxation of the randomly perturbed objective \cite{möllenhoff2023samoptimalrelaxationbayes}. 

\section{Preliminaries}
Consider a training dataset $\mathcal{S} = \{(\mathbf{x_i},\mathbf{y_i})\}_{i=1}^n$ drawn i.i.d from a data distribution $\mathcal{D}$. Let $f(\mathbf{x},\mathbf{w})$ be a neural network model with trainable parameters $\mathbf{w} \in \mathbb{R}^k$. Given some loss function on individual samples $l(f(\mathbf{x_i},\mathbf{w}),\mathbf{y_i}) \in \mathbb{R}^+$, we define the empirical loss on the training dataset to be $L_{\mathcal{S}}(\mathbf{w}) = \frac{1}{n}\sum_{i=1}^n l(f(\mathbf{x_i},\mathbf{w}),\mathbf{y_i})$. Our goal is to train a model $f$ which minimizes the \textit{perturbed} distribution loss $L_{\mathcal{D}}(\mathbf{w}) = \mathbb{E}_{(\mathbf{x},\mathbf{y})\sim\mathcal{D},\mathbf{p}\sim\mathbb{Q}}[l(f(\mathbf{x},\mathbf{w + p}),\mathbf{y})]$, where $\mathbb{Q}$ is a known distribution of possible weight perturbations. Although in principle $\mathbb{Q}$ can be arbitrary, for simplicity's sake we restrict $\mathbb{Q}$ to be a zero-mean isotropic Gaussian scaled by the maximum magnitude element in a given weight filter $w$, defined as $\mathcal{N}(0,\max_j |w_j|\sigma_{test}^2\mathbb{I})$ with variance $\sigma_{test}^2$. We argue that the zero-mean Gaussian assumption faithfully captures the dynamics of analog weight errors. If the analog error distribution has a non-zero mean, it can be compensated for with a constant shift in the weights. Meanwhile, even if the error distribution is non-Gaussian, the weighted sum of inputs computed in each layer is nonetheless Gaussian by the Central Limit Theorem. We directly scale the noise variance by the weight magnitude to more closely emulate inference on analog in-memory computing processors, in which digital weights are mapped to the finite conductance range of memory devices. As a result, the noise applied to the weights is fixed to a given fraction of the weight magnitude, and cannot be mitigated through simple rescaling. 

\textbf{SAM} The modified SAM training loss can be expressed as follows:
\begin{equation}
\label{eq:SAM}
    L^{SAM}_{\mathcal{S}}(\mathbf{w}) = \max_{||\epsilon||_2^2 \leq \rho} L_{\mathcal{S}}(\mathbf{w} + \mathbf{\epsilon})  
\end{equation}

where $\rho$ is a scalar hyperparameter corresponding to the radius of the $l_2$ ball in which the maximization is performed. In practice, the inner maximization within the SAM objective is not tractable, and instead approximated using a single gradient ascent step of length $\rho$.

\textbf{RWP} The RWP training loss we adopt in this paper is described by the following:
\begin{equation}
    \label{eq:rwp}
    L^{RWP}_{\mathcal{S}}(\mathbf{w}) = \mathbb{E}_{\epsilon \sim \mathcal{N}(0,\max_j |w_j|\sigma_{train}^2\mathbb{I})} [L_{\mathcal{S}}(\mathbf{w} + \mathbf{\epsilon})]  
\end{equation}
where $\sigma_{train}^2$ is the variance of the perturbation distribution, left to our choosing. In our implementation, we sample one $\epsilon$ per minibatch, noting that the expectation is taken implicitly over the stochastic minibatches.

\textbf{Sharpness Measures} We define general sharpness as the difference in loss value between the original point on the optimization path and a given perturbed point. We use two varieties of this measure: \textit{ascent-direction} sharpness, in which the perturbed point lies along the gradient, and \textit{average-direction} sharpness, in which the perturbed point is sampled randomly. We also define \textit{gradient sharpness} to refer to the sharpness at the point of perturbation (where the update gradient is computed) for SAM/RWP. In the case of SAM, its gradient sharpness is equivalent to ascent-direction sharpness, whereas for RWP it is equivalent to average-direction sharpness. For the sake of practicality, we use the approximate measure of $m$-sharpness, in which the sharpness is calculated through averaging over minibatches of size $m$.



\textbf{Experimental Setting} We perform our experiments using the Cifar-100, Tiny-ImageNet and ImageNet-100 datasets \cite{krizhevsky2009learning,imagenet,Le2015TinyIV}; Cifar-100 results are presented in the main body, with additional results deferred to Appendix Sec. \ref{sec:perturbed_dynamics}, \ref{sec:tiny-imagenet}, \ref{sec:schedules}. All results in the main body are taken using a ResNet-18 backbone \cite{He2015}; additional experiments with more backbones are included in Appendix Sec. \ref{sec:backbones}. For each experimental result, we report two uncertainty measures: noise uncertainty (first number in tables), the standard deviation calculated across different samples of test-time noise, and weight uncertainty (second number in tables), the standard deviation calculated across different sets of model weights trained using a varying random seed. We average our results across 10 different weight-noise samples and 3 different trained model weights.

\section{Understanding Noise-Resilient Generalization}
\label{sec:rwp_vs_sam}
In this section, we study noise robustness from the generalization perspective, aiming to understand why certain minima generalize well to noise-robust solutions, while others do not. We establish, both theoretically and empirically, a set of properties to predict the generalization capabilities of a minimum based on its training objective.

\subsection{A Noisy Generalization Bound}
To understand how minima generalize when subject to Gaussian weight noise, we look to bound the model's generalization performance based on some measure of training sharpness. To do so, we consider the problem from a PAC-Bayes perspective \cite{McAllester1999}, and introduce the following generalization bound leveraging the isotropic average-direction loss:
\begin{theorem}
  \label{thm:mainbound}
  Assume $\Delta L_{\mathcal{D}} > 0$. For any small $\sigma_{train}$, $\sigma_{test}$ where $\sigma_{train} > \sigma_{test}$, the following holds with probability $1-\delta$:
  \begin{align}
  \label{eq:shortbound}
          \mathbb{E}_{\epsilon \sim \mathcal{N}(0,\sigma^2_{test})}[L_{\mathcal{D}}(w+\epsilon)] \leq \nonumber \\
          \mathbb{E}_{\epsilon \sim \mathcal{N}(0,\sigma^2_{train})}[L_{\mathcal{S}}(w+\epsilon)] + h\left(\frac{||w||^2}{\sigma_{train}^2}\right)
  \end{align}
  where $h: \mathbb{R}_+\rightarrow\mathbb{R}_+$ is a monotonically increasing function.
\end{theorem}
For the full details of the proof, see Appendix Sec. \ref{sec:pac_bayes_proof}. The key observation we take away from this result is that the two terms determining the tightness of the bound are anti-correlated with respect to $\sigma_{train}$: larger $\sigma_{train}$ increases the expected training loss, while decreasing the $h\left(||w||^2/\sigma_{train}^2\right)$ term. Therefore, choosing $\sigma_{train} = \sigma_{test}$ does \textit{not} guarantee the tightest bound; if 
\begin{align*}
    |\mathbb{E}_{\epsilon \sim \mathcal{N}(0,\sigma^2_{train})}[L_{\mathcal{S}}(w+\epsilon)]-\mathbb{E}_{\epsilon \sim \mathcal{N}(0,\sigma^2_{test})}[L_{\mathcal{S}}(w+\epsilon)]|\\  
    < |h\left(||w||^2/\sigma_{train}^2\right)-h\left(||w||^2/\sigma_{test}^2\right)|,
\end{align*}
then a tighter bound is achieved with $\sigma_{train} > \sigma_{test}$. Furthermore, this directly calls into question the utility of minimizing $\mathbb{E}_{\epsilon \sim \mathcal{N}(0,\sigma^2_{test})}[L_{\mathcal{S}}(w+\epsilon)]$: if $w^*$ is the global minimizer of $\mathbb{E}_{\epsilon \sim \mathcal{N}(0,\sigma^2_{test})}[L_{\mathcal{S}}(w+\epsilon)]$, there can exist some $w^\dag$ with tighter generalization bound on $L_{\mathcal{D}}$ if $w^\dag$ better minimizes $\mathbb{E}_{\epsilon \sim \mathcal{N}(0,\sigma^2_{train})}[L_{\mathcal{S}}(w+\epsilon)]$. This result carries strong implications for designing the perturbed training objective, which we show in the following section.

\subsection{Over-regularized Noisy Training Objectives are Optimal}
\label{sec:selecting_training_perturbation}

\begin{figure}[ht]
    \begin{subfigure}{1.0\columnwidth}
      \centering
      \includegraphics[width=0.65\linewidth]{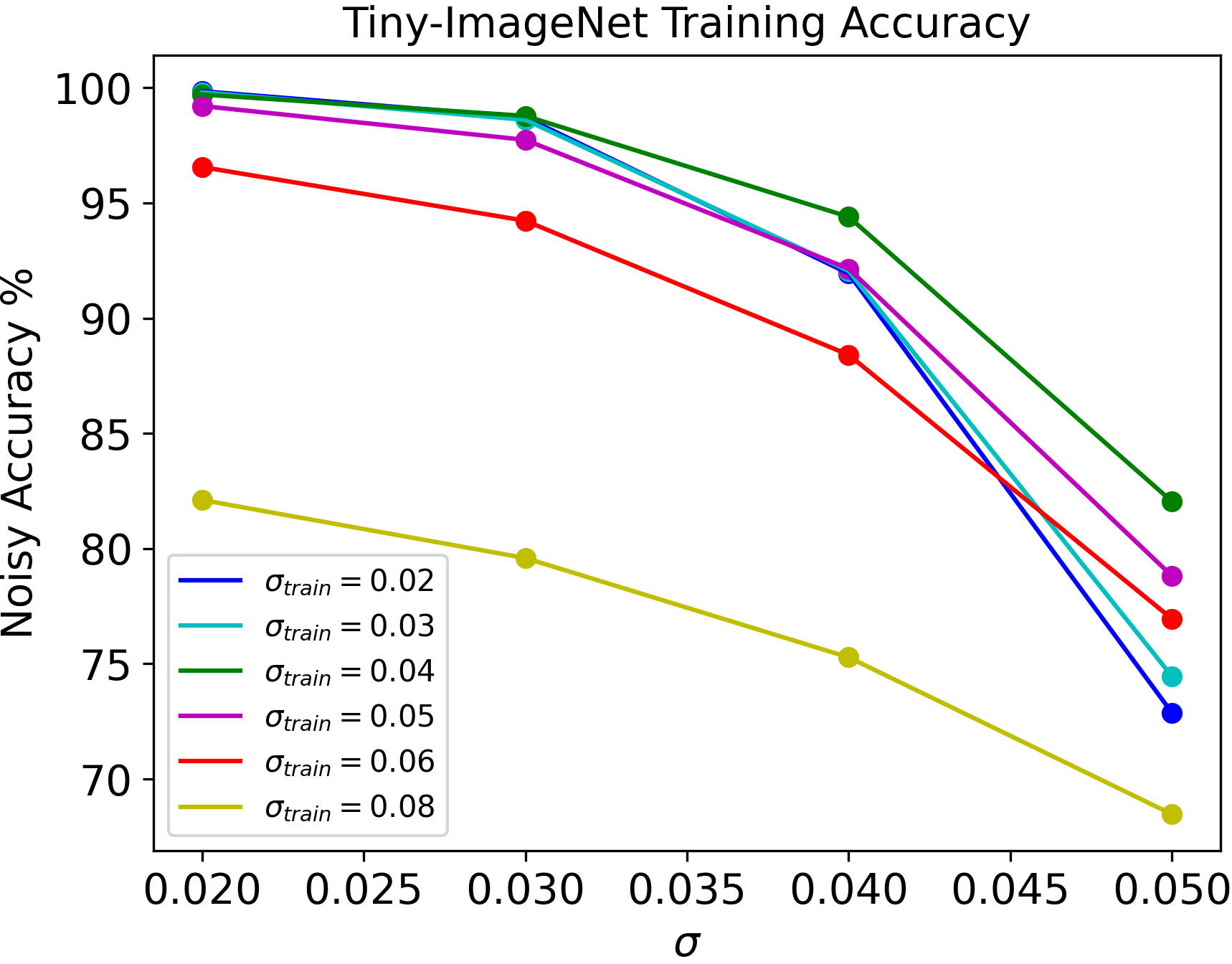}
      \caption{}
      \label{fig:tiny_imagenet_train_acc}
    \end{subfigure}%
    
    \begin{subfigure}{1.0\columnwidth}
      \centering
      \includegraphics[width=0.65\linewidth]{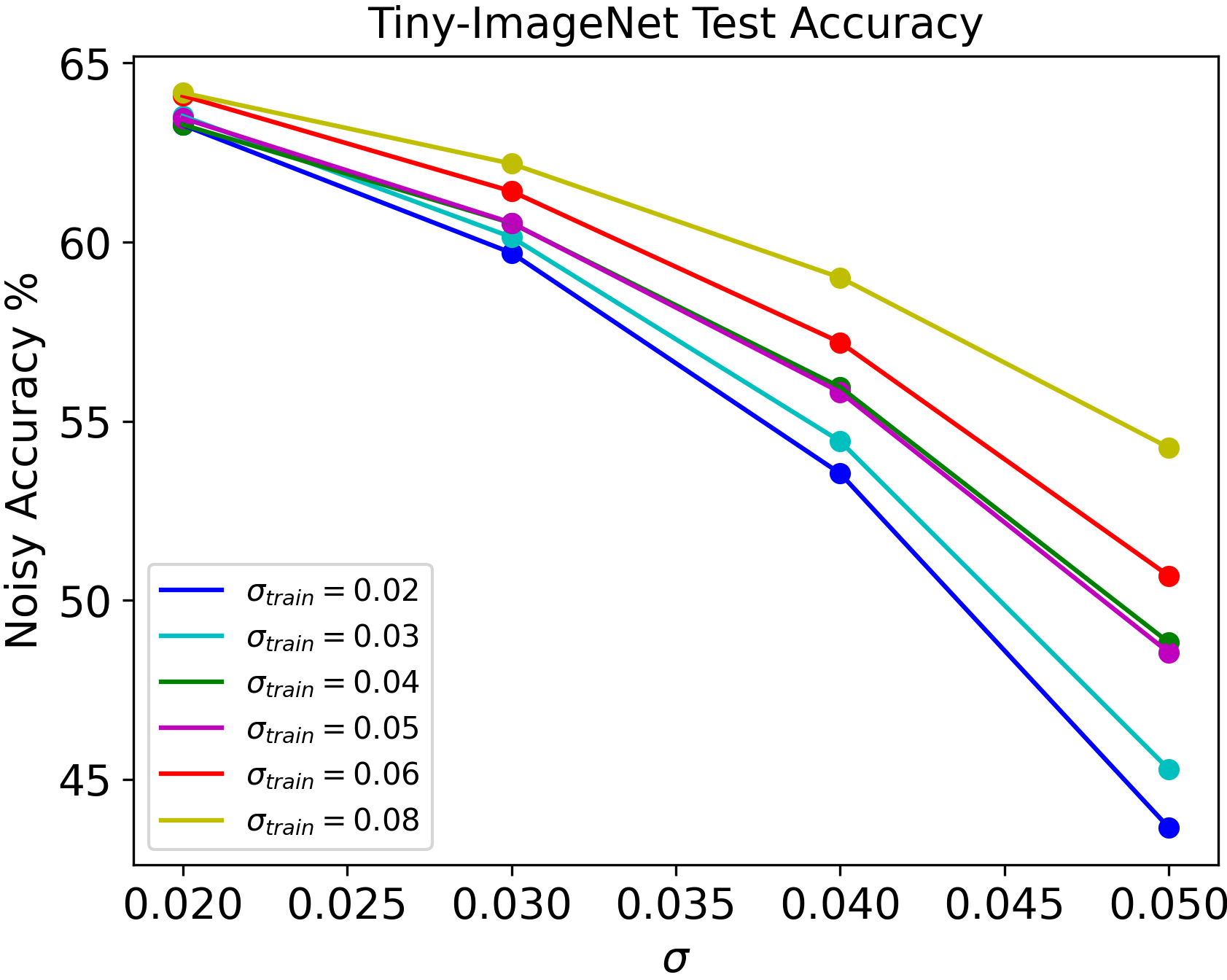}
      \caption{}
      \label{fig:tiny_imagenet_test_acc}
    \end{subfigure}
    \caption{Plot of noisy (a) training accuracy and (b) test accuracy as a function of the applied $\sigma$ for ResNet-18 trained on Tiny-ImageNet with RWP of varying $\sigma_{train}$. }
    \label{fig:tiny_imagenet_generalization}
\end{figure}

We start our empirical investigation of generalization using RWP, given its training objective that naturally aligns with the noise robustness objective. Our initial goal is to answer the question: given the objective of minimizing noisy test error with strength $\sigma_{test}$, what is the ideal training objective to minimize? Although the intuitive choice of matching training and generalization objectives (i.e. selecting $\sigma_{train} = \sigma_{test}$) has previously been accepted as ideal, Theorem \ref{thm:mainbound} provides evidence against this assertion. To test this hypothesis, we evaluate models trained using differing values of $\sigma_{train}$ on the validation dataset for selected values of $\sigma_{test}$, with our results shown in Tab. \ref{table:rwp_sigmas} for Cifar-100. We remark on two key patterns in these results: first, we see that \textit{under}-regularization (i.e. $\sigma_{train} < \sigma_{test}$) never outperforms $\sigma_{train} = \sigma_{test}$, consistent with our understanding from Theorem \ref{thm:mainbound} that asserts that the noisy test loss cannot be bounded by a noisy training loss with smaller $\sigma$. Second, we observe that for every $\sigma_{test}$ setting, the optimal performance is achieved when training with $\sigma_{train} > \sigma_{test}$, with the improvement in accuracy from over-regularization increasing as $\sigma_{test}$ increases. Again, this is consistent with our theoretical prediction that an over-regularized training objective can converge to a better-generalizing minimum than through simply matching noise strengths.

To isolate the effect of generalization from the effects of different training objectives on the optimization routine, we directly compare noisy test accuracy with noisy \textit{training} accuracy for a variety of $\sigma$ values, using minima found with differing $\sigma_{train}$ on Tiny-ImageNet (shown in Fig. \ref{fig:tiny_imagenet_generalization}). First, we observe that the minimum with lowest noisy training loss for strength $\sigma$ is generally found when the training objective is matched to that $\sigma$; in other words, over-regularization (and under-regularization) confer minimal benefit to optimizing for a particular noise objective. Meanwhile, when examining the plots, we clearly see that $\mathbb{E}_{\epsilon \sim \mathcal{N}(0,\sigma^2)}[L_{\mathcal{S}}(w+\epsilon)]$ is \textit{not} a good predictor of $\mathbb{E}_{\epsilon \sim \mathcal{N}(0,\sigma^2)}[L_{\mathcal{D}}(w+\epsilon)]$; most dramatically, $\sigma_{train} = 0.08$, while converging to the highest-loss minima across all the $\sigma$ settings, simultaneously is the best-generalizing model for all test $\sigma$. These observations provide clear evidence that our previous results arise specifically as a result of the generalization capabilities of different minima, as opposed to improvements to optimization.

\begin{table*}[ht]
  \caption{Comparison of RWP-trained ResNet-18 on Cifar-100 with varying perturbation strength at both training and test time. Staircase indicates boundary between $\sigma_{train} \leq \sigma_{test}$ and $\sigma_{train} > \sigma_{test}$.}
  \label{table:rwp_sigmas}
  \resizebox{\textwidth}{!}{\begin{tabular}{c|c|c|c|c|c}
    \hline
     & $\sigma_{test} = 0.0$ &  $\sigma_{test} = 0.02$  &   $\sigma_{test} = 0.03$ &  $\sigma_{test} = 0.05$ &  $\sigma_{test} = 0.07$\\
    \hline
    $\sigma_{train} = 0.0$ & $\mathbf{79.31} \pm 0.04$ & $77.24 \pm 0.51 \pm 0.23$ & $ 71.95 \pm 3.05 \pm 1.06 $ & $ 56.21 \pm 2.34 \pm 1.38$  & $ 47.09 \pm 2.54 \pm 3.64$\\
    \cline{2-2}
    $\sigma_{train} = 0.02$ & $78.86 \pm 0.41$ & $77.64 \pm 0.25 \pm 0.36$ & $75.77 \pm 0.43 \pm 0.33$ & $67.11 \pm 1.61 \pm 0.69$  & $ 50.12 \pm 2.59  \pm 2.89$\\
    \cline{3-3}
    $\sigma_{train} = 0.03$ & $78.96 \pm 0.28$ & $77.57 \pm 0.23 \pm 0.26$ & $75.82 \pm 0.47 \pm 0.30$ & $67.97 \pm 1.38 \pm 0.58$ & $ 52.57 \pm 2.43  \pm 0.99 $\\
    \cline{4-4}
    $\sigma_{train} = 0.05$ & $78.69 \pm 0.12$ &  $\mathbf{77.65} \pm 0.26 \pm 0.05$ & $\mathbf{76.26} \pm 0.47 \pm 0.11$ & $70.43 \pm 1.26 \pm 0.16$ & $ 56.55 \pm 3.29 \pm 0.35$\\
    \cline{5-5}
    $\sigma_{train} = 0.06$ & $78.04 \pm 0.39$ & $77.15 \pm 0.27 \pm 0.26$ & $75.97 \pm 0.42 \pm 0.20$ &  $\mathbf{71.04} \pm 0.91 \pm 0.15$ & $ 59.27 \pm 2.50  \pm  0.20$ \\
    $\sigma_{train} = 0.07$ & $77.17 \pm 0.34$ & $76.32 \pm 0.21 \pm 0.20$ & $75.18 \pm 0.37 \pm 0.23 $ & $ 70.51 \pm 1.04 \pm 0.25 $ & $ 59.39 \pm 2.64  \pm 0.10 $ \\
    \cline{6-6}
    $\sigma_{train} = 0.08$ & $77.05 \pm 0.32$ & $76.24 \pm 0.26 \pm 0.17$ & $75.15 \pm 0.45 \pm 0.18 $ & $70.96  \pm 1.20 \pm 0.45 $ & $\mathbf{61.36} \pm 2.85  \pm  0.90$ \\
    \hline
  \end{tabular}}
\end{table*}

\subsection{SAM Further Improves Small-noise Generalization}
\label{sec:sam_generalization}

\begin{figure}[ht]
  \centering
  \includegraphics[width=0.7\linewidth]{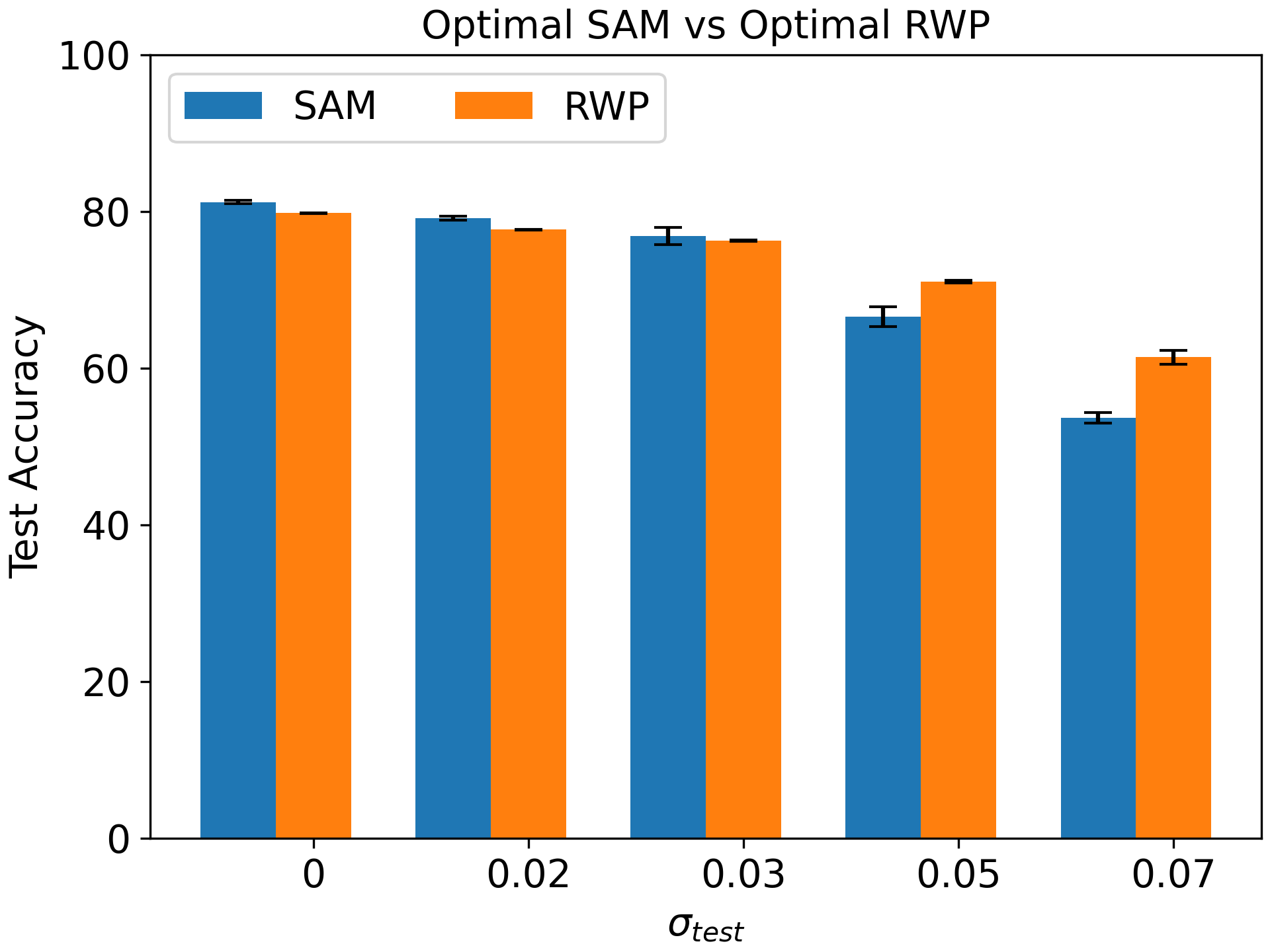}
  \caption{Comparison between optimal RWP and SAM across a variety of Cifar-100 noise settings. }
  \label{fig:sam_vs_rwp}
\end{figure}

Unlike the case of RWP, producing a meaningful generalization bound using the ascent-direction sharpness measure is not straightforward, and previous attempts are known to be insufficient in explaining SAM's generalization performance \cite{andriushchenko2022understandingsharpnessawareminimization}. Nonetheless, understanding the extent to which optimizing for SAM's objective can find noise-robust minima is still a critical question. To explore this, we evaluate SAM training with a variety of $\rho$ values as shown in Tab. \ref{table:sam_rhos} for Cifar-100. The first point we note is that SAM-trained models consistently outperform those trained with SGD on the perturbed accuracy metric, confirming that optimizing the ascent-direction objective does in fact confer benefits to noisy generalization. Directly-comparing the best performing models between SAM and RWP on Cifar-100 (see Fig. \ref{fig:sam_vs_rwp}), we see that for small $\sigma_{test}$, SAM produces better-generalizing minima than any of the corresponding RWP (or SGD) weights. As was done for RWP, we again compare noisy training/test accuracy on Tiny-ImageNet to quantify the generalization gap, shown in Tab. \ref{table:tiny_imagenet_generalization}. For the case of $\sigma=0.02$, we see that although SAM produces the best-generalizing model, the training minimum it converges to has a higher noisy training loss than all but the $\sigma=0.08$ RWP model. This demonstrates that SAM is not better suited toward optimizing the noisy training objective, but it in fact finds minima with even stronger generalization than those found by RWP.

Although SAM is clearly optimal for small $\sigma_{test}$, the generalization benefit is much reduced for greater noise strength, resulting in SAM underperforming RWP. Does this result directly stem from inherent generalization capabilities? In the case of RWP, we note that the best-generalizing value of $\sigma_{train}$ is positively correlated with $\sigma_{test}$. Intuitively, the same should hold true for SAM and $\rho$. However, a surprising trend emerges when ramping test-time noise: the optimal value of $\rho$ is consistent across all of the test-time noise settings. Notably, models trained on Cifar-100 with $\rho=0.8$, our most strongly perturbed SAM models, exhibit no performance advantage over those trained with SGD across any of the noise settings. We argue that this effect is not intrinsically tied to the generalization capabilities of the more-perturbed training objective, but instead the stability of the optimization procedure itself, which we show in the following section.

\begin{table*}
  \caption{Comparison of SAM-trained ResNet-18 on Cifar-100 with varying perturbation strength at both training and test time.}
  \label{table:sam_rhos}
  \resizebox{\textwidth}{!}{\begin{tabular}{c|c|c|c|c|c}
    \hline
     & $\sigma_{test} = 0.0$ &  $\sigma_{test} = 0.02$  &   $\sigma_{test} = 0.03$ &  $\sigma_{test} = 0.05$ &  $\sigma_{test} = 0.07$\\
    \hline
    $\rho = 0.0$ & $79.31 \pm 0.04$ & $77.24 \pm 0.51 \pm 0.23$ & $ 71.95 \pm 3.05 \pm 1.06 $ & $ 56.21 \pm 2.34 \pm 1.38$  & $ 47.09 \pm 2.54 \pm 3.64$  \\
    $\rho = 0.2$ & $\mathbf{81.15} \pm 0.21$ & $78.71 \pm 0.33 \pm 0.15$ & $75.27 \pm 1.89 \pm 1.71$ & $64.36 \pm 1.94 \pm 0.70$ &  $\mathbf{53.61} \pm 2.61 \pm 0.68$ \\
    $\rho = 0.3$ & $80.27 \pm 0.19$ & $\mathbf{79.12} \pm 0.23 \pm 0.28$ & $\mathbf{76.82} \pm 1.13 \pm 1.09$ & $\mathbf{66.51} \pm 2.57 \pm 1.27$ & $52.93 \pm 3.51 \pm 1.46$ \\
    $\rho = 0.5$ & $78.34 \pm 0.05$ & $77.31 \pm 0.25 \pm 0.23$ & $74.41 \pm 2.00 \pm 0.97$ & $63.27 \pm 2.57 \pm 1.81$ & $47.16 \pm 4.19 \pm 2.35$\\
    $\rho = 0.8$ & $72.34 \pm 0.37$ & $ 71.04 \pm 0.30 \pm 0.46$ & $ 66.98 \pm 1.67 \pm 2.44 $ &  $50.85 \pm 3.26 \pm 3.82$ &  $19.02 \pm 3.39\pm 2.54$ \\

    \hline
  \end{tabular}}
\end{table*}

\section{Optimization with Perturbed Training}
In Sec. \ref{sec:rwp_vs_sam}, we established a framework to understand the generalization of minima found from optimizing differing RWP/SAM training objectives. In this section, we examine the optimization procedure itself, and empirically identify key factors that determine its relative success.

\subsection{Observing the Evolution of Noise Robustness During Training}

\begin{figure*}[ht]
    \begin{subfigure}{0.33\textwidth}
      \centering
      \includegraphics[width=1.0\linewidth]{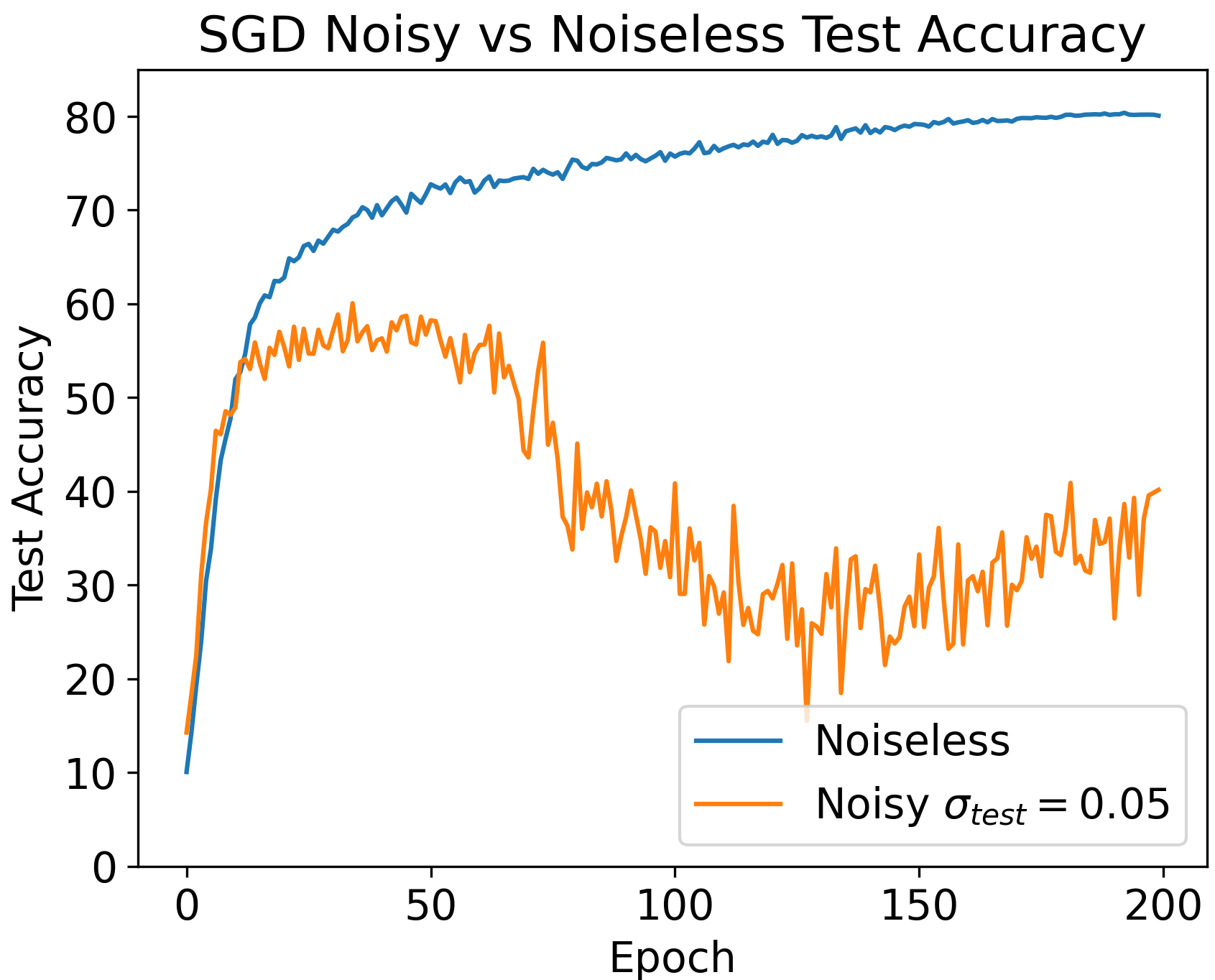}
      \caption{}
      \label{fig:noise_vs_noiseless}
    \end{subfigure}
    \begin{subfigure}{.33\textwidth}
      \centering
      \includegraphics[width=1.0\linewidth]{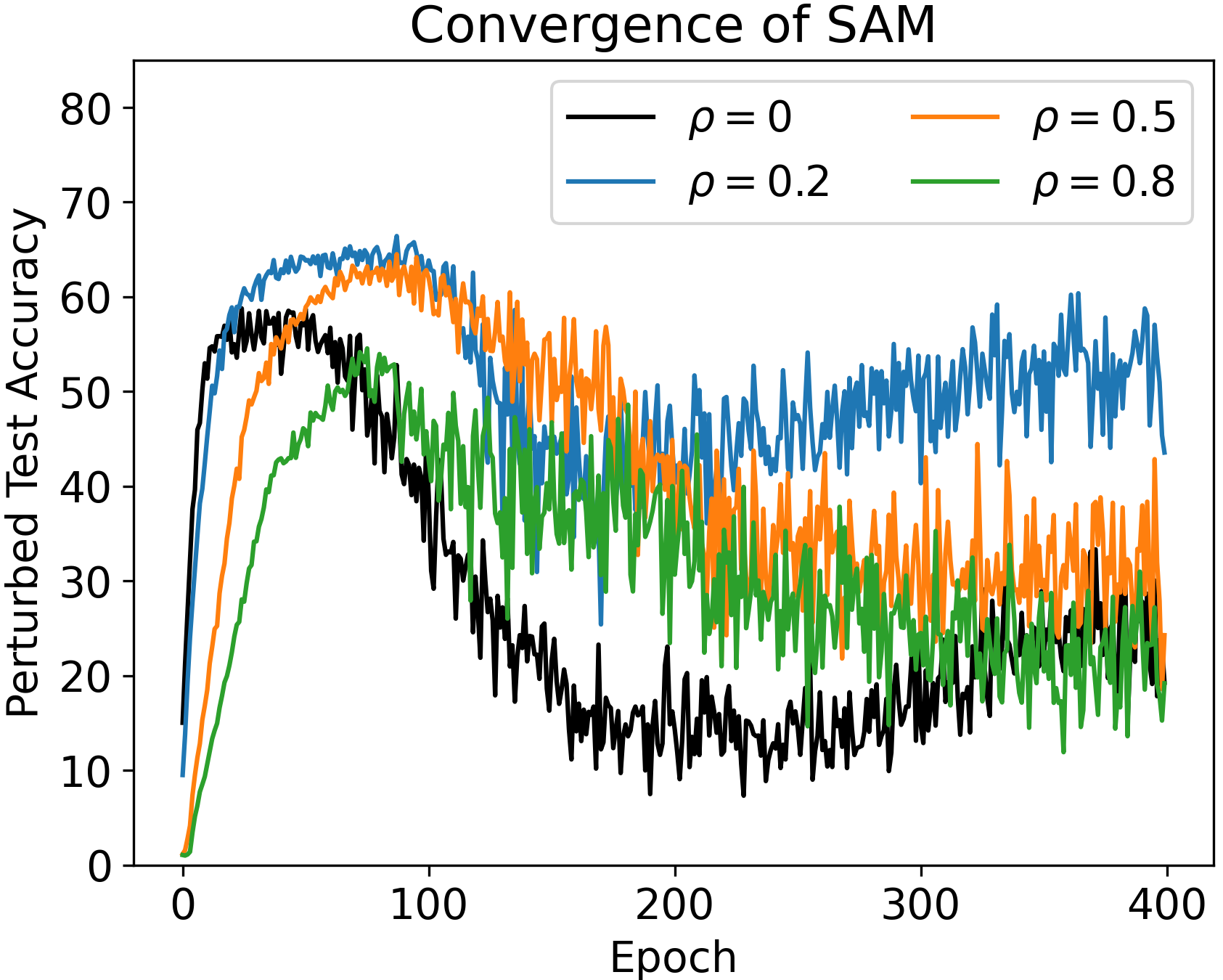}
      \caption{}
      \label{fig:sam_long}
    \end{subfigure}
    \begin{subfigure}{.33\textwidth}
      \centering
      \includegraphics[width=1.0\linewidth]{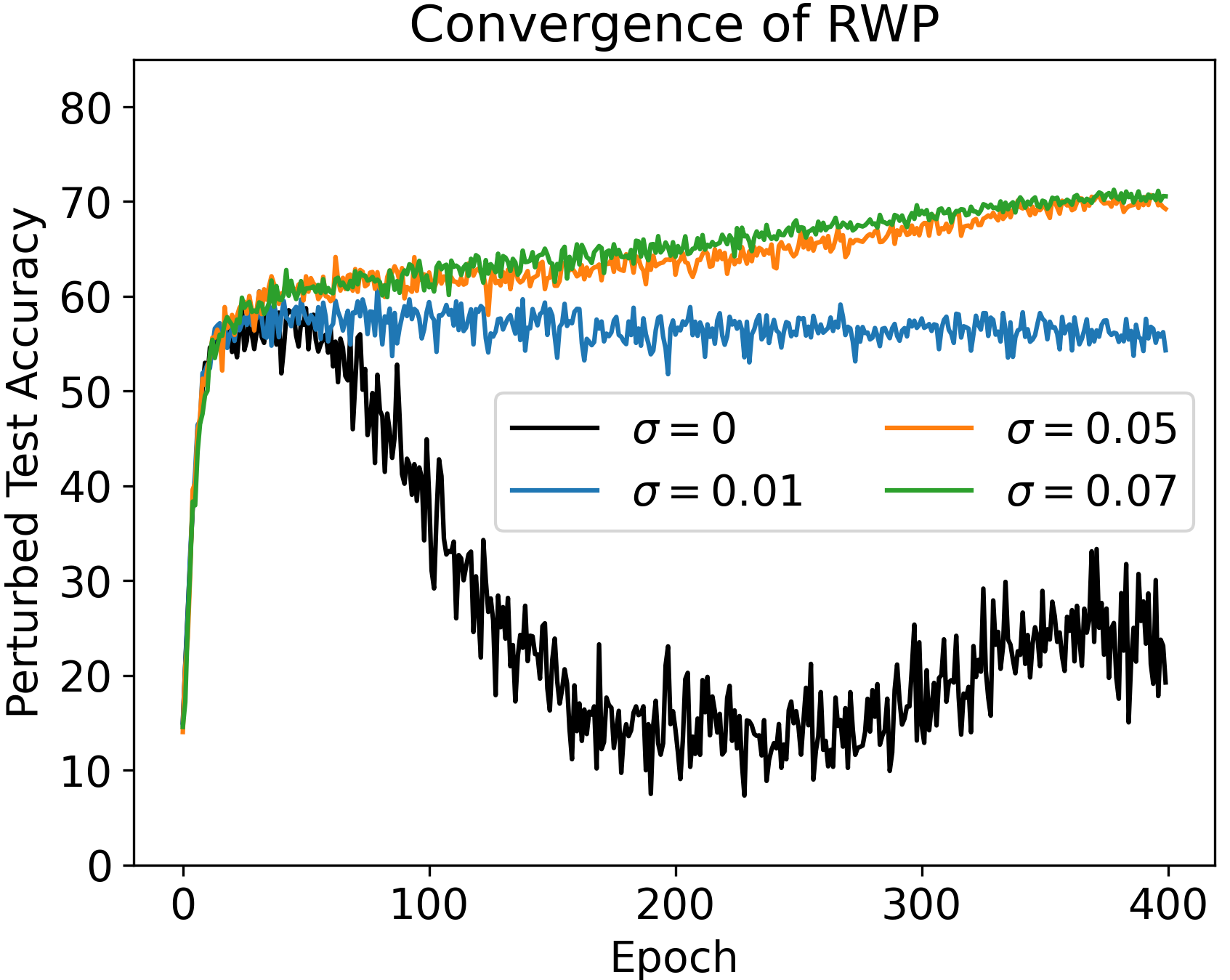}
      \caption{}
      \label{fig:rwp_long}
    \end{subfigure}
  \caption{(a) ResNet-18 test accuracy on Cifar-100 as a function of training epoch when both perturbations ($\sigma_{test}=0.05$) and no perturbations are applied. Training is conducted using SGD. (b) Comparison of perturbed test accuracy evolution for SAM with various $\rho$ values. (c) Comparison of perturbed test accuracy evolution for RWP with various $\sigma_{train}$ values.}
\end{figure*}

As we observe the training processes of the different optimizers, we note that the perturbed test loss does \textit{not} follow the same trends of convergence as the unperturbed loss. To illustrate this, we plot the test accuracy, both perturbed and unperturbed, of a ResNet-18 trained with SGD on Cifar-100 in Fig. \ref{fig:noise_vs_noiseless}. Whereas the noiseless test accuracy continually increases until convergence, the perturbed test accuracy peaks at an earlier epoch, before gradually declining as the optimization continues. Intuitively, this matches with the hypothesis that in early training iterations, the loss landscape is naturally flat, and as such the perturbed test accuracy improves over the early epochs. However, without proper regularization, training will naturally begin to overfit as a sharper minimum is approached.

Next, we compare this same convergence dynamic for both SAM and RWP, varying perturbation strengths $\rho$ and $\sigma$, respectively. We visualize this comparison in Fig. \ref{fig:sam_long} and Fig. \ref{fig:rwp_long}. In the models trained with SAM, we notice a similar dynamic to those trained with SGD: when trained for a sufficiently long schedule, the models undergo the same phenomenon of reaching a peak perturbed test accuracy, after which overfitting ensues. However, as $\rho$ is increased, two trends emerge: the epoch at which the peak perturbed accuracy is achieved at is pushed further back (i.e. the model can be trained closer to convergence before the peak is reached), and the perturbed accuracy at the peak is increased relative to that of SGD. After $\rho$ exceeds a critical value, these trends reverse: overfitting once again initiates at an earlier training epoch, and the peak accuracy is now reduced.

Meanwhile, the training dynamics observed in models trained with RWP stand in contrast to those trained with SAM and SGD. We consider two regimes: small $\sigma_{train}$, and large $\sigma_{train}$. In the former case, the optimization follows the same trend as in SGD and SAM: perturbed test accuracy peaks before convergence, after which further training leads to a decline in noise-resilience. However, in the latter case, the perturbed test accuracy does not peak, instead gradually increasing until the model converges. Even when the model is trained for a greater number of iterations beyond convergence, the perturbed test accuracy simply plateaus, and the model never enters an overfit regime. This stands in particular contrast to SAM, which \textit{cannot} prevent the convergence of the model to an overfit solution, regardless of the selected perturbation radius.

\subsection{Drawback of Large Perturbations: Flatness-Induced Vanishing Gradient}
\label{sec:vanishing_gradient}
From the results in Sec. \ref{sec:sam_generalization}, we have identified that, particularly in the case of SAM, large perturbations lead to convergence issues. First, we quantify this effect: we plot the $l_2$-norm of the update gradient $||\nabla L ||_2$ (i.e. the gradient at the perturbed point) as a function of training epoch (averaged over minibatches) for both techniques, shown in Fig. \ref{fig:grad_norm_sam_v_rwp} for Cifar-100. For fair comparison, we select multiple values for $\sigma$ and $\rho$, capturing the range of perturbation strengths evaluated in Tab. \ref{table:rwp_sigmas} and Tab. \ref{table:sam_rhos}, respectively. In early epochs, we see the norm of the gradient increase as the training path enters into a steep valley within the loss landscape. Later, as training begins to converge to a local minimum, the gradient norm decreases again. Relative to SGD, all of the perturbative training methods reduce the gradient norm across the entire optimization path. However, we notice a stark contrast: the SAM-trained models produce gradient norms that are \textit{significantly} smaller in magnitude than those of the RWP training. These small gradients result in a markedly degraded optimization, in which training fails to progress meaningfully beyond the initialized state.  

Next, we develop an understanding of why SAM gradients experience this drastic attenuation. To do so, we visualize the gradient sharpness during training to understand the loss topography at the perturbed point, shown in Fig. \ref{fig:sharpness_sam_v_rwp}. From this, we see that the sharpness encountered by the small SAM perturbations is much larger in magnitude than that encountered by the large RWP perturbations, a trend that holds true at every point within the loss landscape. From this, we draw a logical conclusion: sharper perturbations more-rapidly guide the optimization toward flatter regions of the loss landscape, which are simultaneously regions with smaller gradients. We visualize this comparison between SAM and RWP landscapes directly in Fig. \ref{fig:sam_vs_rwp_landscape}. This establishes a trade-off between finding flat vs. deep regions of the loss landscape; in the extreme case (e.g. SAM with large $\rho$), the training trajectory quickly enters a flat, small-gradient region of the landscape, causing training to stagnate before finding a sufficiently low-loss minimum.  We dub this effect the \textit{flatness-induced vanishing gradient}.

\begin{figure*}
    \begin{subfigure}{.33\textwidth}
      \centering
      \includegraphics[width=1.0\linewidth]{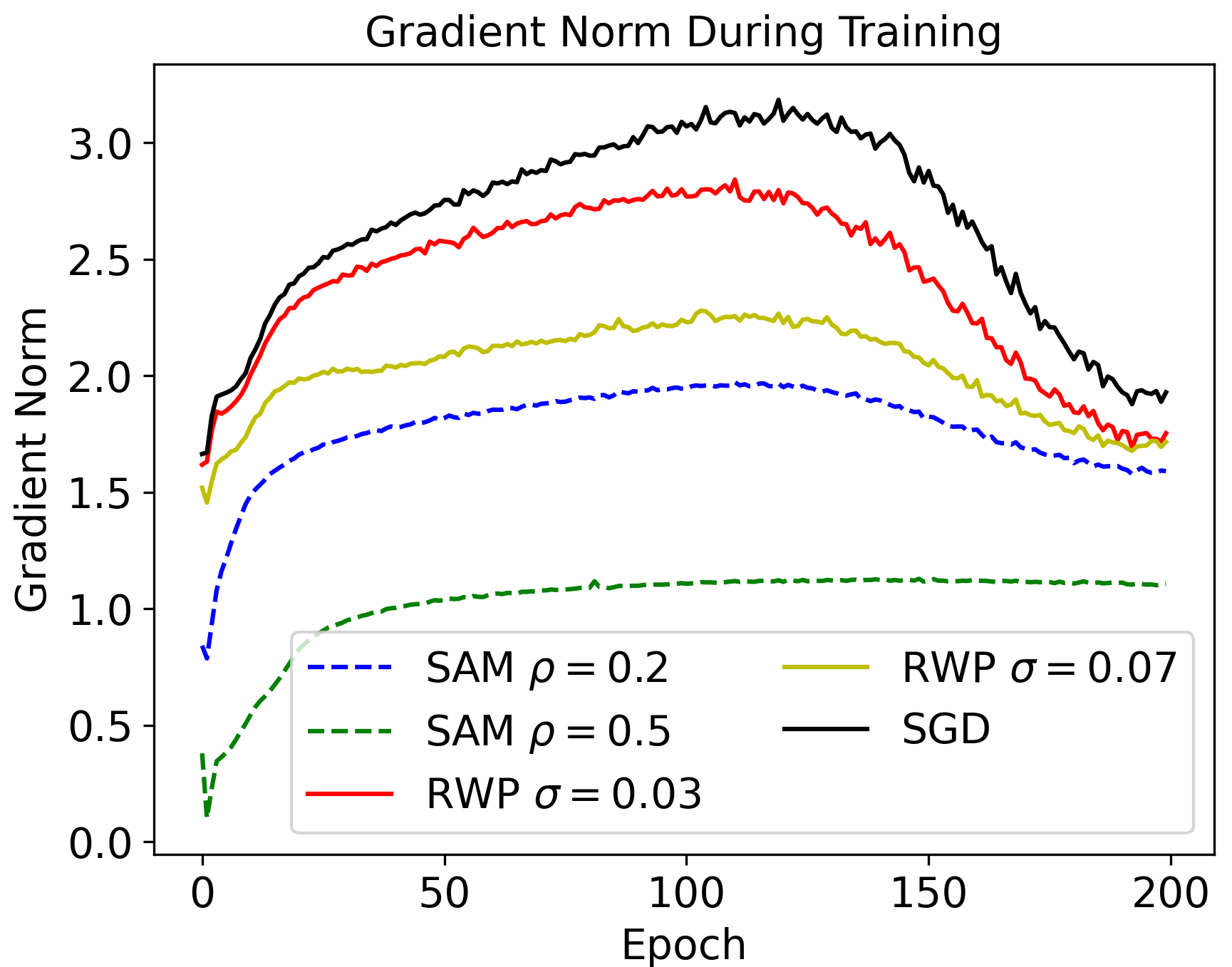}
      \caption{}
      \label{fig:grad_norm_sam_v_rwp}
    \end{subfigure}
    \begin{subfigure}{.33\textwidth}
        \centering
        \includegraphics[width=1.0\linewidth]{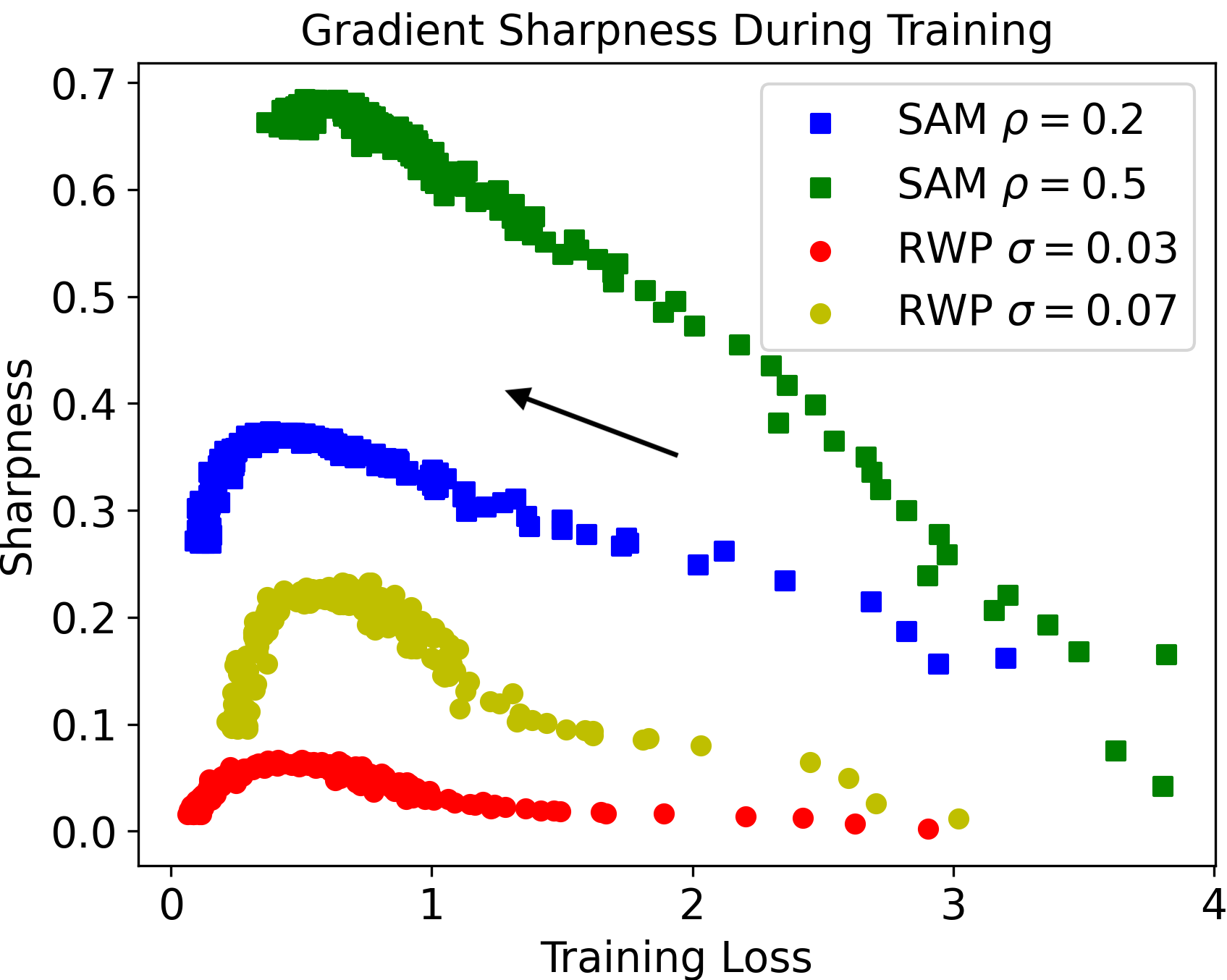}
        \caption{}
        \label{fig:sharpness_sam_v_rwp}
    \end{subfigure}%
    \begin{subfigure}{.33\textwidth}
      \centering
      \includegraphics[width=0.75\linewidth]{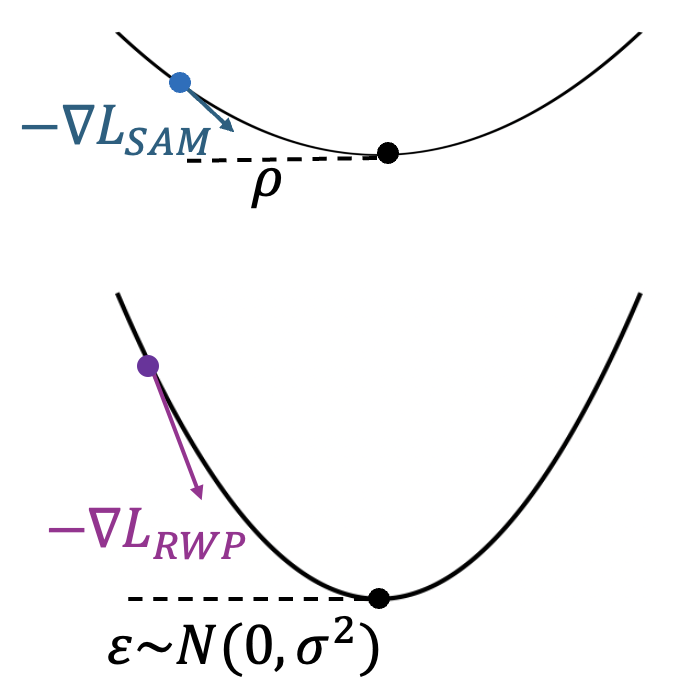}
      \caption{}
      \label{fig:sam_vs_rwp_landscape}
    \end{subfigure}
    \caption{(a) Plot of the update gradient norm $||\nabla L ||_2$ as a function of training epoch for a ResNet-18 trained on Cifar-100 using SGD, SAM, and RWP. (b) Plot of the gradient sharpness (corresponding to ascent-direction for SAM or average-direction for RWP) of both SAM and RWP as a function of training loss. (c) Schematic visualization of loss surfaces for large-$\rho$ SAM's (top) and large-$\sigma$ RWP's (bottom) training trajectories.}
    \label{fig:understanding_loss_landscape}
\end{figure*}

\subsection{Adjusting Perturbations to the Evolving Loss Landscape}

Based on our results in Sec. \ref{sec:vanishing_gradient}, we posit that the vanishing gradient effect is most detrimental in the early training epochs. We base this on two justifications: first, Fig. \ref{fig:grad_norm_sam_v_rwp} shows that the gradient magnitude is consistently the smallest in the early training epochs. Second, on account of the decaying learning rate, the largest weight updates occur in the earliest training epochs- hence, these early iterations play a disproportionately large role in the optimization trajectory, ultimately altering convergence in the later epochs. Thus, we hypothesize that the strength of perturbation tolerable to optimization should increase as training proceeds, suggesting that the optimal perturbation strength selected at the initial stage of training should not remain constant. Based on this insight, we explore the use of modified perturbation schedules for SAM and RWP, engineered to match the evolution of the dynamic loss landscape.

From our observations, we know that perturbation magnitude should continually increase throughout training, although the exact manner is not clear. As such, we explore two functional classes of perturbation schedules, linear and quadratic, each initialized at strength 0 and increasing to a terminal value over the course of a warm-up cycle. The terminal perturbation strength and warm-up cycle length are treated as hyperparameters that we optimize. We include a detailed description of the quadratic schedule for RWP in Algorithm \ref{alg:rwp-quadratic}.

\begin{table*}[ht]
  \caption{Comparison of our three proposed perturbation schedules for SAM/RWP on ResNet-18 trained on Cifar-100. The maximum perturbation strength/schedule length are determined through using a grid search. }
  \label{table:noise_schedule_comp}
  \resizebox{\textwidth}{!}{\begin{tabular}{c|c|c|c|c}
    \hline
     &  $\sigma_{test} = 0.02$  &   $\sigma_{test} = 0.03$ &  $\sigma_{test} = 0.05$ &  $\sigma_{test} = 0.07$\\
    \hline
    Constant-$\sigma$ RWP & $77.65 \pm 0.26 \pm 0.05$ & $76.26 \pm 0.47 \pm 0.11$ & $71.04 \pm 0.91 \pm 0.15$  & $61.36 \pm 2.85  \pm  0.90$\\
    Linear-$\sigma$ RWP & $ 77.95\pm 0.26 \pm 0.17$ & $ 76.61\pm 0.44 \pm 0.13$ & $ 71.85 \pm 0.90 \pm 1.23$  & $62.24 \pm 3.52 \pm 1.39$\\
    Quadratic-$\sigma$ RWP & $ \mathbf{78.26} \pm 0.20 \pm 0.13$ & $\mathbf{77.00} \pm 0.30 \pm 0.17$ & $ \mathbf{72.43} \pm 0.82 \pm 0.21$  & $\mathbf{65.27} \pm 2.05 \pm 0.64$\\
    \hline
    Constant-$\rho$ SAM & $79.12 \pm 0.23 \pm 0.28$ & $76.82 \pm 1.13 \pm 1.09$ & $66.51 \pm 2.57 \pm 1.27$  & $53.61 \pm 2.61 \pm 0.68$\\
    Linear-$\rho$ SAM & $ 79.40 \pm 0.21 \pm 0.19$ & $ \mathbf{77.61} \pm 0.51 \pm 0.42$ & $ 67.27 \pm 1.18 \pm 1.10$  & $58.37 \pm 2.42 \pm 0.75$\\
    Quadratic-$\rho$ SAM & $ \mathbf{79.49} \pm 0.34 \pm 0.14$ & $ 77.13 \pm 1.52 \pm 0.56$ & $ \mathbf{68.11} \pm 1.10 \pm 0.29$  & $\mathbf{60.24} \pm 1.98 \pm 0.83$\\
    \hline
  \end{tabular}}
\end{table*}

We perform our initial experiments on Cifar-100 using both schedule variations and compare across a variety of noise settings, shown in Tab. \ref{table:noise_schedule_comp}. From these results, we see that across all of the test-time noise settings, the ramped perturbation schedules outperform the constant-perturbations baselines of both SAM and RWP. Between the two schedule variations, the quadratic ramp schedule performs best across most of the scenarios. We also observe that the accuracy gain achieved by applying the dynamic perturbation schedule \textit{increases} as $\sigma_{test}$ (and hence $\sigma_{train}$ or $\rho$) is increased. This aligns with our understanding thus far: for larger $\sigma_{test}$, optimal generalization is achieved with even larger perturbations, but is also further inhibited by the vanishing gradient effect. With a dynamic schedule, larger perturbations are better tolerated, producing the most benefit in the large $\sigma_{test}$ scenario.

From our hyperparameter sweep (Tab.  \ref{table:noise_schedule_sam_rho}, \ref{table:noise_schedule_rwp_sigma}), we find that the optimal max perturbations are increased relative to the optimal for constant perturbation experiments. Notably, SAM's optimal $\rho$ increases significantly from 0.3 to 1.0, aligning better with the dynamic loss landscape and supporting our hypothesis of an increasing critical perturbation length. This demonstrates that greater over-regularization is achievable when introduced at the appropriate training stage, and opens a path toward perturbative training which is engineered around the evolution of the loss landscape.

To directly capture the effect ramping perturbation strength has on optimization specifically, we again calculate the noisy training accuracy at convergence for SAM-trained models on Tiny-ImageNet, both with and without the perturbation schedules, shown in Tab. \ref{tab:ramp_convergence}. We see that across all of the noise levels, introducing dynamic perturbations allows training to converge to minima with $\textit{lower}$ noisy training loss than that of training with constant perturbation strength. This confirms that dynamic perturbations directly improve optimization, independent of any possible improvement to the generalization capabilites of these minima.

\begin{algorithm}[t]
\caption{RWP with Quadratic Noise Warm-up}
\label{alg:rwp-quadratic}
\begin{algorithmic}[1]
  \REQUIRE Learning rate $\eta$, max noise $\sigma_{\max}^2$, warmup iterations $T^*$, total iterations $T$
  \STATE Initialize $w_1$
  \FOR{$t = 1,\dots,T$}
    \STATE $\sigma_t^2 \gets \sigma_{\max}^2\,\bigl(\tfrac{\min (t,T^*)}{T^*}\bigr)^{2}$
    \STATE Sample $\epsilon_t \sim \mathcal{N}\bigl(0,\;\max_j|w_{t,j}|\,\sigma_t^2 I\bigr)$
    \STATE $w_{t+1} \gets w_t - \eta\,\nabla L\bigl(w_t + \epsilon_t\bigr)$
  \ENDFOR
\end{algorithmic}
\end{algorithm}

\begin{table}[t]
    \caption{Comparison of noisy training accuracies (with various $\sigma$) on Tiny-ImagNet for SAM $\rho = 0.4$ with and without ramping perturbations.}
    \label{tab:ramp_convergence}
    \begin{tabular}{c|c|c}
    \hline
    $\sigma$ & Constant-$\rho$ SAM & Quadratic-$\rho$ SAM \\
     \hline
      $0.02$  & $89.96\pm3.51\pm4.21$ & $\mathbf{93.27}\pm0.41\pm0.88$\\
      \hline
      $0.03$  & $69.41\pm10.29\pm18.37$ & $\mathbf{82.21}\pm9.08\pm4.67$\\
      \hline
      $0.04$  & $36.21\pm14.40\pm20.03$ & $\mathbf{60.04}\pm2.30\pm1.83$\\
      \hline
      $0.05$  & $12.82\pm10.65\pm8.20$ & $\mathbf{51.64}\pm4.04\pm3.26$\\
      \hline
    \end{tabular}
\end{table}


\section{Validation through Analog Hardware Simulation}

\begin{table}[t]
  \caption{Comparison of Cifar-100 test accuracy using SGD, SAM and RWP-trained ResNet-18 models when performing simulated inference on analog accelerators.}
  \label{table:hardware_comp}
  \centering
   \resizebox{\columnwidth}{!}{\begin{tabular}{c|c|c}
    \hline
     &  RRAM &   SONOS \\
    \hline
    SGD & $65.61 \pm 2.68 \pm 0.51$ & $79.31 \pm 0.03 \pm 0.02$  \\
    \hline
    SAM $\rho = 0.2$ & $66.25 \pm 2.18 \pm 2.70$ & $80.15 \pm 0.03 \pm 0.24$  \\
    SAM $\rho = 0.3$ & $68.07 \pm 1.39 \pm 1.96$ & $80.28 \pm 0.02 \pm 0.20$  \\
    SAM $\rho = 0.4$ & $68.42 \pm 1.25 \pm 0.89$ & $80.07 \pm 0.03 \pm 0.22$  \\
    SAM + Quadratic Schedule & $66.16 \pm 1.52 \pm 2.58$ & $\mathbf{80.84} \pm 0.03 \pm 0.09$  \\
    \hline
    RWP $\sigma = 0.03$ & $71.52 \pm 2.10 \pm 3.14$ &$78.95 \pm 0.02 \pm 0.28$  \\
    RWP $\sigma = 0.05$ & $73.96 \pm 0.90 \pm 0.38$ &$78.69 \pm 0.02 \pm 0.12$ \\
    RWP $\sigma = 0.07$ & $72.46 \pm 1.10 \pm 0.84$ &$77.19 \pm 0.02 \pm 0.34$ \\
    RWP + Quadratic Schedule & $\mathbf{74.33} \pm 0.88 \pm 0.12$ &$78.89 \pm 0.03 \pm 0.17$ \\
    \hline
  \end{tabular}}
\end{table}

To demonstrate that our findings are relevant to our practical motivation, we perform simulations of model inference on analog hardware accelerators. To perform these experiments, we execute neural network inference using CrossSim \cite{CrossSim}, an open-source accuracy simulator for AIMC-based matrix multiplication operations. We select two well-known memory devices for our inference modeling: RRAM (resistive random-access memory) \cite{RRAM} and SONOS flash (silicon-oxide-
nitride-oxide-silicon) \cite{SONOS}. 
The random programming error of these devices exhibit a more complex dependence on weight value than the additive, state-independent errors used in our perturbation experiments.
Whereas the error of SONOS empirically follows a small-$\sigma$ distribution, RRAM produces significantly larger magnitude errors; to quantify this, we calculate the RMSE between the totality of the perturbed/unperturbed model weights, which we find to be $23.936 \pm 0.004$ for RRAM, and $2.371 \pm 0.0005$ for SONOS. We compare these results in Tab. \ref{table:hardware_comp}. From the table, we see that the general trends observed previously hold true: both SAM and RWP improve noise robustness relative to standard SGD on the simulated hardware. On the small-noise SONOS hardware, SAM outperforms RWP, while on the large-noise RRAM, RWP outperforms SAM. In both cases, a further performance improvement is achieved through the adoption of a quadratic perturbation schedule. 


\section{Conclusion}
In this paper, we conducted a comprehensive study to understand weight-noise robustness from the perspectives of both generalization and optimization. We proposed a novel generalization bound hypothesizing that over-regularized RWP produces better-generalizing minima, and empirically confirm this to be true. We found that SAM can find noise-robust solutions, but is hampered by a vanishing gradient effect when perturbation strength is large. We then demonstrated that through the use of a dynamic perturbation schedule, this effect can be mitigated, improving noise robustness of both SAM and RWP. Lastly, we demonstrate these techniques on simulations of analog hardware, illustrating the practical implications of our results. We hope that this extensive investigation of training neural networks robust to weight noise spurs future research of this under-explored problem.

\section*{Acknowledgements}
This work was supported by the Sandia Laboratory-Directed Research and Development (LDRD) Program. In addition, this material is based upon work supported by the U.S. Department of Energy, Office of Science, Basic Energy Sciences, as part of the Microelectronics Energy Efficiency Research Center for Advanced Technologies (MEERCAT), a Microelectronics Science Research Center (MSRC). This article has been authored by an employee of National Technology \& Engineering Solutions of Sandia, LLC under Contract No. DE-NA0003525 with the U.S. Department of Energy (DOE). The employee owns all right, title and interest in and to the article and is solely responsible for its contents. The United States Government retains and the publisher, by accepting the article for publication, acknowledges that the United States Government retains a non-exclusive, paid-up, irrevocable, world-wide license to publish or reproduce the published form of this article or allow others to do so, for United States Government purposes. The DOE will provide public access to these results of federally sponsored research in accordance with the DOE Public Access Plan https://www.energy.gov/downloads/doe-public-access-plan. SAND\#2026-17038O

\section*{Impact Statement}
This paper presents work whose goal is to advance the field of Machine
Learning. There are many potential societal consequences of our work, none
which we feel must be specifically highlighted here.

\bibliography{example_paper}

@inproceedings{foret2021sharpnessaware,
  title={Sharpness-aware Minimization for Efficiently Improving Generalization},
  author={Pierre Foret and Ariel Kleiner and Hossein Mobahi and Behnam Neyshabur},
  booktitle={International Conference on Learning Representations},
  year={2021},
}

@inproceedings{RSAM,
 author = {Liu, Yong and Mai, Siqi and Cheng, Minhao and Chen, Xiangning and Hsieh, Cho-Jui and You, Yang},
 booktitle = {Advances in Neural Information Processing Systems},
 editor = {S. Koyejo and S. Mohamed and A. Agarwal and D. Belgrave and K. Cho and A. Oh},
 pages = {24543--24556},
 title = {Random Sharpness-Aware Minimization},
 volume = {35},
 year = {2022}
}

@inproceedings{li2024friendly,
  title={Friendly Sharpness-Aware Minimization},
  author={Li, Tao and Zhou, Pan and He, Zhengbao and Cheng, Xinwen and Huang, Xiaolin},
  booktitle={Proceedings of the IEEE/CVF Conference on Computer Vision and Pattern Recognition (CVPR)},
  year={2024}
}

@article{kwon2021asam,
  title={ASAM: Adaptive Sharpness-Aware Minimization for Scale-Invariant Learning of Deep Neural Networks},
  author={Kwon, Jungmin and Kim, Jeongseop and Park, Hyunseo and Choi, In Kwon},
  journal={arXiv preprint arXiv:2102.11600},
  year={2021}
}

@inproceedings{kim2022fishersaminformationgeometry,
      title={Fisher SAM: Information Geometry and Sharpness Aware Minimisation}, 
      author={Minyoung Kim and Da Li and Shell Xu Hu and Timothy M. Hospedales},
      year={2022},
      booktitle={Proceedings of Machine Learning Research},
}

@inproceedings{du2022efficientsharpnessawareminimizationimproved,
      title={Efficient Sharpness-aware Minimization for Improved Training of Neural Networks}, 
      author={Jiawei Du and Hanshu Yan and Jiashi Feng and Joey Tianyi Zhou and Liangli Zhen and Rick Siow Mong Goh and Vincent Y. F. Tan},
      year={2022},
      booktitle={International Conference on Learning Representations},
}

@inproceedings{jiang2023adaptivepolicyemploysharpnessaware,
      title={An Adaptive Policy to Employ Sharpness-Aware Minimization}, 
      author={Weisen Jiang and Hansi Yang and Yu Zhang and James Kwok},
      year={2023},
      booktitle={International Conference on Learning Representations},
}

@inproceedings{liu2022efficientscalablesharpnessawareminimization,
      title={Towards Efficient and Scalable Sharpness-Aware Minimization}, 
      author={Yong Liu and Siqi Mai and Xiangning Chen and Cho-Jui Hsieh and Yang You},
      year={2022},
      booktitle={Proceedings of the IEEE/CVF Conference on Computer Vision and Pattern Recognition (CVPR)},
}

@inproceedings{Mi2022,
 author = {Mi, Peng and Shen, Li and Ren, Tianhe and Zhou, Yiyi and Sun, Xiaoshuai and Ji, Rongrong and Tao, Dacheng},
 booktitle = {Advances in Neural Information Processing Systems},
 editor = {S. Koyejo and S. Mohamed and A. Agarwal and D. Belgrave and K. Cho and A. Oh},
 pages = {30950--30962},
 title = {Make Sharpness-Aware Minimization Stronger: A Sparsified Perturbation Approach},
 volume = {35},
 year = {2022}
}

@inproceedings{andriushchenko2022understandingsharpnessawareminimization,
      title={Towards Understanding Sharpness-Aware Minimization}, 
      author={Maksym Andriushchenko and Nicolas Flammarion},
      year={2022},
      booktitle={International Conference on Marchine Learning (ICML)},
}

@inproceedings{wen2023doessharpnessawareminimizationminimize,
      title={How Does Sharpness-Aware Minimization Minimize Sharpness?}, 
      author={Kaiyue Wen and Tengyu Ma and Zhiyuan Li},
      year={2023},
      booktitle={International Conference on Learning Representations},
}

@inproceedings{khanh2024,
 author = {Khanh, Pham Duy and Luong, Hoang-Chau and Mordukhovich, Boris S. and Tran, Dat Ba},
 booktitle = {Advances in Neural Information Processing Systems},
 editor = {A. Globerson and L. Mackey and D. Belgrave and A. Fan and U. Paquet and J. Tomczak and C. Zhang},
 pages = {13149--13182},
 title = {Fundamental Convergence Analysis of Sharpness-Aware Minimization},
 volume = {37},
 year = {2024}
}

@inproceedings{baek2024samrobustlabelnoise,
      title={Why is SAM Robust to Label Noise?}, 
      author={Christina Baek and Zico Kolter and Aditi Raghunathan},
      year={2024},
      booktitle={International Conference on Learning Representations},
}

@inproceedings{SAMPA,
 author = {Xie, Wanyun and Pethick, Thomas and Cevher, Volkan},
 booktitle = {Advances in Neural Information Processing Systems},
 editor = {A. Globerson and L. Mackey and D. Belgrave and A. Fan and U. Paquet and J. Tomczak and C. Zhang},
 pages = {51333--51357},
 title = {SAMPa: Sharpness-aware Minimization Parallelized},
 volume = {37},
 year = {2024}
}

@ARTICLE{an1996,

  author={An, Guozhong},

  journal={Neural Computation}, 

  title={The Effects of Adding Noise During Backpropagation Training on a Generalization Performance}, 

  year={1996},

  volume={8},

  number={3},

  pages={643-674},

  doi={10.1162/neco.1996.8.3.643}}

@misc{neelakantan2015addinggradientnoiseimproves,
      title={Adding Gradient Noise Improves Learning for Very Deep Networks}, 
      author={Arvind Neelakantan and Luke Vilnis and Quoc V. Le and Ilya Sutskever and Lukasz Kaiser and Karol Kurach and James Martens},
      year={2015},
      eprint={1511.06807},
      archivePrefix={arXiv},
      primaryClass={stat.ML},
}

@InProceedings{zhou19d,
  title = 	 {Toward Understanding the Importance of Noise in Training Neural Networks},
  author =       {Zhou, Mo and Liu, Tianyi and Li, Yan and Lin, Dachao and Zhou, Enlu and Zhao, Tuo},
  booktitle = 	 {Proceedings of the 36th International Conference on Machine Learning},
  pages = 	 {7594--7602},
  year = 	 {2019},
  editor = 	 {Chaudhuri, Kamalika and Salakhutdinov, Ruslan},
  volume = 	 {97},
  series = 	 {Proceedings of Machine Learning Research},
  month = 	 {09--15 Jun},
  publisher =    {PMLR}
}

@misc{bisla2022lowpassfilteringsgdrecovering,
      title={Low-Pass Filtering SGD for Recovering Flat Optima in the Deep Learning Optimization Landscape}, 
      author={Devansh Bisla and Jing Wang and Anna Choromanska},
      year={2022},
      eprint={2201.08025},
      archivePrefix={arXiv},
      primaryClass={cs.LG},
}

@inproceedings{möllenhoff2023samoptimalrelaxationbayes,
      title={SAM as an Optimal Relaxation of Bayes}, 
      author={Thomas Möllenhoff and Mohammad Emtiyaz Khan},
      year={2023},
      booktitle={International Conference on Learning Representations},
}

@article{li2024revisiting,
  title={Revisiting Random Weight Perturbation for Efficiently Improving Generalization},
  author={Li, Tao and Tao, Qinghua and Yan, Weihao and  Lei, Zehao and Wu, Yingwen and Fang, Kun and He, Mingzhen and Huang, Xiaolin},
journal={Transactions on Machine Learning Research (TMLR)},
  year={2024}
}

@InProceedings{rusak2020,
author="Rusak, Evgenia
and Schott, Lukas
and Zimmermann, Roland S.
and Bitterwolf, Julian
and Bringmann, Oliver
and Bethge, Matthias
and Brendel, Wieland",
editor="Vedaldi, Andrea
and Bischof, Horst
and Brox, Thomas
and Frahm, Jan-Michael",
title="A Simple Way to Make Neural Networks Robust Against Diverse Image Corruptions",
booktitle="Computer Vision -- ECCV 2020",
year="2020",
publisher="Springer International Publishing",
address="Cham",
pages="53--69",
isbn="978-3-030-58580-8"
}

@INPROCEEDINGS{fang2023,

  author={Fang, Xiuwen and Ye, Mang and Yang, Xiyuan},

  booktitle={2023 IEEE/CVF International Conference on Computer Vision (ICCV)}, 

  title={Robust Heterogeneous Federated Learning under Data Corruption}, 

  year={2023},

  volume={},

  number={},

  pages={4997-5007},

}

@inproceedings{kar20223d,
  title={3D Common Corruptions and Data Augmentation},
  author={Kar, O{\u{g}}uzhan Fatih and Yeo, Teresa and Atanov, Andrei and Zamir, Amir},
  booktitle={Proceedings of the IEEE/CVF Conference on Computer Vision and Pattern Recognition},
  pages={18963--18974},
  year={2022}
}

@inproceedings{mintun2021,
 author = {Mintun, Eric and Kirillov, Alexander and Xie, Saining},
 booktitle = {Advances in Neural Information Processing Systems},
 editor = {M. Ranzato and A. Beygelzimer and Y. Dauphin and P.S. Liang and J. Wortman Vaughan},
 pages = {3571--3583},
 title = {On Interaction Between Augmentations and Corruptions in Natural Corruption Robustness},
 volume = {34},
 year = {2021}
}

@InProceedings{Guo_2023_CVPR,
    author    = {Guo, Yong and Stutz, David and Schiele, Bernt},
    title     = {Improving Robustness of Vision Transformers by Reducing Sensitivity To Patch Corruptions},
    booktitle = {Proceedings of the IEEE/CVF Conference on Computer Vision and Pattern Recognition (CVPR)},
    month     = {June},
    year      = {2023},
    pages     = {4108-4118}
}

@inproceedings{PRIME2022,
    title = {PRIME: A Few Primitives Can Boost Robustness to Common Corruptions}, 
    author = {Apostolos Modas and Rahul Rade and Guillermo {Ortiz-Jim\'enez} and Seyed-Mohsen {Moosavi-Dezfooli} and Pascal Frossard},
    year = {2022},
    booktitle = {European Conference on Computer Vision (ECCV)}
}

@inproceedings{goodfellow2015explainingharnessingadversarialexamples,
      title={Explaining and Harnessing Adversarial Examples}, 
      author={Ian J. Goodfellow and Jonathon Shlens and Christian Szegedy},
      year={2015},
      booktitle={International Conference on Learning Representations},
}

@InProceedings{Mustafa_2019_ICCV,
author = {Mustafa, Aamir and Khan, Salman and Hayat, Munawar and Goecke, Roland and Shen, Jianbing and Shao, Ling},
title = {Adversarial Defense by Restricting the Hidden Space of Deep Neural Networks},
booktitle = {Proceedings of the IEEE/CVF International Conference on Computer Vision (ICCV)},
month = {October},
year = {2019}
}

@inproceedings{Yan2018,
 author = {Yan, Ziang and Guo, Yiwen and Zhang, Changshui},
 booktitle = {Advances in Neural Information Processing Systems},
 editor = {S. Bengio and H. Wallach and H. Larochelle and K. Grauman and N. Cesa-Bianchi and R. Garnett},
 pages = {},
 title = {Deep Defense: Training DNNs with Improved Adversarial Robustness},
 volume = {31},
 year = {2018}
}

@article{yang2022,
author = {Yang, Xiaoxuan and Wu, Changming and Li, Mo and Chen, Yiran},
title = {Tolerating Noise Effects in Processing-in-Memory Systems for Neural Networks: A Hardware–Software Codesign Perspective},
journal = {Advanced Intelligent Systems},
volume = {4},
number = {8},
pages = {2200029},
doi = {https://doi.org/10.1002/aisy.202200029},
eprint = {https://advanced.onlinelibrary.wiley.com/doi/pdf/10.1002/aisy.202200029},
year = {2022}
}

@INPROCEEDINGS{gokmen2019,

  author={Gokmen, Tayfun and Rasch, Malte J. and Haensch, Wilfried},

  booktitle={2019 IEEE International Electron Devices Meeting (IEDM)}, 

  title={The marriage of training and inference for scaled deep learning analog hardware}, 

  year={2019},

  volume={},

  number={},

  pages={22.3.1-22.3.4},

  doi={10.1109/IEDM19573.2019.8993573}}

@ARTICLE{kariyappa2021,

  author={Kariyappa, Sanjay and Tsai, Hsinyu and Spoon, Katie and Ambrogio, Stefano and Narayanan, Pritish and Mackin, Charles and Chen, An and Qureshi, Moinuddin and Burr, Geoffrey W.},

  journal={IEEE Transactions on Electron Devices}, 

  title={Noise-Resilient DNN: Tolerating Noise in PCM-Based AI Accelerators via Noise-Aware Training}, 

  year={2021},

  volume={68},

  number={9},

  pages={4356-4362},

  doi={10.1109/TED.2021.3089987}}

@article{Rasch2023,
  author = {Malte J. Rasch and Charles Mackin and Manuel Le Gallo and An Chen and Andrea Fasoli and Frédéric Odermatt and Ning Li and S. R. Nandakumar and Pritish Narayanan and Hsinyu Tsai and Geoffrey W. Burr and Abu Sebastian and Vijay Narayanan},
  title = {Hardware-aware training for large-scale and diverse deep learning inference workloads using in-memory computing-based accelerators},
  journal = {Nature Communications},
  volume = {14},
  number = {1},
  pages = {5282},
  year = {2023},
  month = {August},
  doi = {10.1038/s41467-023-40770-4},
  issn = {2041-1723}
}

@ARTICLE{xiao2023,

  author={Xiao, T. Patrick and Feinberg, Ben and Bennett, Christopher H. and Prabhakar, Venkatraman and Saxena, Prashant and Agrawal, Vineet and Agarwal, Sapan and Marinella, Matthew J.},

  journal={IEEE Circuits and Systems Magazine}, 

  title={On the Accuracy of Analog Neural Network Inference Accelerators}, 

  year={2022},

  volume={22},

  number={4},

  pages={26-48},

  doi={10.1109/MCAS.2022.3214409}}

@inproceedings{visualloss,
  title={Visualizing the Loss Landscape of Neural Nets},
  author={Li, Hao and Xu, Zheng and Taylor, Gavin and Studer, Christoph and Goldstein, Tom},
  booktitle={Neural Information Processing Systems},
  year={2018}
}

@article{Keskar2016,
	author = {Nitish Shirish Keskar and Dheevatsa Mudigere and Jorge Nocedal and Mikhail Smelyanskiy and Ping Tak Peter Tang},
	title = {On Large-Batch Training for Deep Learning: Generalization Gap and Sharp Minima},
	journal = {arXiv preprint arXiv:1609.04836},
	year = {2016}
}

@inproceedings{DR17,
        title = {Computing Nonvacuous Generalization Bounds for Deep (Stochastic) Neural Networks with Many More Parameters than Training Data},
       author = {Gintare Karolina Dziugaite and Daniel M. Roy},
         year = {2017},
    booktitle = {Proceedings of the 33rd Annual Conference on Uncertainty in Artificial Intelligence (UAI)},
archivePrefix = {arXiv},
       eprint = {1703.11008},
}

@inproceedings{jiang2019fantasticgeneralizationmeasures,
      title={Fantastic Generalization Measures and Where to Find Them}, 
      author={Yiding Jiang and Behnam Neyshabur and Hossein Mobahi and Dilip Krishnan and Samy Bengio},
      year={2019},
      booktitle = {International Conference on Learning Representations},
}

@article{Sebastian2020,
  author = {Abu Sebastian and Manuel Le Gallo and Riduan Khaddam-Aljameh and Evangelos Eleftheriou},
  title = {Memory devices and applications for in-memory computing},
  journal = {Nature Nanotechnology},
  volume = {15},
  number = {7},
  pages = {529--544},
  year = {2020},
  month = {July},
  doi = {10.1038/s41565-020-0655-z},
  issn = {1748-3395}
}

@article{He2015,
	author = {Kaiming He and Xiangyu Zhang and Shaoqing Ren and Jian Sun},
	title = {Deep Residual Learning for Image Recognition},
	journal = {arXiv preprint arXiv:1512.03385},
	year = {2015}
}

@Techreport{krizhevsky2009learning,
 author = {Krizhevsky, Alex and Hinton, Geoffrey},
 address = {Toronto, Ontario},
 institution = {University of Toronto},
 number = {0},
 publisher = {Technical report, University of Toronto},
 title = {Learning multiple layers of features from tiny images},
 year = {2009},
 title_with_no_special_chars = {Learning multiple layers of features from tiny images},
}

@misc{CrossSim,
  author={Ben Feinberg and T. Patrick Xiao and Curtis J. Brinker and Christopher H. Bennett and Matthew J. Marinella and Sapan Agarwal},
  title={CrossSim: accuracy simulation of analog in-memory computing},
  url = {https://github.com/sandialabs/cross-sim},
  year={2022}
}

@INPROCEEDINGS{RRAM,

  author={Milo, Valerio and Anzalone, Francesco and Zambelli, Cristian and Pérez, Eduardo and Mahadevaiah, Mamathamba K. and Ossorio, Oscar G. and Olivo, Piero and Wenger, Christian and Ielmini, Daniele},

  booktitle={2021 IEEE International Reliability Physics Symposium (IRPS)}, 

  title={Optimized programming algorithms for multilevel RRAM in hardware neural networks}, 

  year={2021},

  volume={},

  number={},

  pages={1-6},

  doi={10.1109/IRPS46558.2021.9405119}}

@ARTICLE{SONOS,

  author={Xiao, T. Patrick and Feinberg, Ben and Bennett, Christopher H. and Agrawal, Vineet and Saxena, Prashant and Prabhakar, Venkatraman and Ramkumar, Krishnaswamy and Medu, Harsha and Raghavan, Vijay and Chettuvetty, Ramesh and Agarwal, Sapan and Marinella, Matthew J.},

  journal={IEEE Transactions on Circuits and Systems I: Regular Papers}, 

  title={An Accurate, Error-Tolerant, and Energy-Efficient Neural Network Inference Engine Based on SONOS Analog Memory}, 

  year={2022},

  volume={69},

  number={4},

  pages={1480-1493},

  doi={10.1109/TCSI.2021.3134313}}

@INPROCEEDINGS{imagenet,

  author={Deng, Jia and Dong, Wei and Socher, Richard and Li, Li-Jia and Kai Li and Li Fei-Fei},

  booktitle={2009 IEEE Conference on Computer Vision and Pattern Recognition}, 

  title={ImageNet: A large-scale hierarchical image database}, 

  year={2009},

  volume={},

  number={},

  pages={248-255}}

@misc{Le2015TinyIV,
  title={Tiny ImageNet Visual Recognition Challenge},
  author={Ya Le and Xuan S. Yang},
  year={2015},
}

@article{DPRN,
  title={Deep Pyramidal Residual Networks},
  author={Han, Dongyoon and Kim, Jiwhan and Kim, Junmo},
  journal={IEEE CVPR},
  year={2017}
}

@article{chatterji2020,
      title={The intriguing role of module criticality in the generalization of deep networks}, 
      author={Niladri S. Chatterji and Behnam Neyshabur and Hanie Sedghi},
      year={2020},
      journal={International Conference on Learning Representations},
}

@inproceedings{McAllester1999,
    author = {McAllester, David A.},
    title = {PAC-Bayesian model averaging},
    year = {1999},
    isbn = {1581131674},
    publisher = {Association for Computing Machinery},
    address = {New York, NY, USA},
    doi = {10.1145/307400.307435},
    booktitle = {Proceedings of the Twelfth Annual Conference on Computational Learning Theory},
    pages = {164–170},
    numpages = {7},
    location = {Santa Cruz, California, USA},
    series = {COLT '99}
    }
\bibliographystyle{icml2026}

\newpage
\appendix
\onecolumn

\section{Proof of Theorem 4.1}
\label{sec:pac_bayes_proof}

First, we prove the following between relation between $\mathbb{E}_{\epsilon \sim \mathcal{N}(0,\sigma^2_{test})}[L_{\mathcal{D}}(w+\epsilon)]$ and $\mathbb{E}_{\epsilon \sim \mathcal{N}(0,\sigma^2_{train})}[L_{\mathcal{D}}(w+\epsilon)]$:

\begin{lemma}
  \label{lem:increasing_sigma}
  Assume $\sigma_{train}$,  $\sigma_{test}$ small and $\Delta L_{\mathcal{D}} > 0$. Then if $\sigma_{train} > \sigma_{test}$, we have that 
  \begin{equation*}
      \mathbb{E}_{\epsilon \sim \mathcal{N}(0,\sigma^2_{test})}[L_{\mathcal{D}}(w+\epsilon)] < \mathbb{E}_{\epsilon \sim \mathcal{N}(0,\sigma^2_{train})}[L_{\mathcal{D}}(w+\epsilon)]
  \end{equation*}
\end{lemma}
We remark that the assumption of $\Delta L_{\mathcal{D}} > 0$ is expected to hold at convergence: even if optimization converges to a saddle point (as is the case generally), we expect a majority of Hessian eigenvalues to be positive (i.e. increasing in more directions than decreasing).
\begin{proof}
     Noting that the $\sigma$ are small, we use a second-order Taylor expansion to express $L_{\mathcal{D}}(w+\epsilon)$:

    \begin{align*}
        L_{\mathcal{D}}(w+\epsilon) \approx L_{\mathcal{D}}(w)+\nabla L(w)\epsilon + \frac{1}{2}\epsilon^T\nabla^2L(w)\epsilon
    \end{align*}
    Taking the expectation of both sides w.r.t $\epsilon$, we have:
    \begin{align*}
        \mathbb{E}_{\epsilon \sim \mathcal{N}(0,\sigma^2)}[L_{\mathcal{D}}(w+\epsilon)] = L_{\mathcal{D}}(w)+\cancel{\mathbb{E}_{\epsilon \sim \mathcal{N}(0,\sigma^2)}[\nabla L_{\mathcal{D}}(w)\epsilon]} + \mathbb{E}_{\epsilon \sim \mathcal{N}(0,\sigma^2)}\left[\frac{1}{2}\epsilon^T\nabla^2L_{\mathcal{D}}(w)\epsilon\right] \\
        = L_{\mathcal{D}}(w)+ \mathbb{E}_{\epsilon \sim \mathcal{N}(0,\sigma^2)}\left[\frac{1}{2}Tr(\epsilon^T\nabla^2L_{\mathcal{D}}(w)\epsilon)\right] \\
        = L_{\mathcal{D}}(w)+ \mathbb{E}_{\epsilon \sim \mathcal{N}(0,\sigma^2)}\left[\frac{1}{2}Tr(\epsilon\epsilon^T\nabla^2L_{\mathcal{D}}(w))\right] \\
        = L_{\mathcal{D}}(w) + \frac{1}{2}\sigma^2 Tr(\nabla^2 L_{\mathcal{D}}(w)) \\
        = L_{\mathcal{D}}(w) + \frac{1}{2}\sigma^2\Delta L_{\mathcal{D}}(w)
    \end{align*}
    Because $\Delta L_{\mathcal{D}} > 0$, $\mathbb{E}_{\epsilon \sim \mathcal{N}(0,\sigma^2)}[L_{\mathcal{D}}(w+\epsilon)]$ strictly increases as $\sigma$ increases, hence proving the lemma.
\end{proof}

\begin{theorem}
    \label{thm:full_bound}
    Assume $\Delta L_{\mathcal{D}} > 0$. For any small $\sigma_{train}$, $\sigma_{test}$ where $\sigma_{train} > \sigma_{test}$, the following holds with probability $1-\delta$:
    \begin{align}
          \label{eq:new_bound}
          \mathbb{E}_{\epsilon \sim \mathcal{N}(0,\sigma^2_{test})}[L_{\mathcal{D}}(w+\epsilon)] \leq \nonumber \\
          \mathbb{E}_{\epsilon \sim \mathcal{N}(0,\sigma^2_{train})}[L_{\mathcal{S}}(w+\epsilon)] + \sqrt{\frac{\frac{1}{4}k\log\left( 1 + \frac{||w||^2}{k\sigma_{train}^2} \right)+\frac{1}{4}+\log\frac{n}{\delta}+2\log(6n+3k)}{n-1}}
    \end{align}
    where $n=|\mathcal{S}|$ and $k$ is the number of model parameters.
\end{theorem}
\begin{proof}
    We begin with a perturbed version of the PAC Bayes generalization bound, as derived in \cite{foret2021sharpnessaware,chatterji2020}:
    \begin{align}
          \mathbb{E}_{\epsilon \sim \mathcal{N}(0,\sigma^2)}[L_{\mathcal{D}}(w+\epsilon)] \leq \nonumber \\
          \mathbb{E}_{\epsilon \sim \mathcal{N}(0,\sigma^2)}[L_{\mathcal{S}}(w+\epsilon)] + \sqrt{\frac{\frac{1}{4}k\log\left( 1 + \frac{||w||^2}{k\sigma^2} \right)+\frac{1}{4}+\log\frac{n}{\delta}+2\log(6n+3k)}{n-1}}
    \end{align}
    From Lemma \ref{lem:increasing_sigma}, we know that   
    \begin{equation*}
      \mathbb{E}_{\epsilon \sim \mathcal{N}(0,\sigma^2_{1})}[L_{\mathcal{D}}(w+\epsilon)] < \mathbb{E}_{\epsilon \sim \mathcal{N}(0,\sigma^2_{2})}[L_{\mathcal{D}}(w+\epsilon)]
  \end{equation*}
  if $\sigma_2 > \sigma_1$, meaning that any upper bound of $\mathbb{E}_{\epsilon \sim \mathcal{N}(0,\sigma^2_{2})}[L_{\mathcal{D}}(w+\epsilon)]$ must likewise upper bound $\mathbb{E}_{\epsilon \sim \mathcal{N}(0,\sigma^2_{1})}[L_{\mathcal{D}}(w+\epsilon)]$. Thus, by replacing the bound on $\mathbb{E}_{\epsilon \sim \mathcal{N}(0,\sigma^2_{test})}[L_{\mathcal{D}}(w+\epsilon)]$ with that of $\sigma_{train}$, we arrive at Eq. \ref{eq:new_bound}.
\end{proof}

\section{Implementation Details}
All Cifar-100/Tiny-ImageNet models are trained for 200 epochs using SGD as the base optimizer with batch size 128. We use an initial learning rate of $0.05$ with a cosine learning rate schedule. We use a weight decay of $5*10^{-4}$ and momentum of $0.9$. We apply the standard data augmentations of random cropping, flipping, and normalization. In addition, we train using label smoothing of 0.1. For the ImageNet-100 models, we train for 100 epochs using an initial learning rate of $1$ and weight decay of $1*10^{-4}$. In all experiments, early stopping is applied to select the training epoch which achieves the highest perturbed test accuracy.

\section{Visualizing Test Error vs Noise Strength}
Our proof of Theorem \ref{thm:full_bound} relies on the assumption that test error strictly increases as $\sigma_{test}$ is increased. Although we expect this to hold true in the case that $\sigma_{train}$, $\sigma_{test}$ are small, it is not immediately clear that our experimental settings satisfy this condition. To empirically confirm the validity of this assumption, we plot the test error as a function of $\sigma_{test}$ for several models trained with RWP in Fig. \ref{fig:acc_vs_sigma}. From these results, we see that test error consistently increases over the range of $\sigma_{test}$ values we evaluate.

\begin{figure*}
    \begin{subfigure}{0.33\textwidth}
      \centering
      \includegraphics[width=1.0\linewidth]{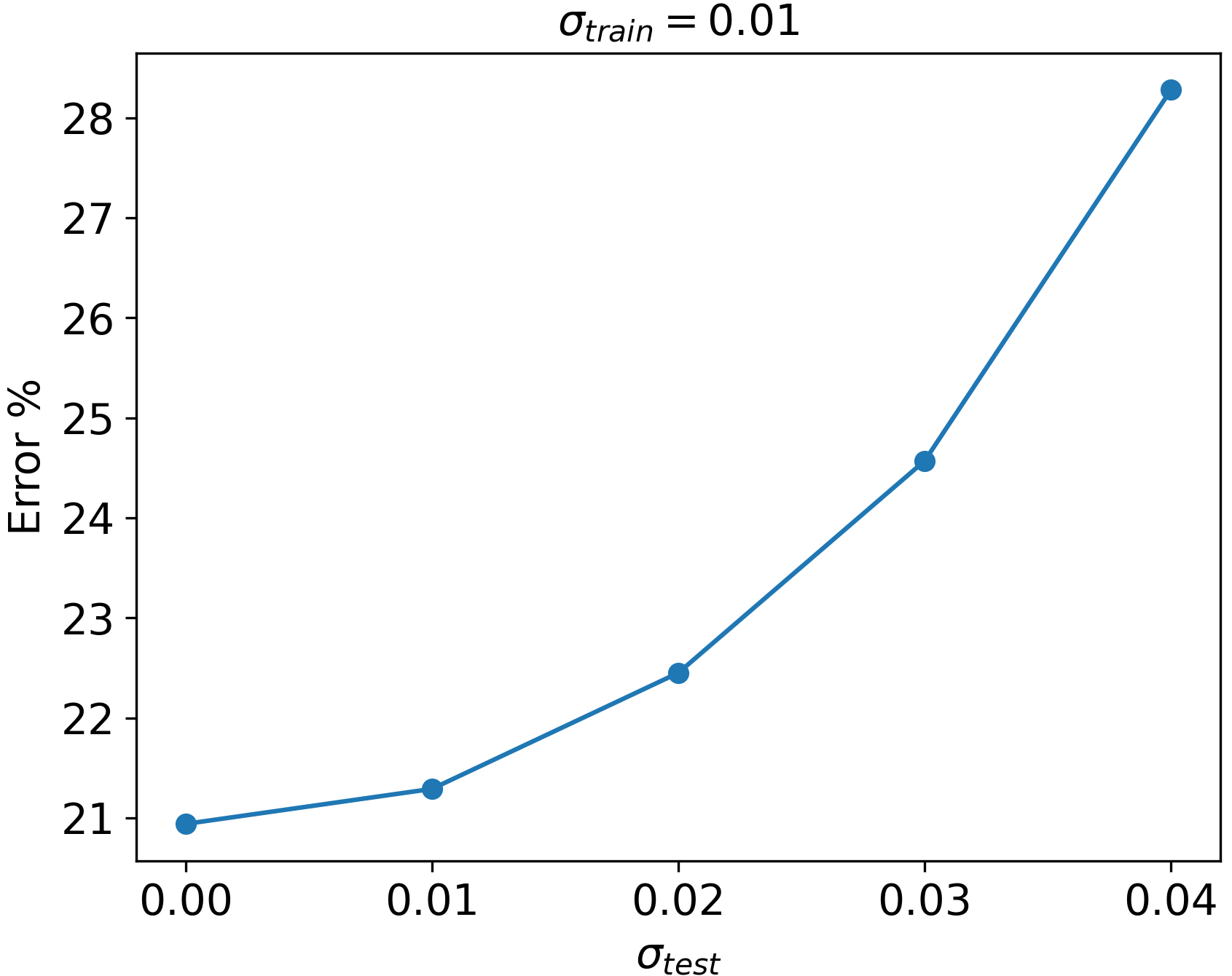}
      \caption{}
    \end{subfigure}
    \begin{subfigure}{.33\textwidth}
      \centering
      \includegraphics[width=1.0\linewidth]{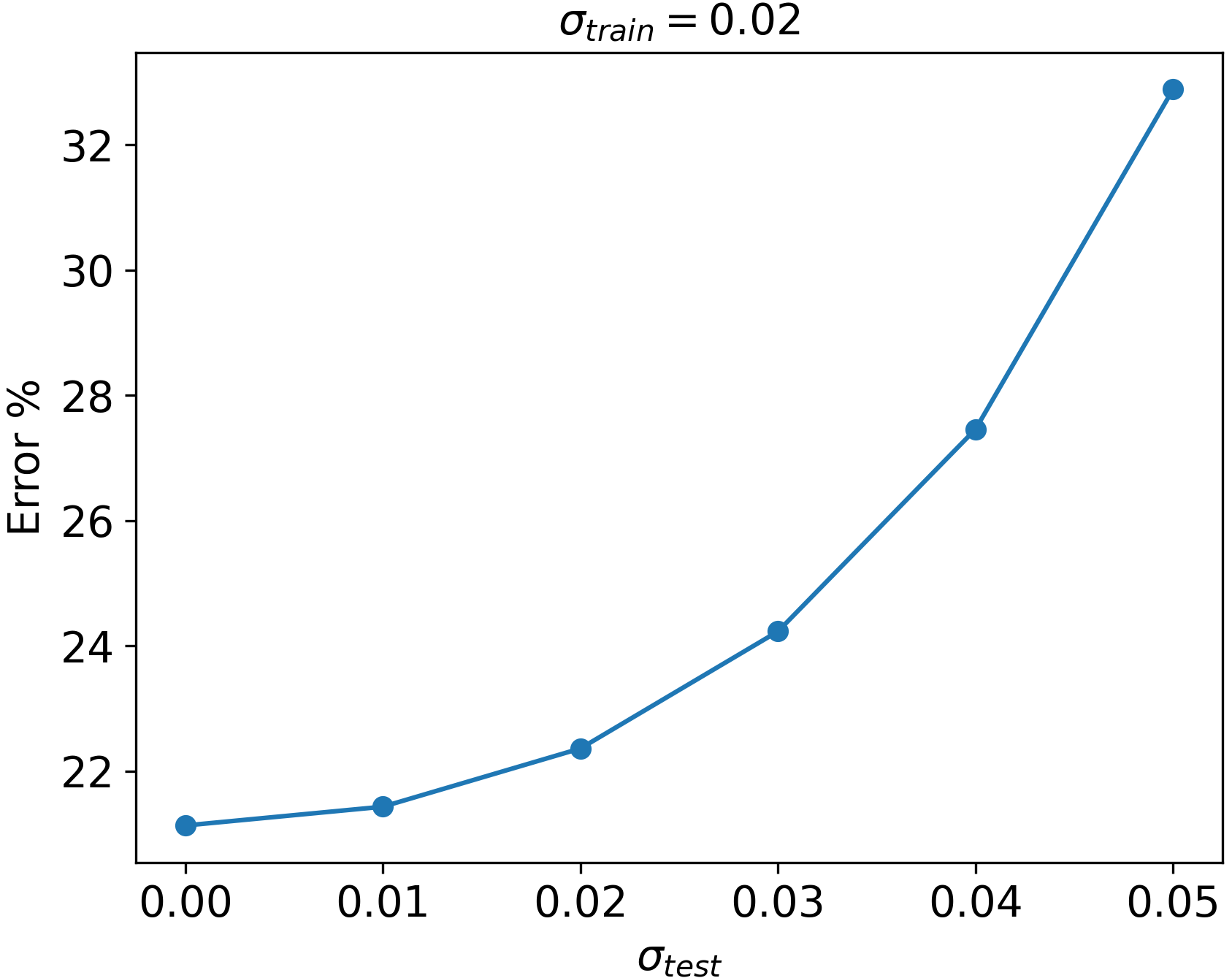}
      \caption{}
    \end{subfigure}
    \begin{subfigure}{.33\textwidth}
      \centering
      \includegraphics[width=1.0\linewidth]{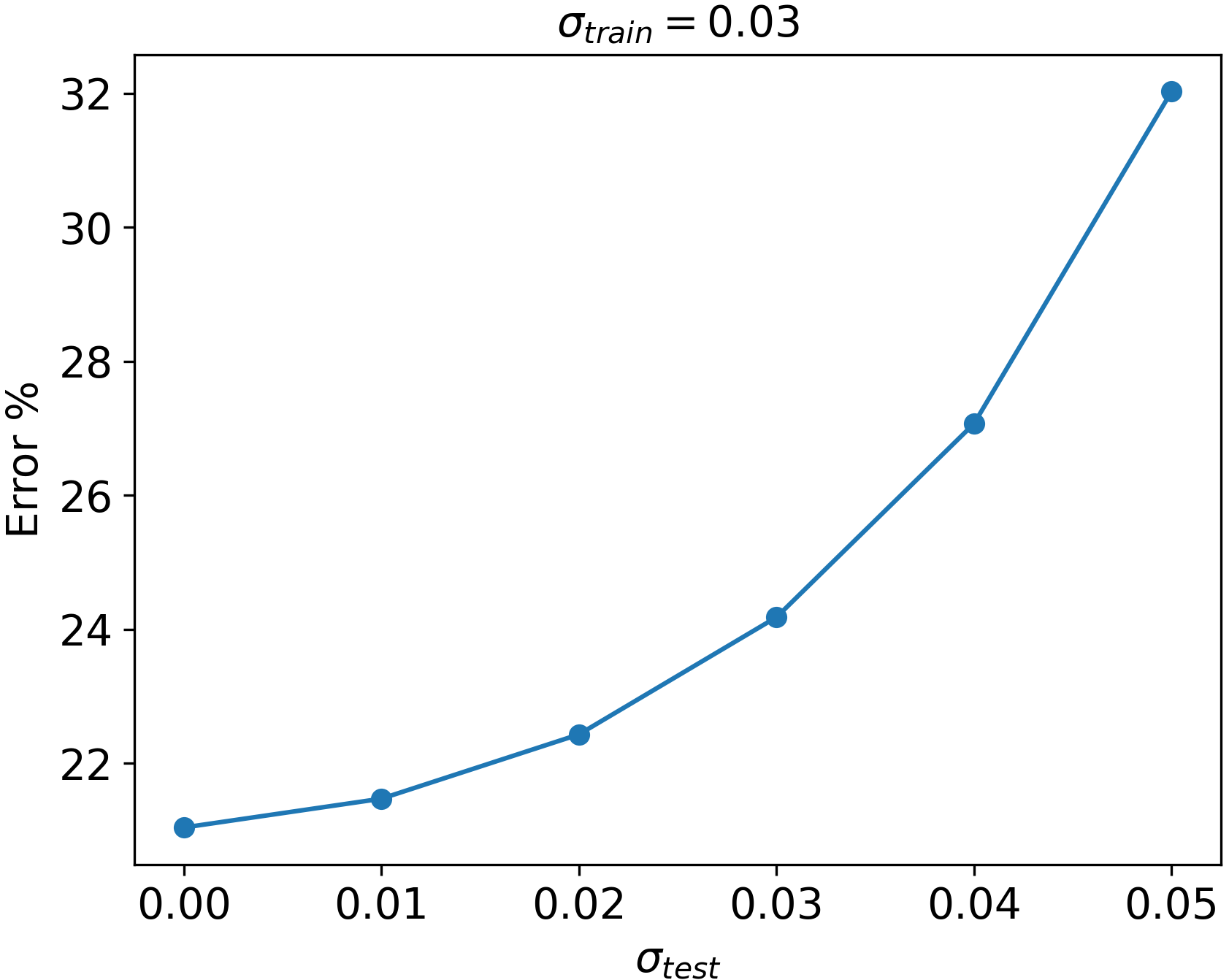}
      \caption{}
    \end{subfigure}
  \caption{Plot of noisy test accuracy as a function of $\sigma_{test}$ for RWP trained with (a) $\sigma_{train} = 0.01$, (b) $\sigma_{train} = 0.02$, and (c) $\sigma_{train} = 0.03$. }
  \label{fig:acc_vs_sigma}
\end{figure*}

\section{Detailing the Dynamics of Perturbed Training}
\label{sec:perturbed_dynamics}
\subsection{Understanding SAM's Overfitting}
\label{sec:sam_overfit}
To empirically investigate the evolution of SAM's noise robustness during training, we plot the cosine similarity between the original loss gradient $\nabla L(\mathbf{w})$ and the perturbed gradient $\nabla L(\mathbf{w}+\mathbf{\epsilon})$ for both SAM and RWP, shown in Fig. \ref{fig:cos_sim}. From this, we see that for RWP, the cosine similarity between the gradients gradually diverges as training progresses and the optimization path enters a deeper loss valley with more rapidly-changing gradients. As such, the gradients are, on average, nearly orthogonal by the end of training. On the other hand, in the SAM-trained models, the cosine similarity drops sharply after the early training iterations, but thereafter roughly plateaus for the rest of the training schedule. As a result, even as the training converges to a sharper minimum, the SAM gradient still contains a \textit{significant} component along the direction of the original gradient. Hence, further training will continue to push the model toward a minimum more characteristic of the one found by SGD, although at an attenuated rate. For the strongly perturbed SAM using $\rho=1$, the gradients do diverge at later training epochs; however the strong perturbations in early iterations prevent the optimization from progressing meaningfully, as will be shown in the following section.

\begin{figure}
    \begin{subfigure}{.5\textwidth}
      \centering
      \includegraphics[width=1.0\linewidth]{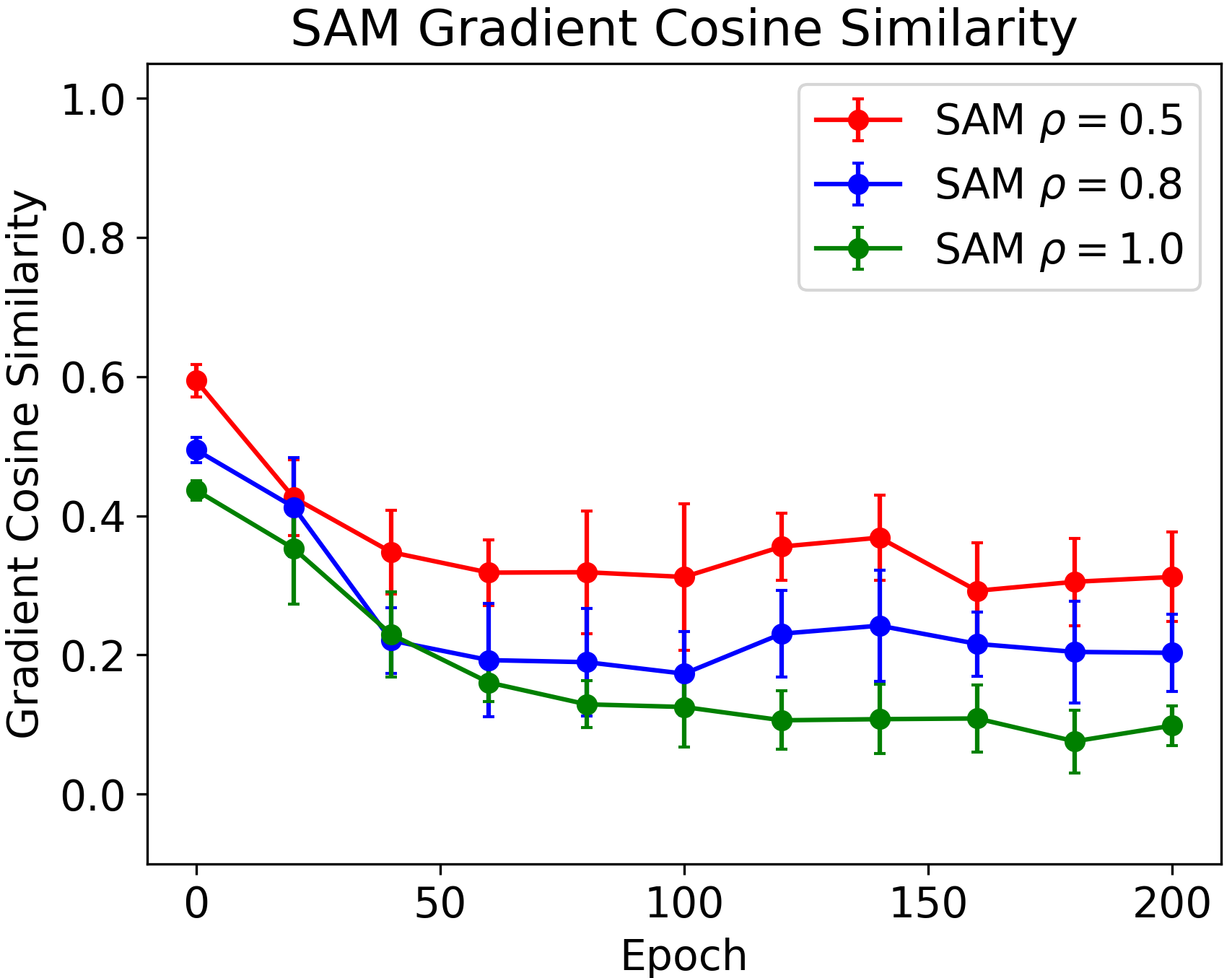}
      \caption{}
      \label{fig:sam_cos_sim}
    \end{subfigure}%
    \begin{subfigure}{.5\textwidth}
      \centering
      \includegraphics[width=1.0\linewidth]{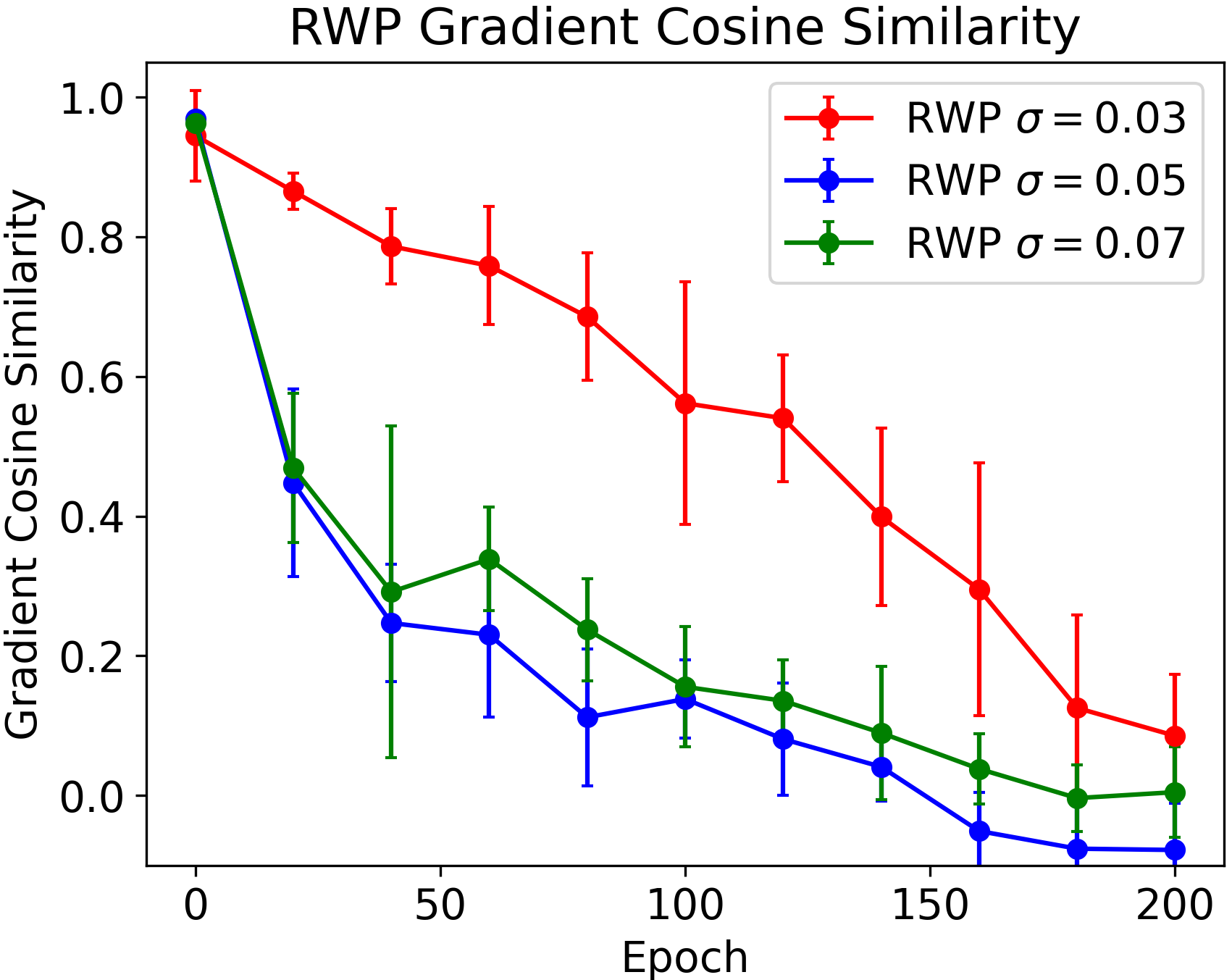}
      \caption{}
      \label{fig:rwp_cos_sim}
    \end{subfigure}
    \caption{Plot of cosine similarity between perturbed/unperturbed gradients for various (a) SAM and (b) RWP configurations on a ResNet-18 as a function of training epoch.}
    \label{fig:cos_sim}
\end{figure}

\subsection{Charting SAM/RWP Performance}
As a corollary to our understanding of noisy generalization, we also attempt to capture the shape of the test loss minimum through visualizing the depth (accuracy in unperturbed setting) vs. width (accuracy loss when perturbations are applied) of the minimum. To do so, we plot $\sigma_{test}=0.02$ \& $\sigma_{test}=0.07$ test accuracy vs. unperturbed test accuracy for a variety of SGD, SAM and RWP models trained using varying $\sigma_{train}$ \& $\rho$ in Fig. \ref{fig:performance_comp}. For $\sigma_{test}=0.02$, we observe that SAM and RWP follow a similar linear trend, but with the best SAM minima capable of achieving both lower perturbed and unperturbed accuracy. This implies that SAM's improvement to noisy generalization largely arises from the same mechanism that improves SAM's noiseless generalization, as these lower-loss minima simultaneously exhibit lower perturbed loss. In the case of $\sigma_{test}=0.07$, we again observe a roughly linear trend with respect to SAM minima, with the solution exhibiting the highest perturbed accuracy likewise achieving the highest unperturbed accuracy. This consistent trend confirms that SAM \textit{cannot} effectively increase noisy generalization at the cost of the noiseless generalization (i.e. trading a deep \& narrow minimum for a shallow \& wide one). On the other hand, we see that in the case of RWP, both deep-but-sharp and flat-but-shallow (the right and left regions respectively) perform non-ideally, with a modest region in the middle representing the optimal trade-off.

\begin{figure*}
    \begin{subfigure}{.5\textwidth}
      \centering
      \includegraphics[width=1.0\linewidth]{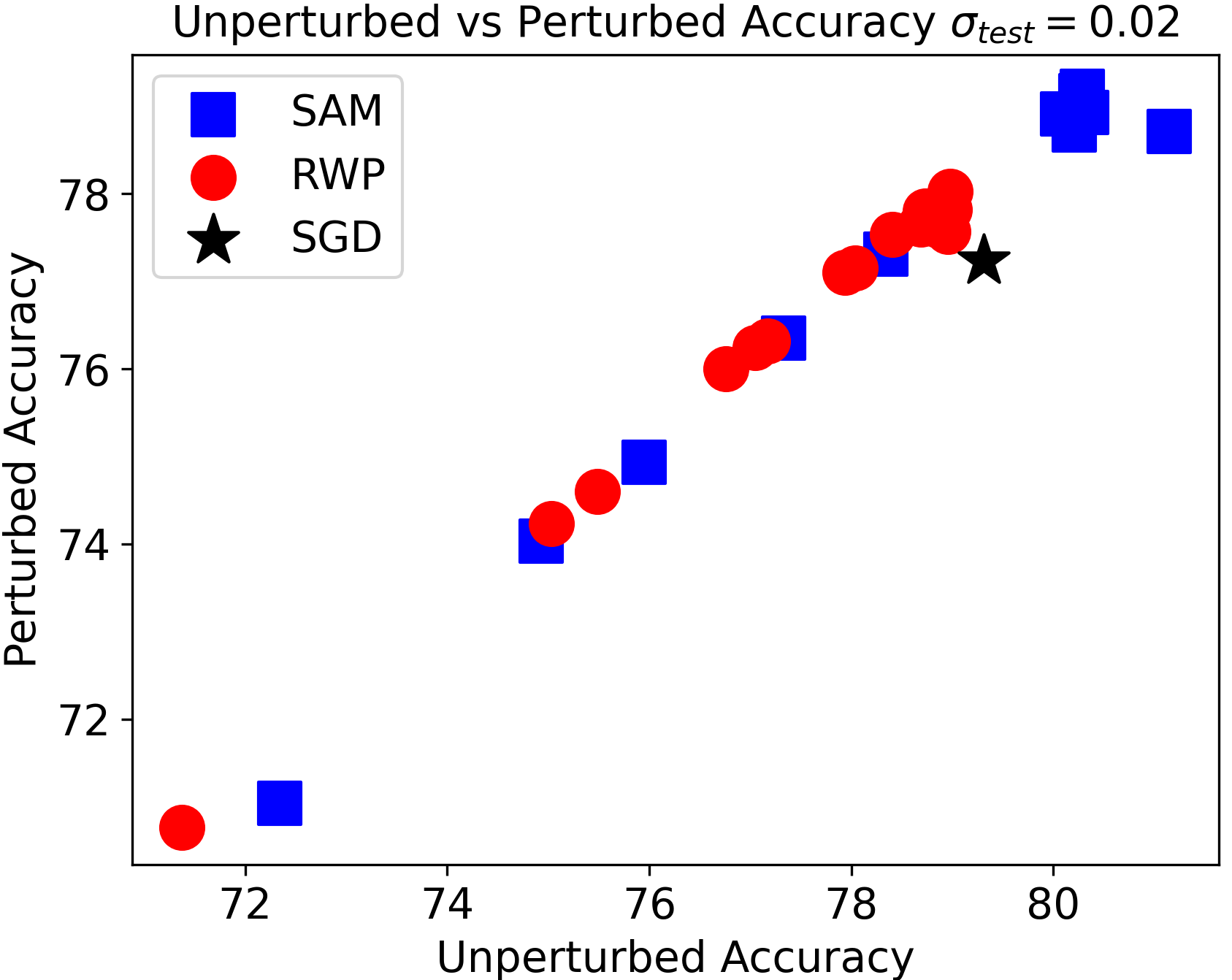}
      \caption{}
      \label{fig:perturb_vs_unperturb_07}
    \end{subfigure}%
    \begin{subfigure}{.5\textwidth}
      \centering
      \includegraphics[width=1.0\linewidth]{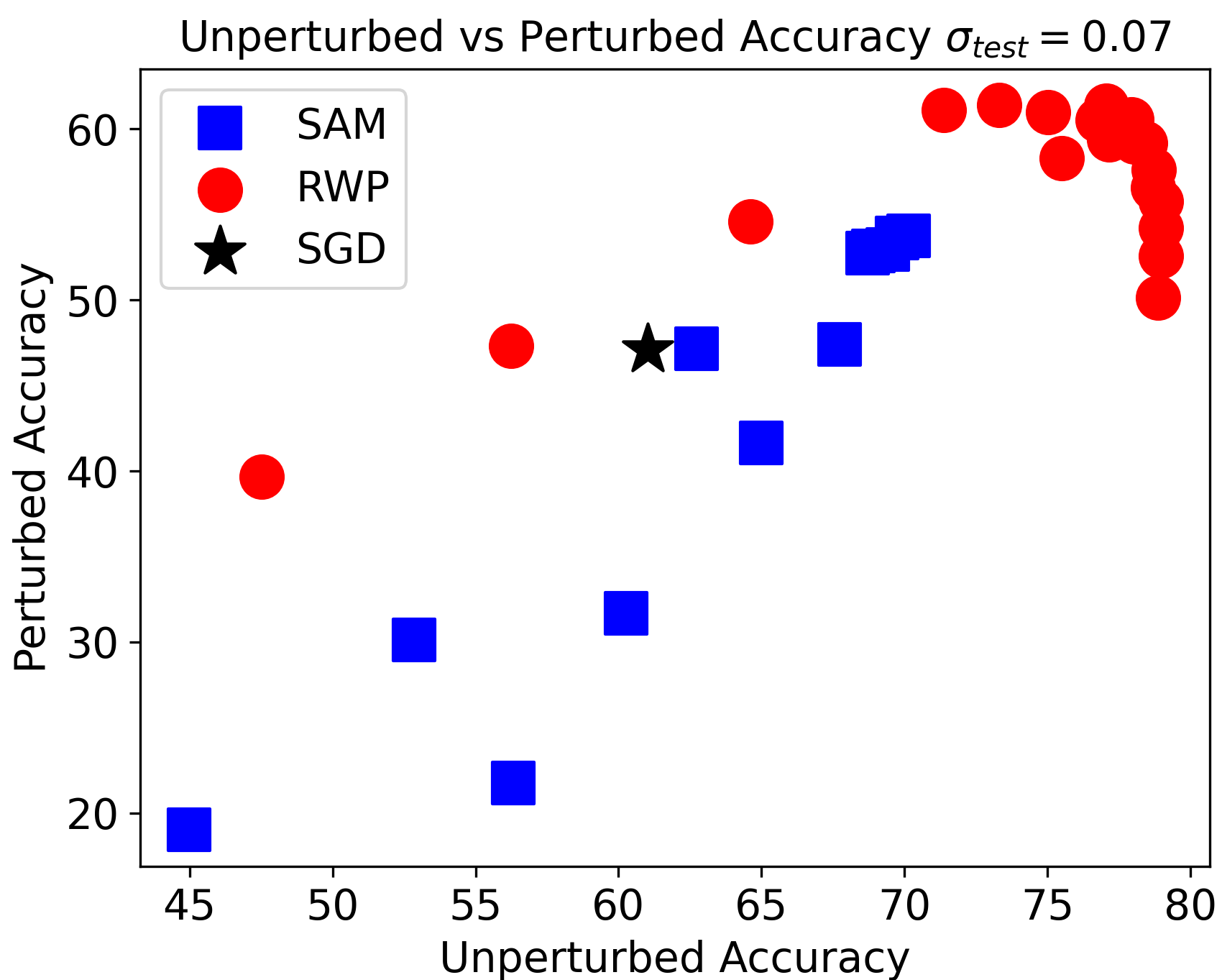}
      \caption{}
      \label{fig:unperturbed_vs_perturbed_02}
    \end{subfigure}
    \caption{(a) Scatter plot of unperturbed test accuracy vs (a) $\sigma_{test}=0.02$ and (b) $\sigma_{test}=0.07$ perturbed test accuracy for varying perturbation sizes of SAM, RWP and SGD. }
    \label{fig:performance_comp}
\end{figure*}


\subsection{Visualizing the Loss Landscape}
\label{sec:flatness_all}
To visually analyze the loss minima found through perturbative training, we produce plots visualizing the Cifar-100 loss landscape (both training and test) of several different minima (using the method proposed in \cite{visualloss}), shown in Fig. \ref{fig:contour_comparison}. In the top row of the figure, (e.g. Figs. \ref{fig:rwp_2e-2_contour},  \ref{fig:rwp_5e-2_contour}, \ref{fig:rwp_7e-2_contour}), we visualize models trained with RWP with increasing $\sigma_{train}$. As expected, the model trained using the larger value of $\sigma_{train}$ produces a visually much flatter minimum, confirming the straightforward correlation we expect. This same pattern holds true for the test landscape minima of RWP, plotted in Figs. \ref{fig:rwp_2e-2_test_contour}, \ref{fig:rwp_5e-2_test_contour}, \ref{fig:rwp_7e-2_test_contour}. Comparing SAM's training minima amongst each other, we see that $\rho=0.3$ appears the flattest, in agreement with our quantitative results. However, comparing the SAM minima against RWP minima breaks the straightforward correlation between training/test loss flatness. For example, the RWP $\sigma_{train}=0.05$ training minimum (Fig. \ref{fig:rwp_5e-2_contour}) is flatter than all of the SAM minima (Fig. \ref{fig:sam_02_contour},\ref{fig:sam_03_contour},\ref{fig:sam_05_contour}); however the test landscape minima for SAM $\rho=0.2$ and $\rho=0.3$ (Fig. \ref{fig:sam_02_test_contour}, \ref{fig:sam_03_test_contour}) are both flatter than the corresponding test minimum for RWP (Fig. \ref{fig:rwp_5e-2_test_contour}). This agrees with our previous results which show that training loss flatness does not correlate with strong noisy generalization, hence the existence of flat training minima which generalize worse than correspondingly sharper training minima.

\begin{figure}
    \begin{subfigure}{.33\textwidth}
      \centering
      \includegraphics[width=1.0\linewidth]{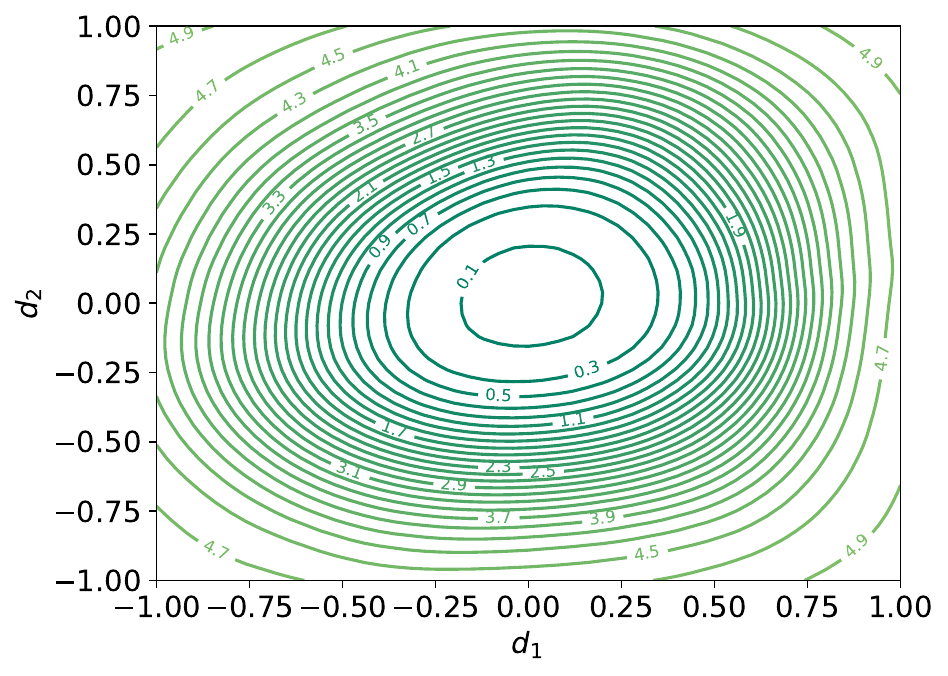}
      \caption{}
      \label{fig:rwp_2e-2_contour}
    \end{subfigure}%
    \begin{subfigure}{.33\textwidth}
      \centering
      \includegraphics[width=1.0\linewidth]{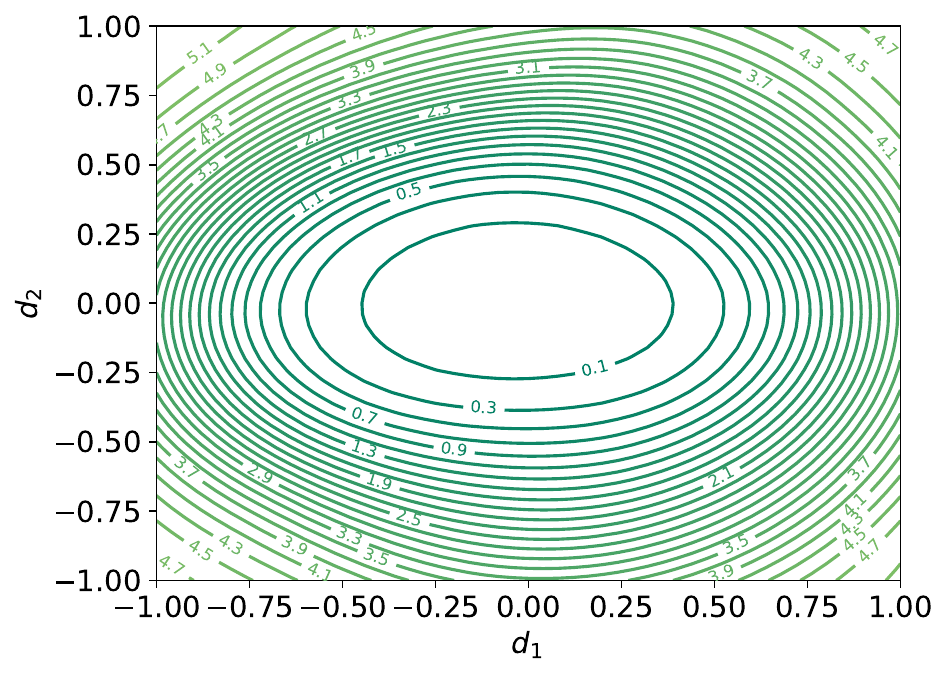}
      \caption{}
      \label{fig:rwp_5e-2_contour}
    \end{subfigure}
    \begin{subfigure}{.33\textwidth}
      \centering
      \includegraphics[width=1.0\linewidth]{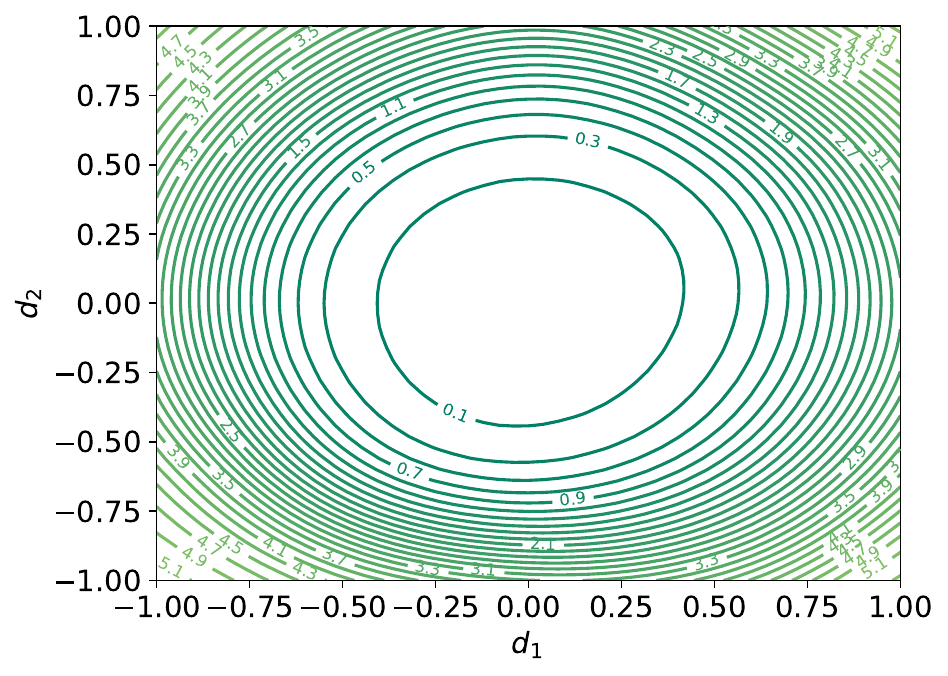}
      \caption{}
      \label{fig:rwp_7e-2_contour}
    \end{subfigure}
        \begin{subfigure}{.33\textwidth}
      \centering
      \includegraphics[width=1.0\linewidth]{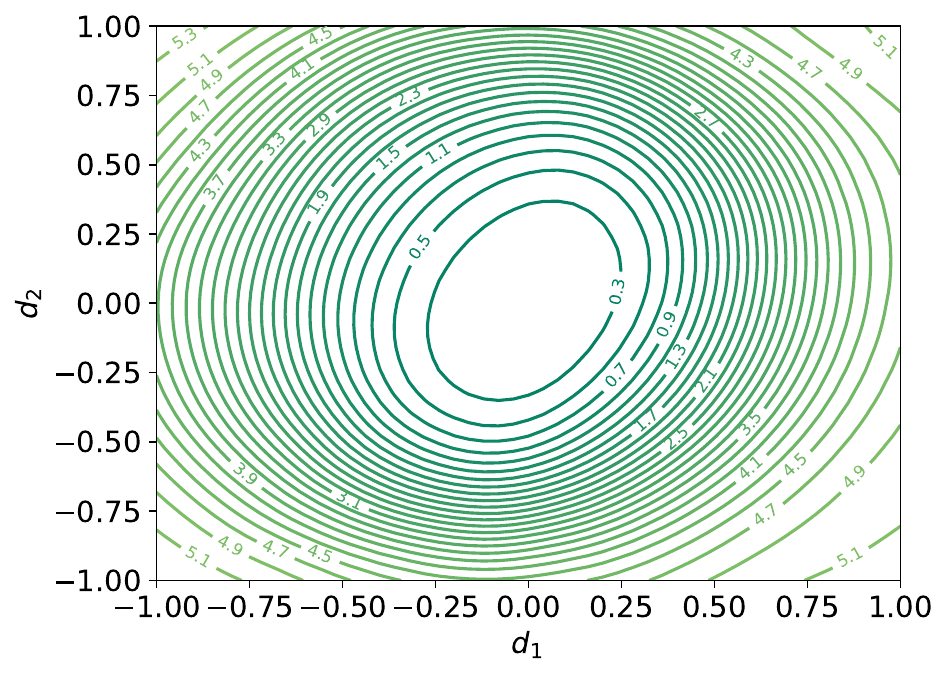}
      \caption{}
      \label{fig:sam_02_contour}
    \end{subfigure}%
    \begin{subfigure}{.33\textwidth}
      \centering
      \includegraphics[width=1.0\linewidth]{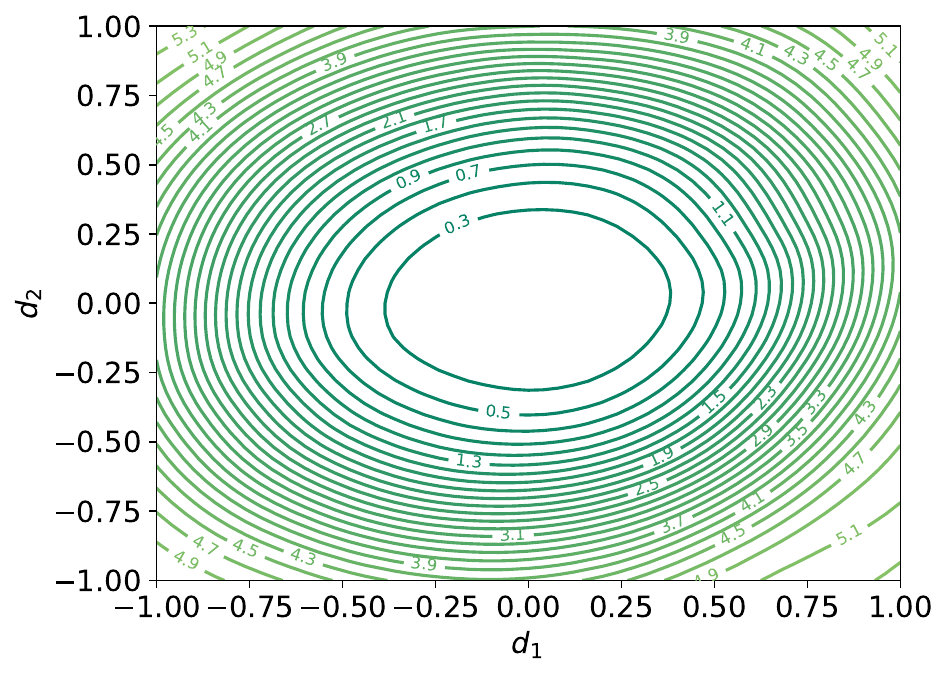}
      \caption{}
      \label{fig:sam_03_contour}
    \end{subfigure}
    \begin{subfigure}{.33\textwidth}
      \centering
      \includegraphics[width=1.0\linewidth]{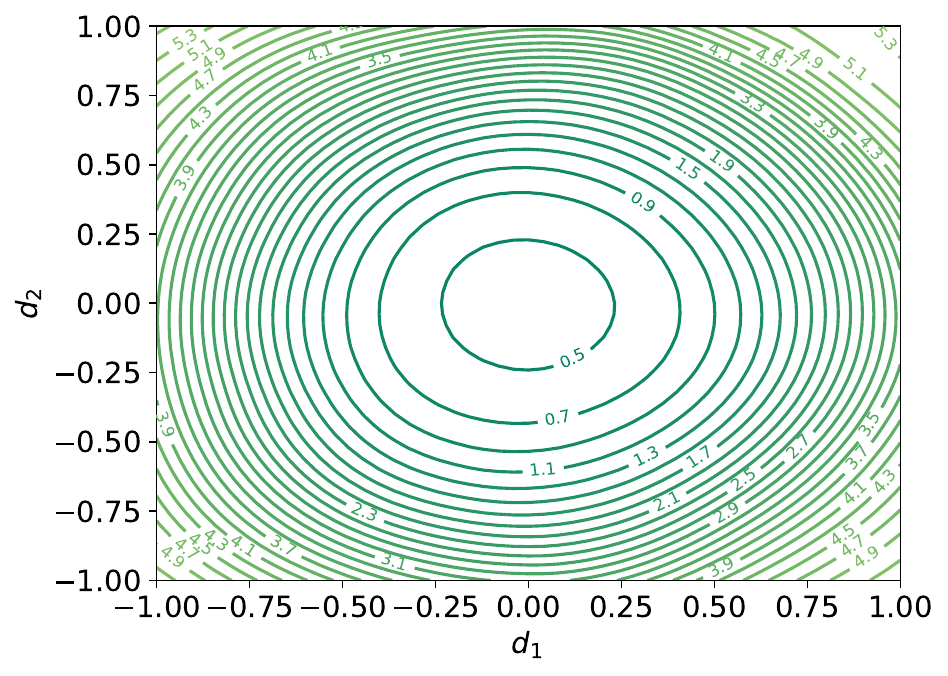}
      \caption{}
      \label{fig:sam_05_contour}
    \end{subfigure}
    \begin{subfigure}{.33\textwidth}
      \centering
      \includegraphics[width=1.0\linewidth]{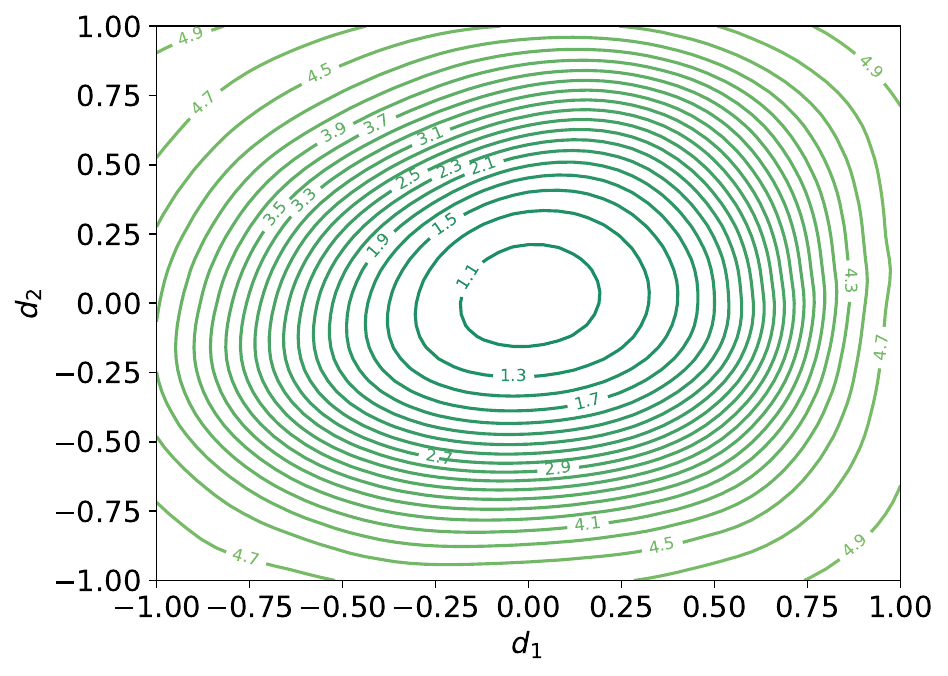}
      \caption{}
      \label{fig:rwp_2e-2_test_contour}
    \end{subfigure}%
    \begin{subfigure}{.33\textwidth}
      \centering
      \includegraphics[width=1.0\linewidth]{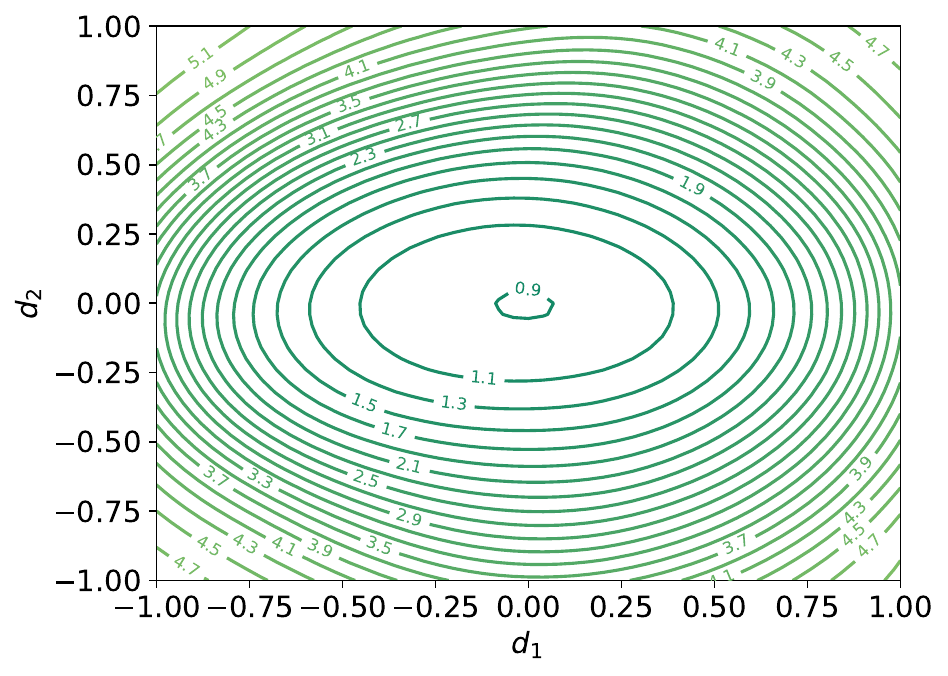}
      \caption{}
      \label{fig:rwp_5e-2_test_contour}
    \end{subfigure}
    \begin{subfigure}{.33\textwidth}
      \centering
      \includegraphics[width=1.0\linewidth]{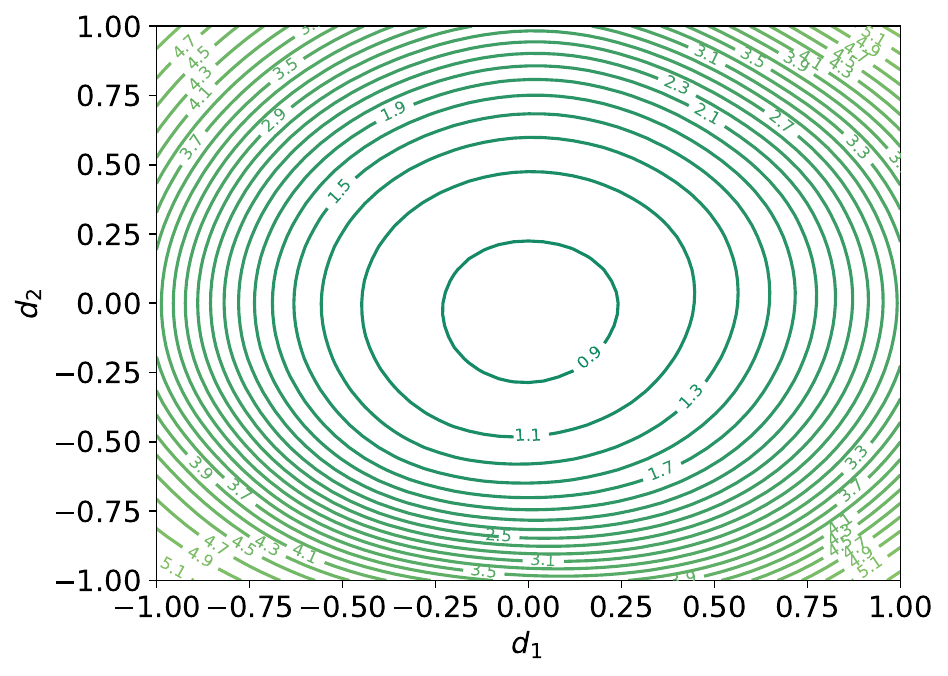}
      \caption{}
      \label{fig:rwp_7e-2_test_contour}
    \end{subfigure}
        \begin{subfigure}{.33\textwidth}
      \centering
      \includegraphics[width=1.0\linewidth]{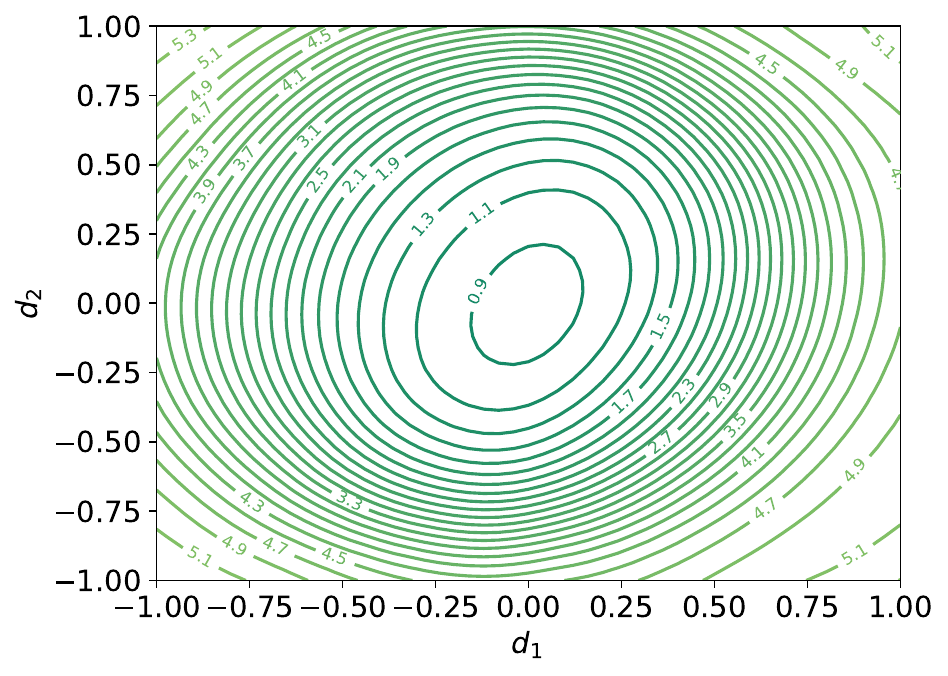}
      \caption{}
      \label{fig:sam_02_test_contour}
    \end{subfigure}%
    \begin{subfigure}{.33\textwidth}
      \centering
      \includegraphics[width=1.0\linewidth]{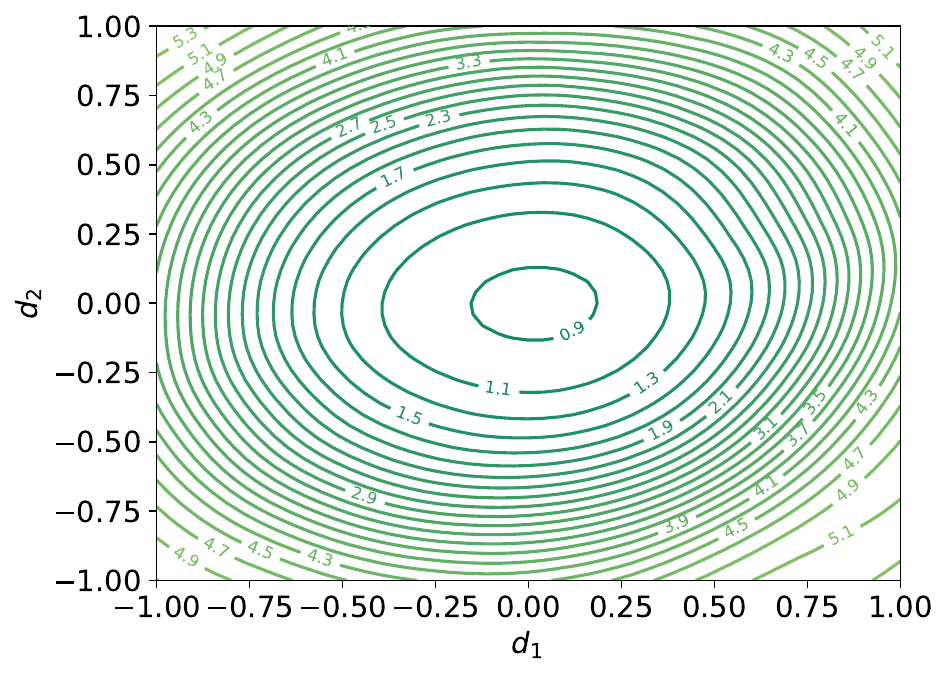}
      \caption{}
      \label{fig:sam_03_test_contour}
    \end{subfigure}
    \begin{subfigure}{.33\textwidth}
      \centering
      \includegraphics[width=1.0\linewidth]{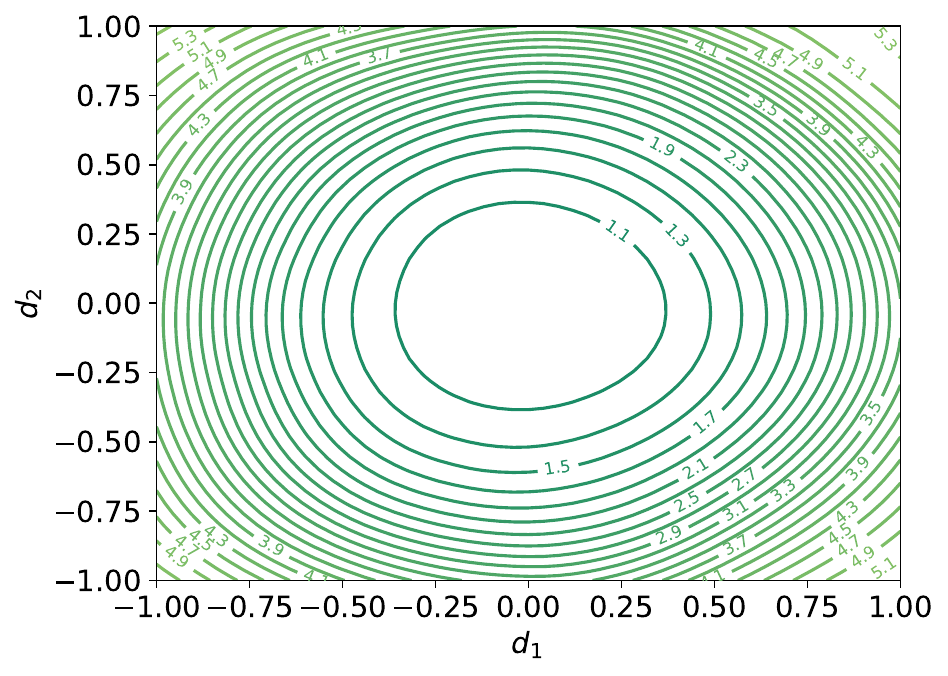}
      \caption{}
      \label{fig:sam_05_test_contour}
    \end{subfigure}
  \caption{2D visualization along filter-normalized directions of the (unperturbed) Cifar-100 loss landscape of minima found using  RWP $\sigma_{train} = 0.02$, $\sigma_{train} = 0.05$, $\sigma_{train} = 0.07$ ((a)-(c) for training landscape visualization, (g)-(i) for test landscape visualization), and  SAM $\rho=0.2$, $\rho=0.3$, $\rho=0.5$ ((d)-(f) for training landscape visualization, (j)-(l) for test landscape visualization).}
  \label{fig:contour_comparison}
\end{figure}

\subsection{Quantifying Attenuated Training}
To confirm the effect of the flatness-induced vanishing gradient during training, we plot the total distance traveled in the parameter space as training on Cifar-100 unfolds $\sum_{t=1}^N ||w_{t+1}-w_t||$ (where $t$ corresponds to epoch number), shown in Fig. \ref{fig:total_distance}. Compared to SGD, we see that all of the perturbed training trajectories display both a smaller gradient norm and smaller distance traveled. For the case of SAM, we see a dramatically smaller distance traveled as compared to the RWP models, demonstrating that the diminished gradient magnitude indeed attenuates training.

\begin{figure}
      \centering
      \includegraphics[scale=0.5]{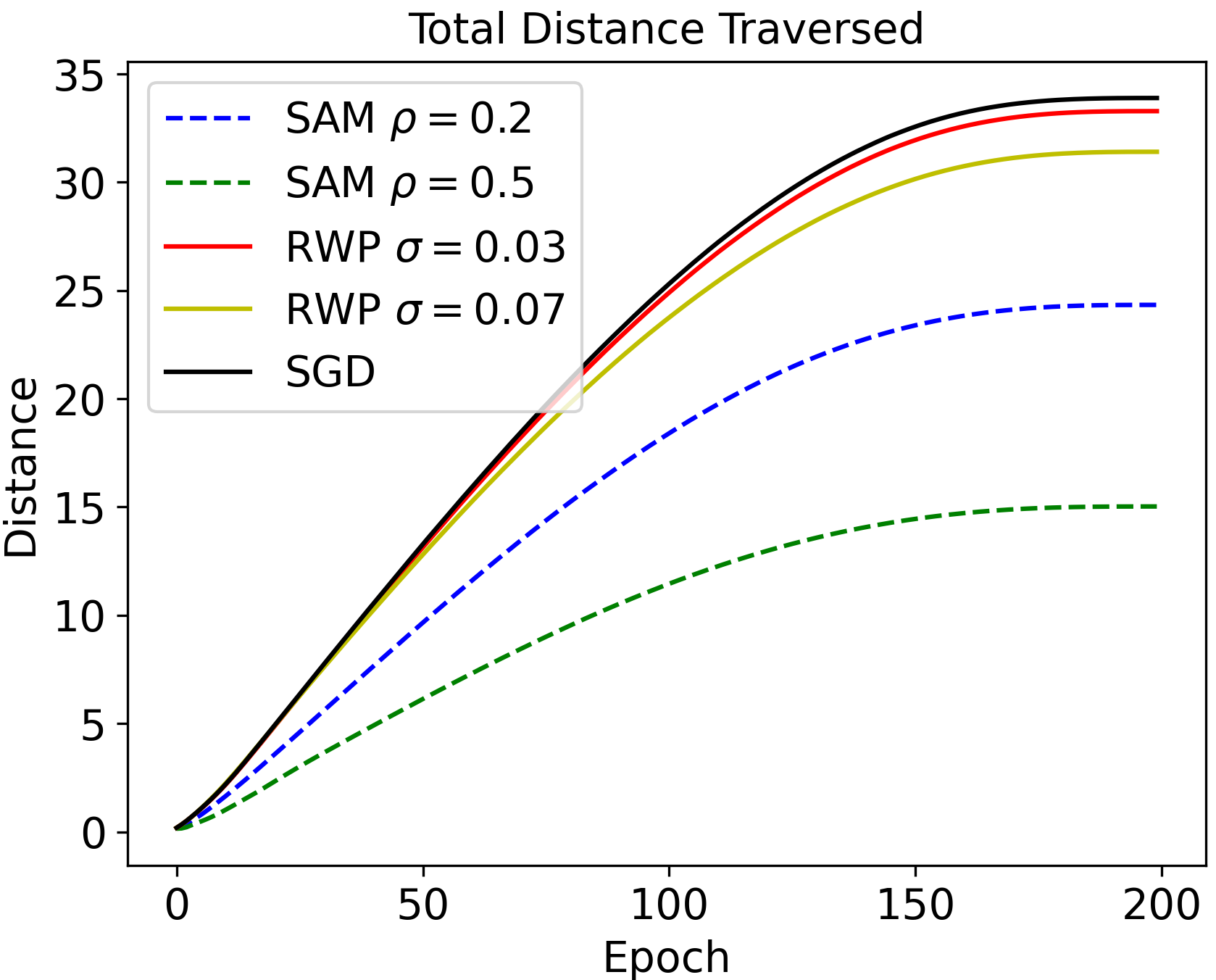}
    \caption{Plot of the total distance traversed during training, $\sum_{i=1}^N ||w_{t+1}-w_t||$}
    \label{fig:total_distance}
\end{figure}

\subsection{Comparing Training and Test Sharpness During Training}
As an additional perspective from which to evaluate noisy generalization, we also directly track the evolution of sharpness during optimization. In Fig. \ref{fig:test_sharpness} and Fig. \ref{fig:avg_sharpness_comp}, we plot average-direction training sharpness as a function of training loss, average direction test sharpness as a function of test loss, and test sharpness as a function of training sharpness ($\sigma = 0.02$ and $\sigma = 0.07$, respectively) for several SAM and RWP training trajectories. From the figures, we highlight a few observations: first, we note that a mininum's test sharpness is almost always greater than its training sharpness, made especially clear in Fig. \ref{fig:sharpness_training_vs_test_02} and Fig. \ref{fig:sharpness_training_vs_test_07}. Second, we observe that, particularly toward the later training epochs, the evolution of training and test sharpness are in fact anti-correlated (i.e. test sharpness \textit{increases} as training sharpness \textit{decreases}). Lastly, we see strong evidence against the direct correlation of training and test sharpness (and by extension generalization). For example, we see in the plot of training sharpness in Fig. \ref{fig:sharpness_07} that SAM exhibits virtually the same degree of flatness as SGD over its entire trajectory. However, when plotting the test sharpness of the same trajectories, we see that in the earlier epochs SAM traverses a test loss basin that is significantly less sharp than that of SGD; it is within this basin that SAM (with early stopping) finds weights with improved noise-resilience.

\begin{figure*}
    \begin{subfigure}{.33\textwidth}
        \centering
        \includegraphics[width=1.0\linewidth]{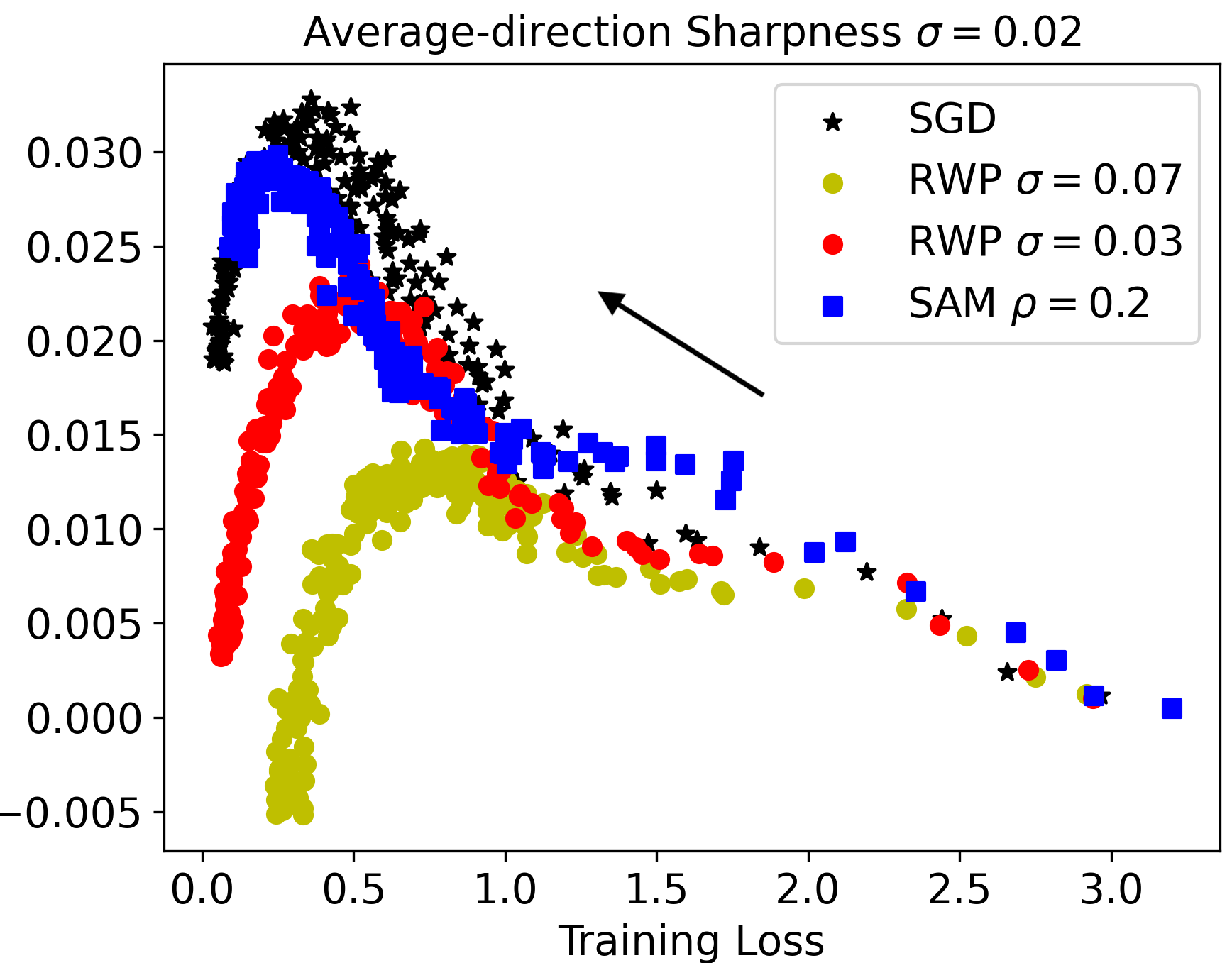}
        \caption{}
        \label{fig:sharpness_02}
    \end{subfigure}%
    \begin{subfigure}{.33\textwidth}
      \centering
      \includegraphics[width=1.0\linewidth]{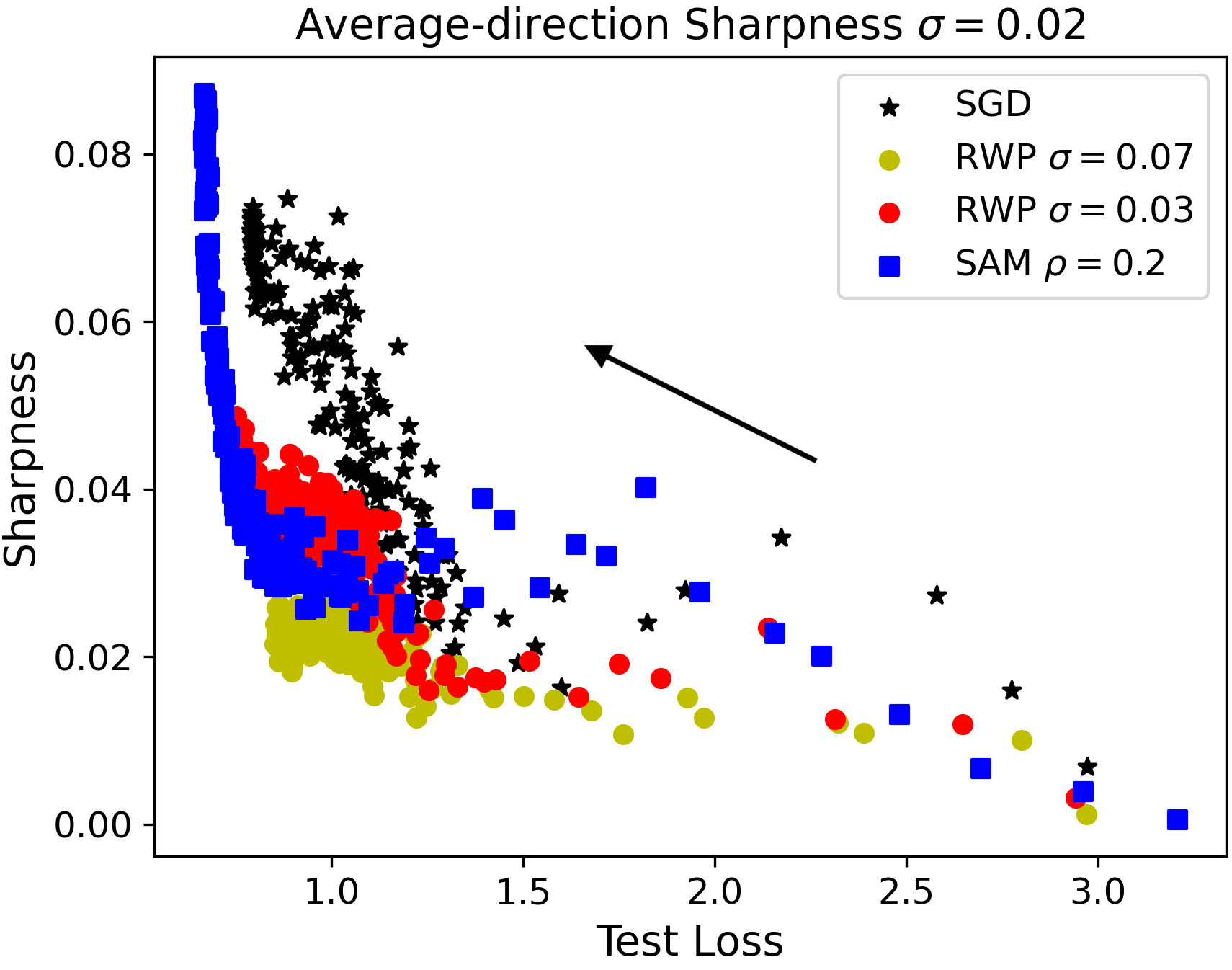}
      \caption{}
      \label{fig:test_sharpness_02}
    \end{subfigure}
    \begin{subfigure}{.33\textwidth}
      \centering
      \includegraphics[width=1.0\linewidth]{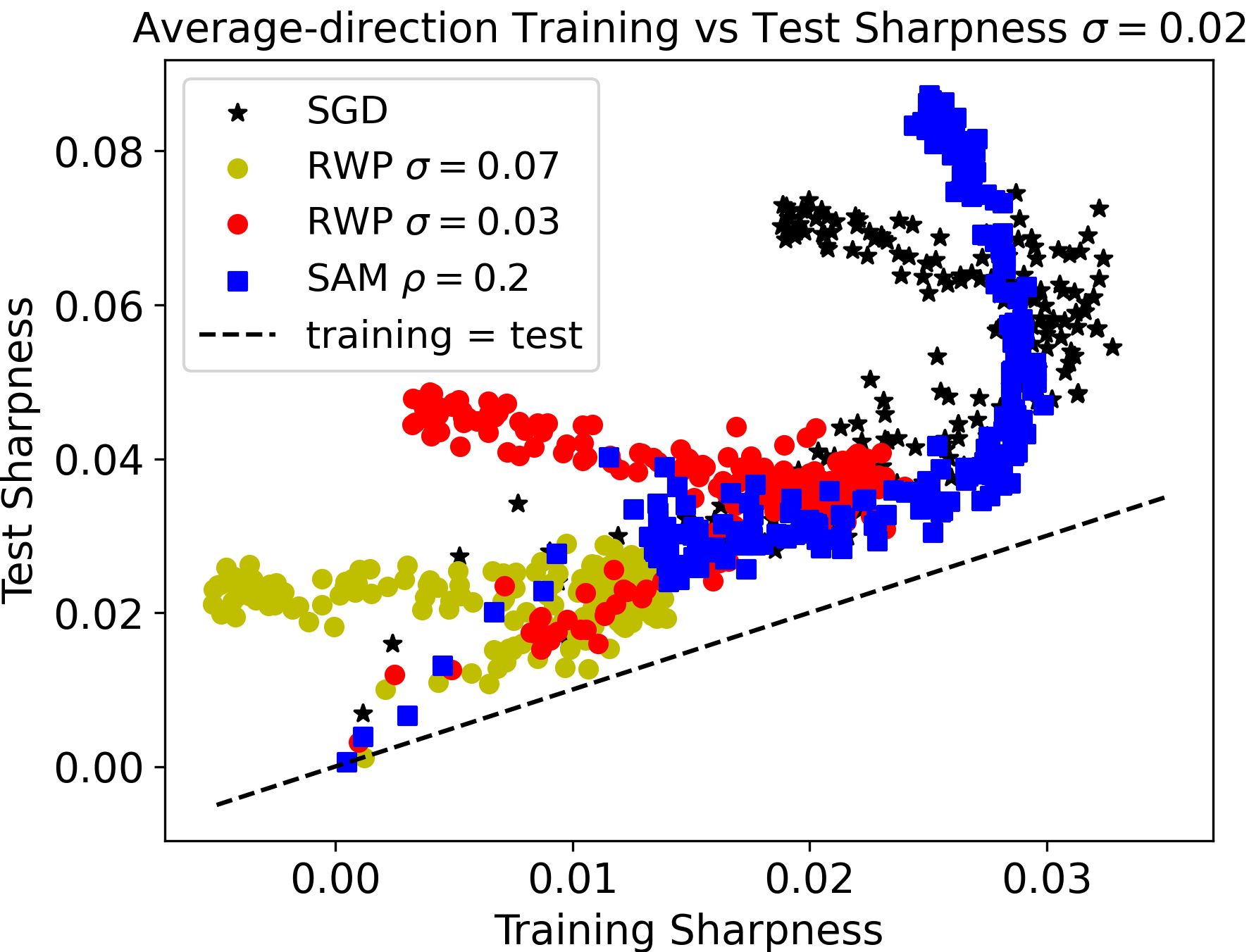}
      \caption{}
      \label{fig:sharpness_training_vs_test_02}
    \end{subfigure}
    \caption{ (a) Plot of ResNet-18 Cifar-100 average-direction sharpness for $\sigma = 0.02$. as a function of training loss. (b) Plot of average-direction sharpness as a function of \textit{test} loss for $\sigma=0.02$. (c) Plot of \textit{test} sharpness as a function of \textit{training} for $\sigma=0.02$.}
    \label{fig:test_sharpness}
\end{figure*}

\begin{figure}
        \begin{subfigure}{.33\textwidth}
        \centering
        \includegraphics[width=1.0\linewidth]{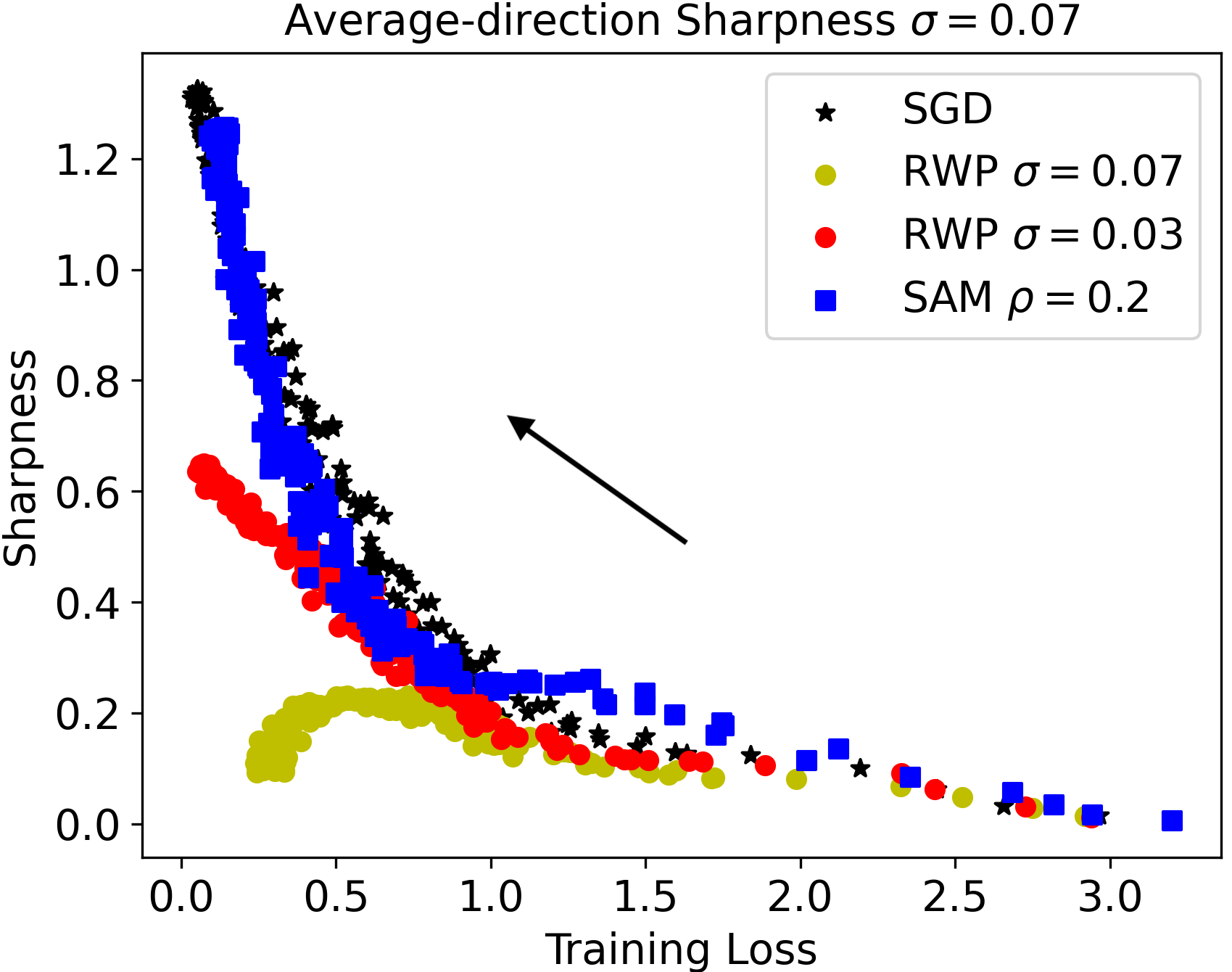}
        \caption{}
        \label{fig:sharpness_07}
    \end{subfigure}
    \begin{subfigure}{.33\textwidth}
      \centering
      \includegraphics[width=1.0\linewidth]{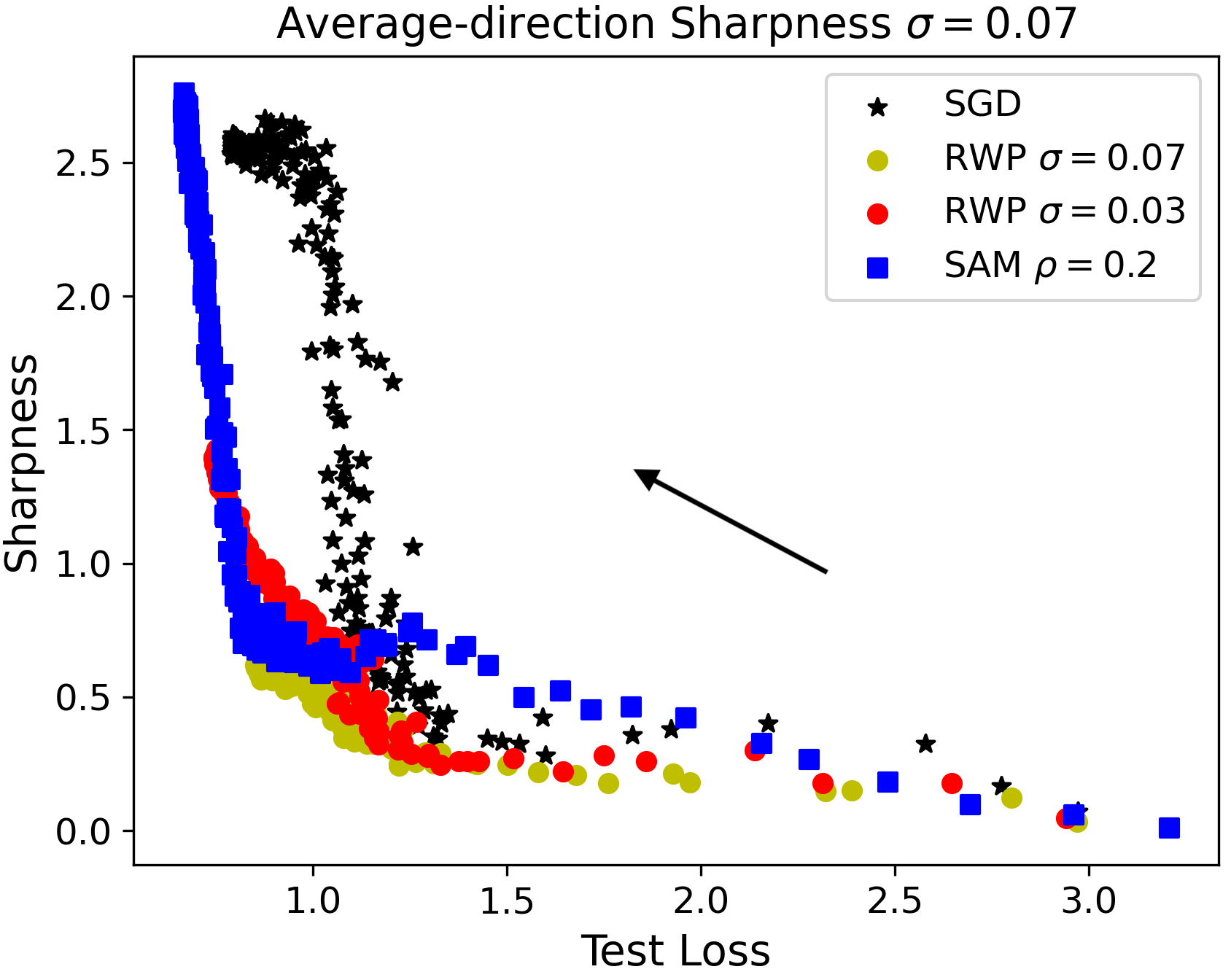}
      \caption{}
      \label{fig:test_sharpness_07}
    \end{subfigure}
    \begin{subfigure}{.33\textwidth}
      \centering
      \includegraphics[width=1.0\linewidth]{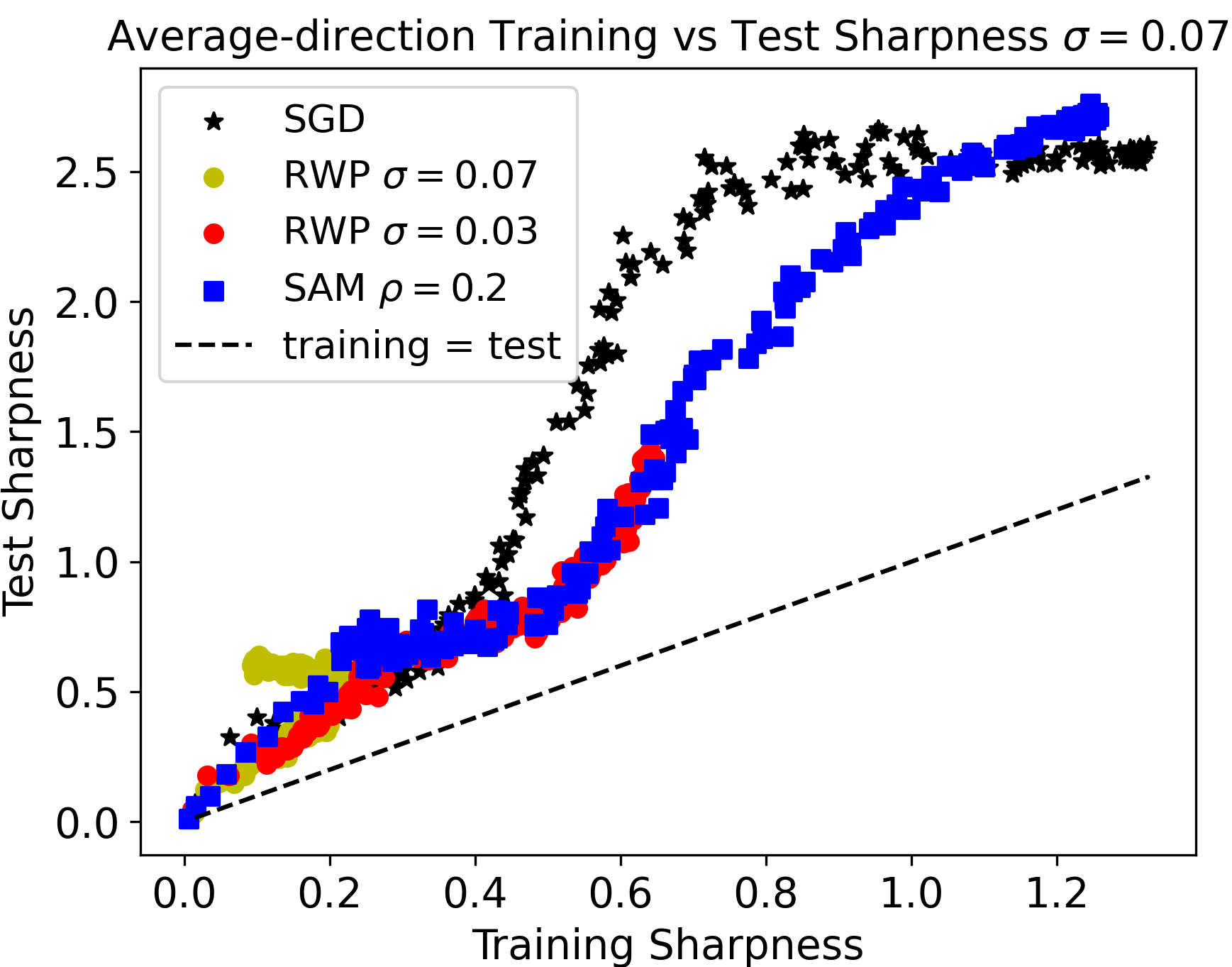}
      \caption{}
      \label{fig:sharpness_training_vs_test_07}
    \end{subfigure}
    \caption{(a) Plot of average-direction sharpness for $\sigma = 0.07$ for a ResNet-18 training on Cifar-100. (b) Plot of average-direction sharpness as a function of \textit{test} loss for $\sigma=0.07$. (c) Plot of test sharpness as a function of training sharpness. }
    \label{fig:avg_sharpness_comp}
\end{figure}

\section{Does Noise Distribution Matter?}
In all of our previous analysis, we have made the assumption of matching RWP's training noise distribution class with that of our test noise distribution $\mathbb{Q}$, in this case a normal distribution. This naturally raises two questions: first, does over-regularization's generalization benefit hold only for this specific perturbation pattern, or can is it true more generally? Second, is RWP's advantage over SAM partially attributed to knowledge of the shape of the test distribution, an advantage that SAM does not benefit from? To address these questions, we perform experiments in which we define $\mathbb{Q}\sim Laplace(0,b\max_i|w_i|)$, and compare the performance of SAM, RWP with a Gaussian perturbation distribution, and RWP with the same Laplace distribution. We present these results in Tab. \ref{table:laplace_noise}. Comparing SAM's performance with that of Gaussian RWP, we observe that the same general trends still hold: in the small noise regime, SAM is optimal, whereas in the large noise regime, RWP is optimal. Directly comparing Gaussian RWP with Laplace RWP, we see that the performance of Laplace RWP is marginally better than that of the Gaussian version, indicating that prior knowledge of the test distribution only modestly impacts training. Lastly, we see that over-regularized RWP is in fact optimal across all the noise settings, further evidence that the benefit of over-regularization holds more generally.

\begin{table}
  \caption{Comparison on Cifar-100 of SGD, SAM and RWP-trained ResNet-18 when test-time noise follows a Laplace distribution $Laplace(0,b\max_i|w_i|)$.}
  \label{table:laplace_noise}
  \resizebox{\textwidth}{!}{\begin{tabular}{c|c|c|c|c}
    \hline
     &  $b = 0.01$  &   $b = 0.02$ &  $b = 0.03$ &  $b = 0.04$\\
    \hline
    SGD & $78.49 \pm 0.18 \pm 0.07$ & $73.42 \pm 2.59 \pm 0.39$ & $59.66 \pm 1.74 \pm 1.22$ & $54.08 \pm 2.05 \pm 1.54$ \\
    \hline
    SAM $\rho = 0.2$ & $79.52 \pm 0.15 \pm 0.14$ & $77.05 \pm 0.38 \pm 0.16$ & $67.85 \pm 2.76 \pm 0.55$ & $61.11 \pm 1.63 \pm 1.40$ \\
    SAM $\rho = 0.3$ & $\mathbf{79.75} \pm 0.14 \pm 0.23$ & $\mathbf{77.66} \pm 0.78 \pm 0.59$ & $69.47 \pm 1.60 \pm 0.97$ & $62.25 \pm 2.40 \pm 0.82$ \\
    SAM $\rho = 0.5$ & $77.95 \pm 0.16 \pm 0.08$ & $ 75.41 \pm 2.02 \pm 0.78$ & $66.72 \pm 1.58 \pm 0.43$ & $58.35 \pm 5.19 \pm 2.01$ \\
    \hline
    RWP $\sigma = 0.03$ & $78.30 \pm 0.15 \pm 0.25$ & $76.32 \pm 0.42 \pm 0.27$ & $72.05 \pm 1.07 \pm 0.41$&$ 63.18 \pm 2.71 \pm 1.18$ \\
    RWP $\sigma = 0.05$ & $ 78.28 \pm 0.14 \pm 0.08$ &$76.66 \pm 0.35 \pm 0.05$ & $73.53 \pm 0.72 \pm 0.14$&$ 67.40 \pm 1.66 \pm 0.34$ \\
    RWP $\sigma = 0.07$ & $76.82 \pm 0.14 \pm 0.17$ &$75.56 \pm 0.34 \pm 0.13$ & $72.95 \pm 0.66 \pm 0.23$ & $68.00 \pm 1.53 \pm 0.64$\\
    \hline
    RWP $b = 0.01$ & $ 78.45 \pm 0.21 \pm 0.11$ &$75.83 \pm 0.60 \pm 0.35$ & $70.13 \pm 1.44 \pm 1.08$ & $57.69 \pm 3.06 \pm 2.72$\\
    RWP $b = 0.02$ & $78.67 \pm 0.15 \pm 0.24$ &$76.67 \pm 0.40 \pm 0.21$ & $72.51 \pm 0.81 \pm 0.26$ & $64.34 \pm 1.81 \pm 0.44$\\
    RWP $b = 0.03$ & $78.42 \pm 0.15 \pm 0.06$ &$ 76.79\pm 0.36 \pm 0.10$ & $ 73.31\pm 0.88 \pm 0.10$ & $66.52 \pm 2.06 \pm 0.17$\\
    RWP $b = 0.04$ & $77.79 \pm 0.15 \pm 0.29$ &$ 76.38\pm 0.30 \pm 0.36$ & $ \mathbf{73.60}\pm 0.61 \pm 0.52$ & $\mathbf{68.27} \pm 1.34 \pm 0.87$\\
    \hline
  \end{tabular}}
\end{table}

\section{Experiments on Differing Backbone Models}
\label{sec:backbones}

\subsection{Evaluating Generality of SAM/RWP Performance}
To confirm that our observations of SAM/RWP hold true generally, we perform similar experiments as in Sec. \ref{sec:selecting_training_perturbation} on a variety of ResNets and WRNs, shown in Tab. \ref{table:depth_comp} and Tab. \ref{table:width_comp}. For a fair comparison, we replace the bottleneck block in the ResNet-50 with the same basic block architecture used in the shallower ResNets. As is the case with ResNet-18, we find that for $\sigma_{test}=0.02$, SAM-trained models achieve higher accuracy, whereas for $\sigma_{test}=0.07$, RWP-trained models dominate. We additionally recreate Fig. \ref{fig:understanding_loss_landscape} for training with a ResNet-50, demonstrating the same vanishing gradient effect observed in the ResNet-18.

\begin{figure}
    \begin{subfigure}{.5\textwidth}
        \centering
        \includegraphics[width=1.0\linewidth]{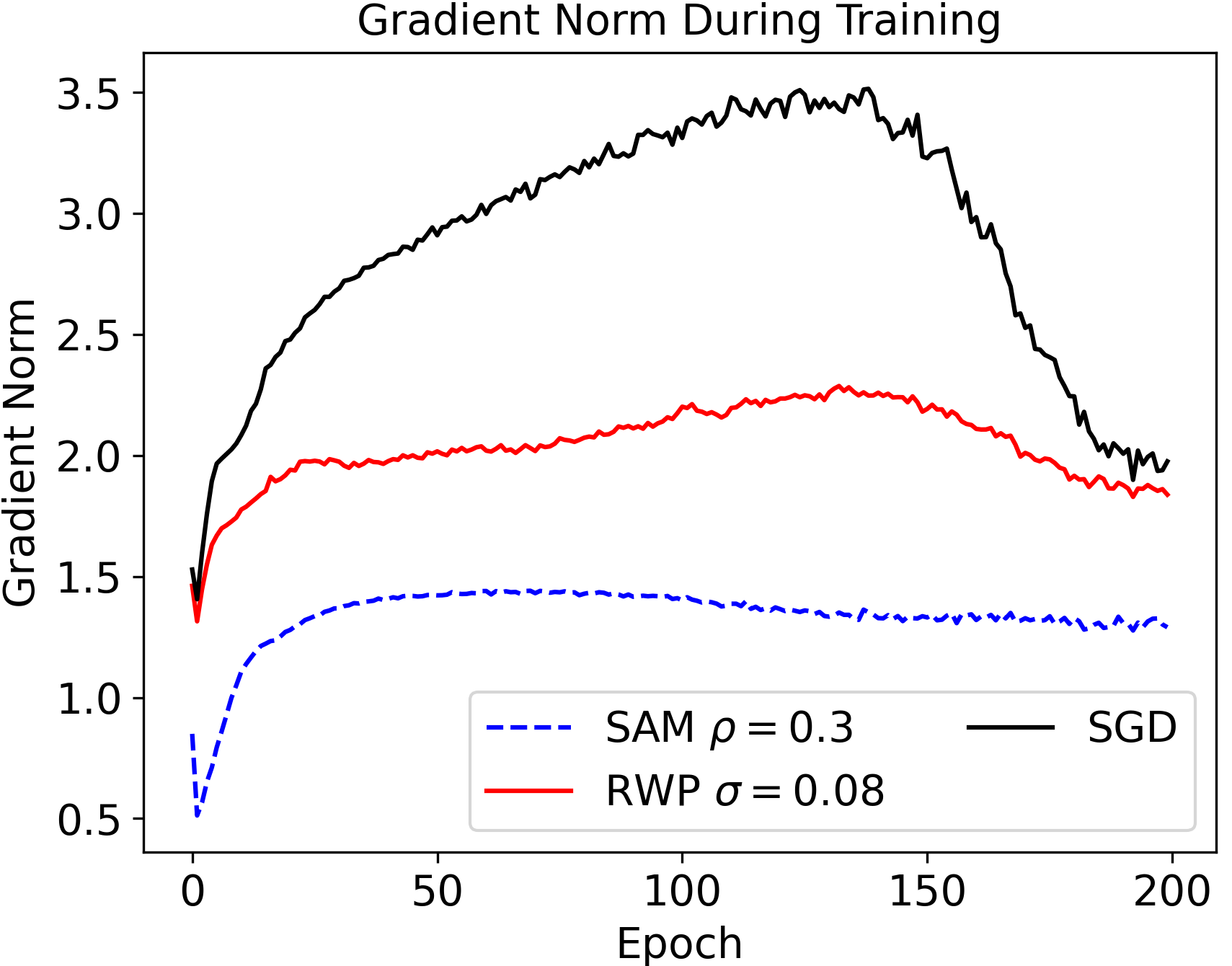}
        \caption{}
        \label{fig:grad_norm_resnet_50}
    \end{subfigure}%
    \begin{subfigure}{.5\textwidth}
      \centering
      \includegraphics[width=1.0\linewidth]{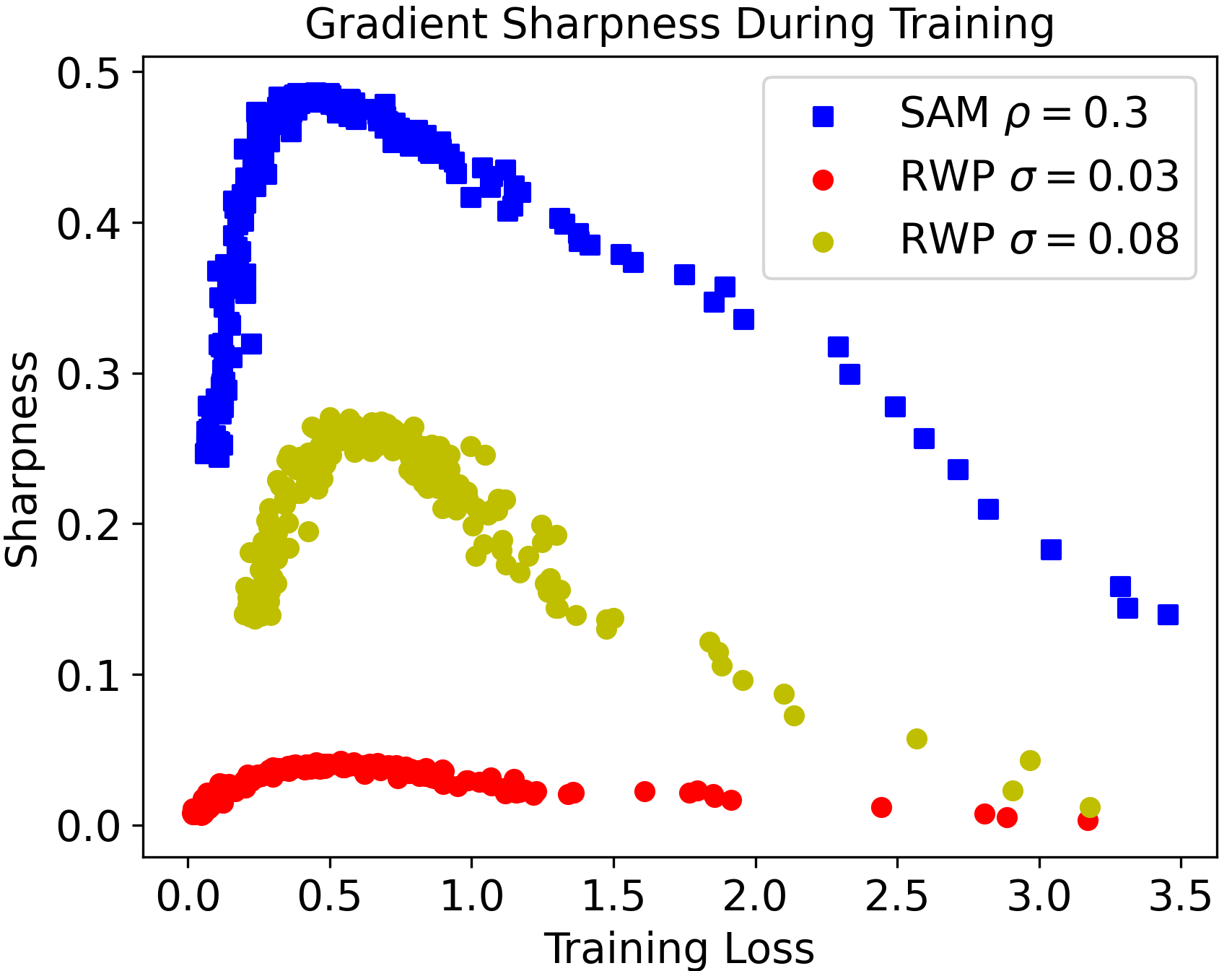}
      \caption{}
      \label{fig:grad_sharpness_resnet_50}
    \end{subfigure}
    \caption{Plots of (a) gradient norm and  (b) gradient sharpness for ResNet-50 during Cifar-100 training.}
    \label{fig:resnet_50_comp}
\end{figure}

\begin{table}
  \caption{Comparison of RWP and SAM for ResNets of varying depths. The best-performing optimizer for a given architecture is bolded, whereas the best-performing architecture for a given optimizer is underlined.}
  \label{table:depth_comp}
  \resizebox{\textwidth}{!}{\begin{tabular}{c|c|c|c|c}
    \hline
     &  \multicolumn{2}{|c|}{$\sigma_{test} = 0.02$}  &   \multicolumn{2}{|c}{$\sigma_{test} = 0.07$} \\
    \hline
    & SAM & RWP & SAM & RWP\\
    \hline
    ResNet-9 & $\mathbf{74.14} \pm 0.94\pm 1.67$ & $ 72.70 \pm 0.31 \pm 0.07$ & $ 39.58 \pm 5.03 \pm 0.97$  & $ \mathbf{51.10} \pm 3.14 \pm 1.70$\\
    ResNet-18 & $\mathbf{79.12} \pm 0.23 \pm 0.28$  & $77.65 \pm 0.26 \pm 0.05$ & $53.61 \pm 2.61 \pm 0.68$  & $\mathbf{61.36} \pm 2.85  \pm  0.90$\\
    ResNet-34 &$ \mathbf{80.82}\pm 0.24 \pm 0.18$ & $ 79.19 \pm 0.21 \pm 0.19$  & $ \underline{55.38} \pm 3.23 \pm 1.67$  & $\mathbf{64.67} \pm 2.54 \pm 0.54$\\
    ResNet-50 & $\underline{\mathbf{81.29}} \pm 0.17 \pm 0.17$ & $\underline{79.81} \pm 0.19 \pm 0.17$ & $54.22 \pm 3.39 \pm 2.85$  & $ \underline{\mathbf{64.87}} \pm 2.28 \pm 1.23$\\
    \hline
  \end{tabular}}
\end{table}

\begin{table}
  \caption{Comparison of RWP and SAM for ResNets of varying widths. The best-performing optimizer for a given architecture is bolded, whereas the best-performing architecture for a given optimizer is underlined.}
  \label{table:width_comp}
  \resizebox{\textwidth}{!}{\begin{tabular}{c|c|c|c|c}
    \hline
     &  \multicolumn{2}{|c|}{$\sigma_{test} = 0.02$}  &   \multicolumn{2}{|c}{$\sigma_{test} = 0.07$} \\
    \hline
    & SAM & RWP & SAM & RWP\\
    \hline
    ResNet-18 & $\mathbf{79.12} \pm 0.23 \pm 0.28$  & $77.65 \pm 0.26 \pm 0.05$ & $53.61 \pm 2.61 \pm 0.68$  & $\mathbf{61.36} \pm 2.85  \pm  0.90$\\
    WRN-16-5 & $\mathbf{80.75} \pm 0.31 \pm 0.11$  & $79.36 \pm 0.27 \pm 0.06$ & $51.92 \pm 4.79 \pm 0.69$  & $ \mathbf{59.11} \pm 3.48 \pm 0.18$\\
    WRN-16-10 & $\mathbf{83.13} \pm 0.23 \pm 0.32$  & $81.08 \pm 0.21 \pm 0.28$ & $\underline{60.11} \pm 3.04 \pm 2.31$  & $\underline{\mathbf{66.23}} \pm 2.26  \pm 0.75$\\
    WRN-16-15 & $ \underline{\mathbf{83.53}} \pm 0.18 \pm 0.09$ & $ \underline{81.24} \pm 0.21 \pm 0.08$ & $ 59.55 \pm 3.78 \pm 2.38$ & $\mathbf{64.63} \pm 3.96 \pm 0.83$\\
    \hline
  \end{tabular}}
\end{table}

\subsection{Dynamic Perturbation Schedules}
To verify the generality of our dynamic perturbation schedules, we also perform experiments using different backbone models, specifically ResNet-50, WRN-16-10, and PyramidNet-50 \cite{DPRN} shown in Tab. \ref{table:schedules_deep_v_wide}. Here, we note a contrast between the WRN and the ResNet-50/PyramidNet-50: while the perturbation schedule enhances noise robustness for both network architectures, the improvement is markedly larger for the latter two networks. We conclude that ResNet-50 and PyramidNet-50 naturally exhibit a sharper loss landscape that is a significant factor in the degraded performance of SAM/RWP; as a result, the ramped perturbations improve performance by a large amount. Meanwhile, the naturally smooth loss landscape of the WRN leads to a fairly stable training process, even without the perturbation schedules. As a result, the improvement from applying the schedule is much more modest, in-line with performance gains in the small-noise setting applied to the narrower networks.

\begin{table}
  \caption{Comparison of perturbed test accuracy when perturbation schedules are applied during training for differing backbone architectures.}
  \label{table:schedules_deep_v_wide}
  \centering
  \begin{tabular}{c|c|c|c}
    \hline
     &   & \multicolumn{2}{|c}{$\sigma_{test} = 0.07$} \\
    \hline
    & Schedule & SAM & RWP\\
    \hline
    ResNet-50 & Constant & $54.22 \pm 3.39 \pm 2.85$  & $ 64.87 \pm 2.28 \pm 1.23$ \\
    ResNet-50 & Quadratic &$\mathbf{59.18} \pm 3.23 \pm 2.69$ &  $\mathbf{69.35} \pm 1.29 \pm 0.35$ \\
    \hline
    WRN-16-10 & Constant & $60.11 \pm 3.04 \pm 2.31$  & $66.23 \pm 2.26  \pm 0.75$ \\
    WRN-16-10 & Quadratic &$\mathbf{61.60} \pm 3.19 \pm 0.49$  & $ \mathbf{67.19} \pm 2.28 \pm 0.90 $ \\
    \hline
    PyramidNet-50 & Constant & $47.40 \pm 3.49 \pm 1.25$  & $58.89 \pm 3.23  \pm 2.71$ \\
    PyramidNet-50 & Quadratic &$\mathbf{52.95} \pm 3.48 \pm 2.69$  & $ \mathbf{63.88} \pm 1.73 \pm 1.22 $ \\
    \hline
  \end{tabular}
\end{table}

\section{Tiny-Imagenet and ImageNet-100 Experiments}
\label{sec:tiny-imagenet}

As further evidence for the generality of our results, we perform experiments on both the intermediate-scale Tiny-ImageNet dataset and the large-scale ImageNet-100 dataset. First, we explore optimization effects, plotting the gradient norm and gradient sharpness evolution during training on both Tiny-ImageNet and ImageNet-100, shown in Fig. \ref{fig:tiny_imagenet_comp} and Fig. \ref{fig:imagenet_100_comp}, respectively. As in the case of Cifar-100, we see the same evidence for the flatness-induced vanishing gradient, which manifests in a more pronounced way for SAM training. Next, we perform experiments to understand generalization effects, evaluating the noisy generalization of SGD, SAM and RWP across a variety of noise settings on Tiny-ImageNet, shown in Tab. \ref{table:tiny_imagenet_generalization}. In doing so, we remark on several trends: first, we consistently see that an over-regularized RWP training objective improves generalization, as $\sigma_{train}=0.08$ is the best-performing RWP model across all the noise settings. Second, we again see that SAM generalizes well in the low-noise case, but cannot scale to increasing perturbation strengths. Examining the effects of applying dynamic perturbation schedules, we see that both SAM and RWP's performance are broadly improved- notably, we see that ramping perturbations allows both optimizers to converge to lower-loss training minima, confirming the optimization benefit this confers.
\begin{figure}
    \begin{subfigure}{.5\textwidth}
        \centering
        \includegraphics[width=1.0\linewidth]{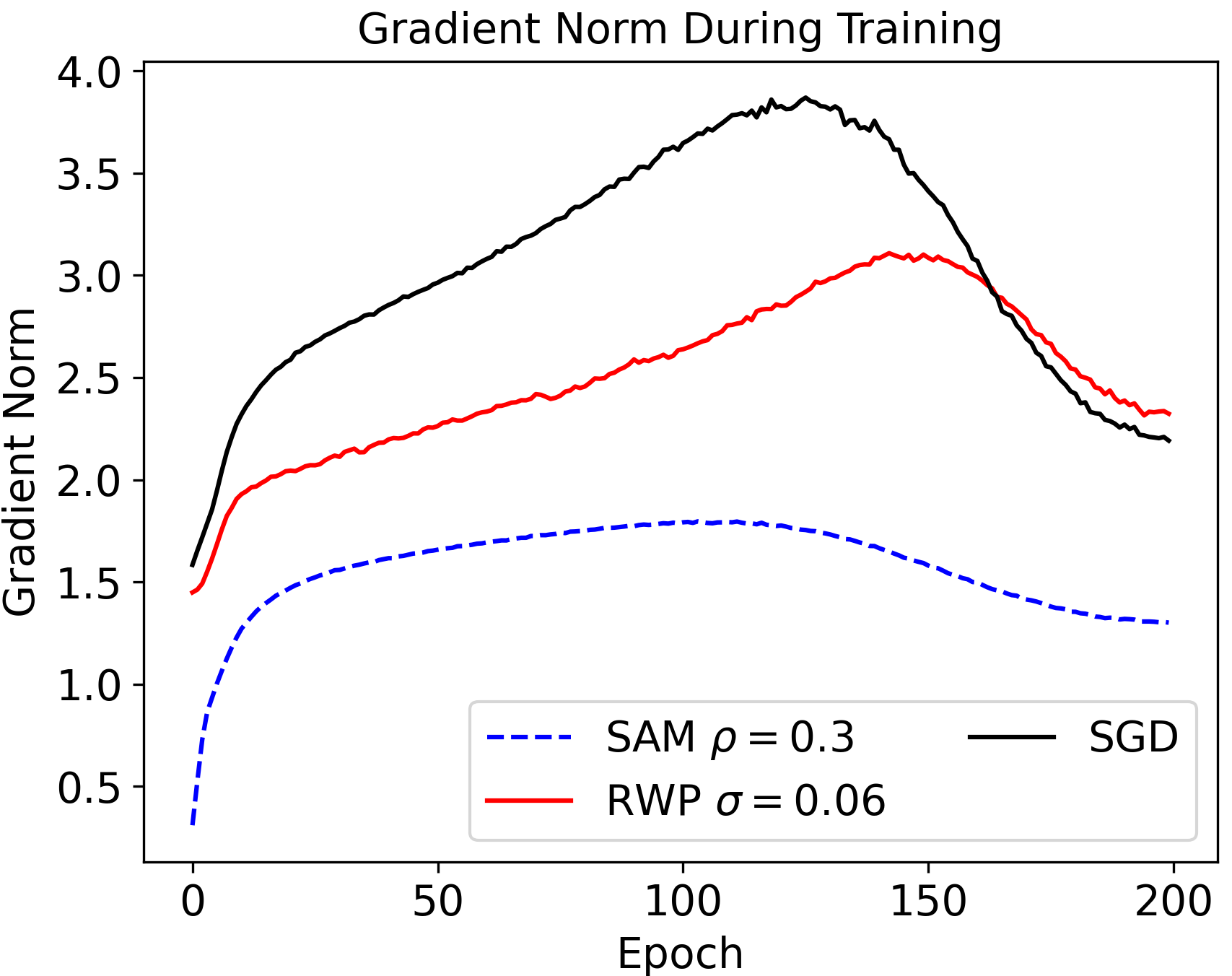}
        \caption{}
        \label{fig:grad_norm_tiny_imagenet}
    \end{subfigure}%
    \begin{subfigure}{.5\textwidth}
      \centering
      \includegraphics[width=1.0\linewidth]{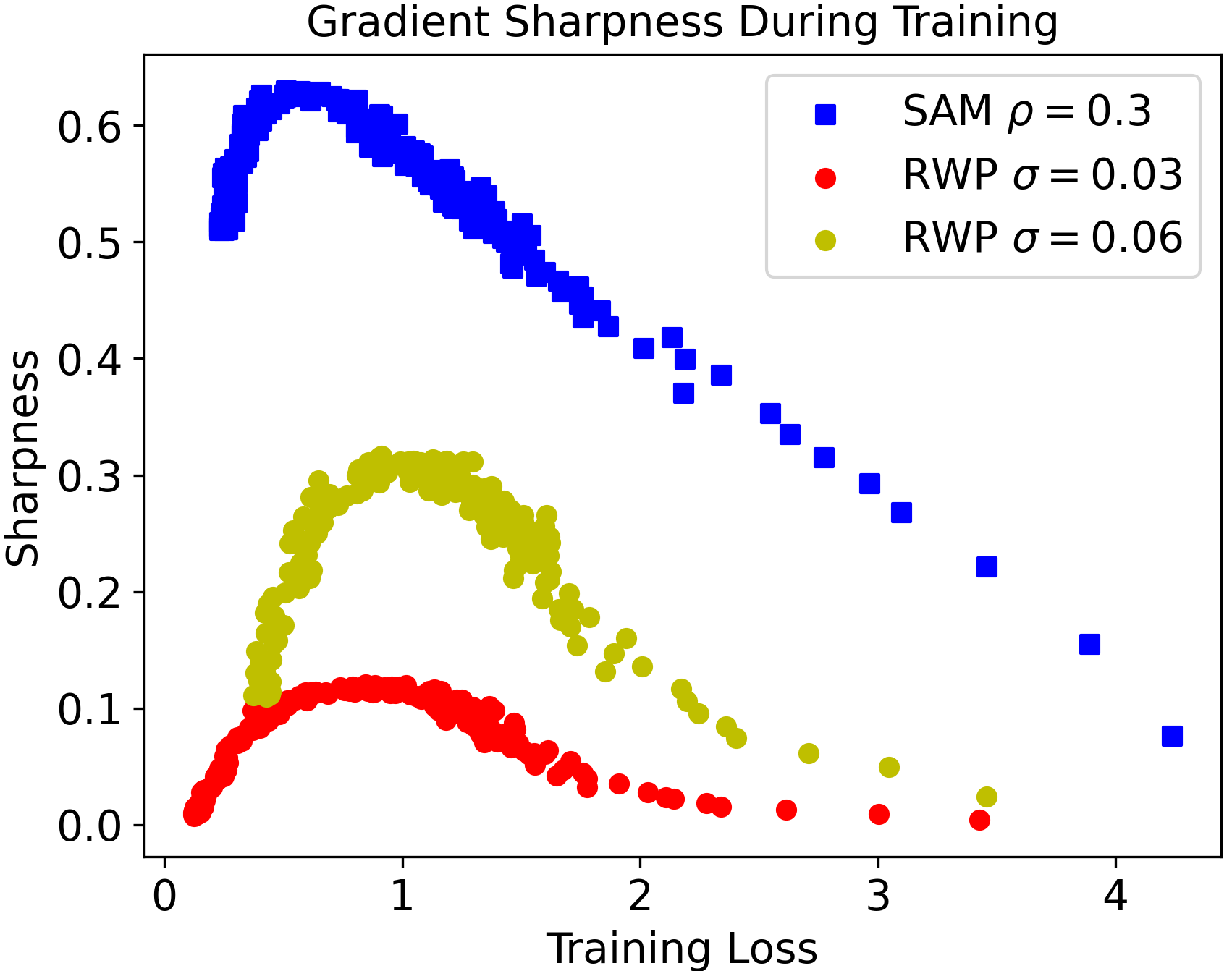}
      \caption{}
      \label{fig:grad_sharpness_tiny_imagenet}
    \end{subfigure}
    \caption{Plots of (a) gradient norm and (b) gradient sharpness for ResNet-18 during Tiny-Imagenet training.}
    \label{fig:tiny_imagenet_comp}
\end{figure}

\begin{figure}
    \begin{subfigure}{.5\textwidth}
      \centering
      \includegraphics[width=0.9\linewidth]{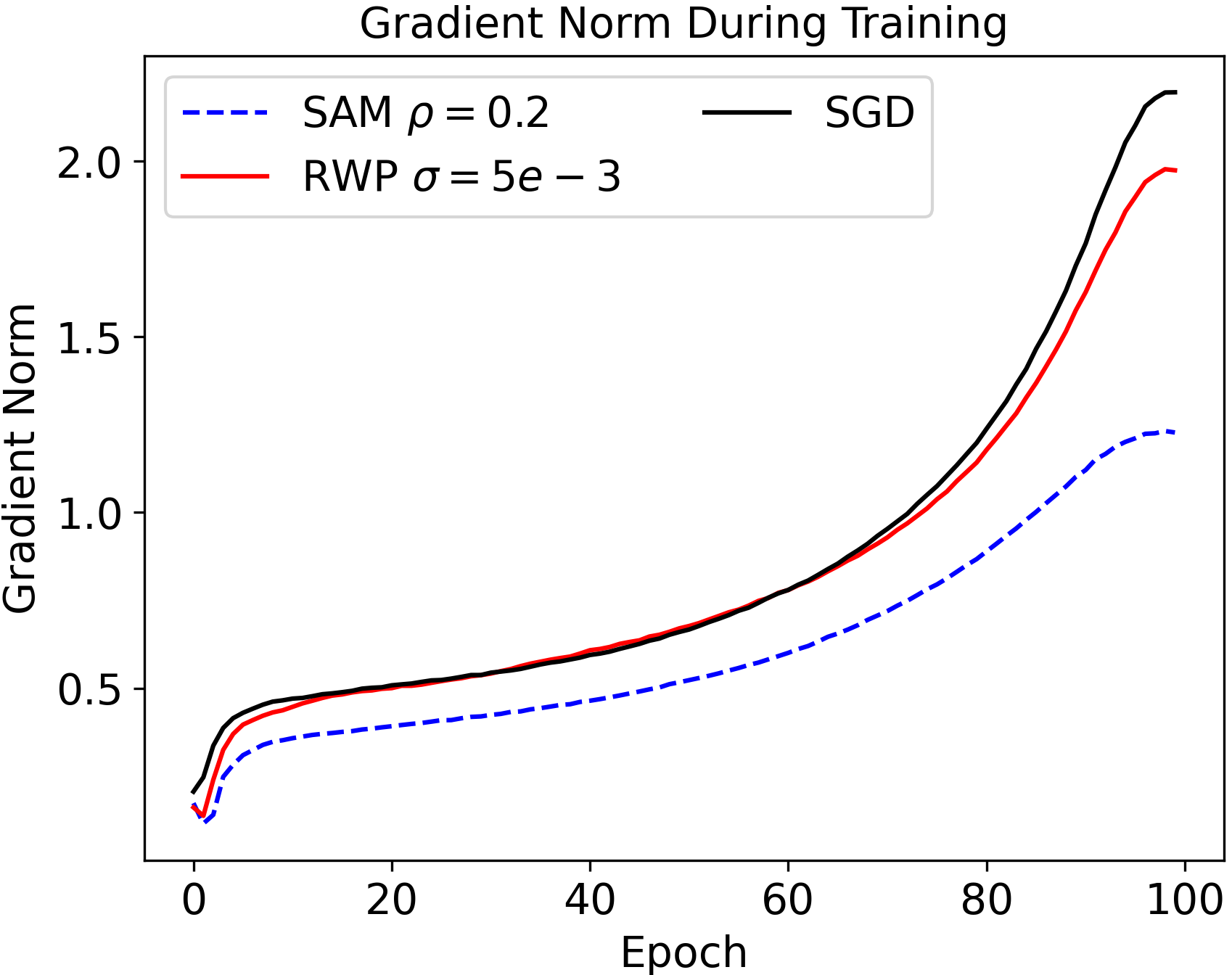}
      \caption{}
      \label{fig:imagenet_grad_norm}
    \end{subfigure}%
    \begin{subfigure}{.5\textwidth}
      \centering
      \includegraphics[width=0.9\linewidth]{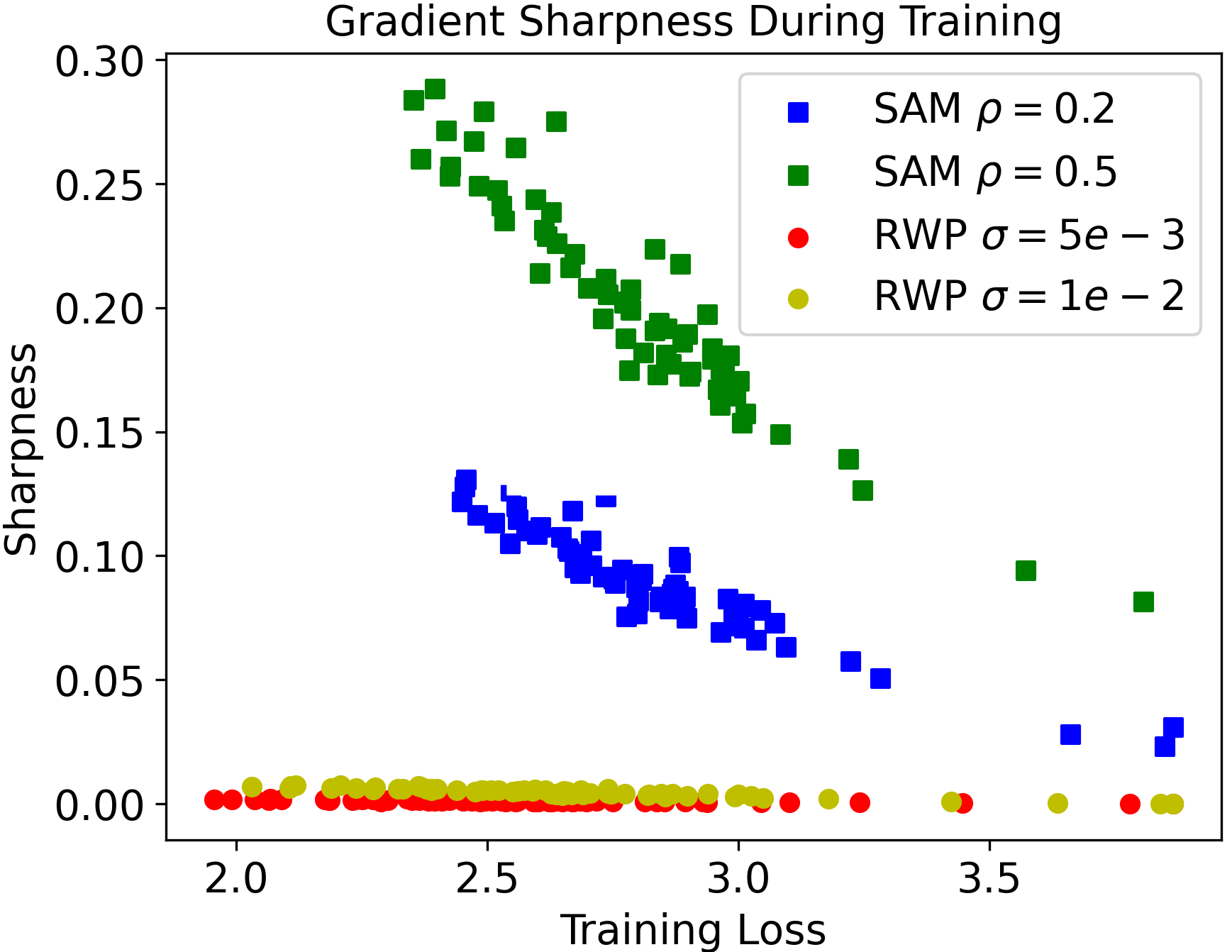}
      \caption{}
      \label{fig:imagenet_sharpness}
    \end{subfigure}
    \caption{Plots of (a) gradient norm and (b) gradient sharpness for ResNet-18 during ImageNet-100 training.}
    \label{fig:imagenet_100_comp}
\end{figure}

\begin{table}
\centering
\caption{Comparison of training and test accuracies on Tiny-ImagNet for various RWP and SAM models.}
\label{table:tiny_imagenet_generalization}
\resizebox{\textwidth}{!}{\begin{tabular}{c|cc|cc}
\hline
Setting
  & \multicolumn{2}{c|}{$\sigma=0.02$}
  & \multicolumn{2}{c}{$\sigma=0.03$} \\
 & Train Acc. & Test Acc.
 & Train Acc. & Test Acc.\\
 \hline
 $\sigma_{train}=0.02$   & $\mathbf{99.85}\pm0.03\pm0.02$ & $63.27\pm0.29\pm0.15$
   & $98.72\pm0.39\pm0.14$ & $59.70\pm0.52\pm0.19$\\
  $\sigma_{train}=0.03$  & $99.80\pm0.04\pm0.01$ & $63.52\pm0.34\pm0.28$
   & $98.61\pm0.41\pm0.30$ & $60.14\pm0.68\pm0.42$\\
  $\sigma_{train}=0.04$   & $99.71\pm0.05\pm0.02$ & $63.28\pm0.37\pm0.22$
   & $\mathbf{98.78}\pm0.45\pm0.16$ & $60.53\pm0.71\pm0.18$\\
  $\sigma_{train}=0.05$  & $99.21\pm0.07\pm0.09$ & $63.46\pm0.32\pm0.57$
  & $97.74\pm0.46\pm0.30$ & $60.54\pm0.62\pm0.43$\\
  $\sigma_{train}=0.06$  & $96.56\pm0.18\pm0.32$ & $64.09\pm0.36\pm0.13$
  & $94.23\pm0.65\pm0.52$ & $61.42\pm0.63\pm0.16$\\
  $\sigma_{train}=0.08$  & $82.11\pm0.43\pm1.76$ & $64.17\pm0.44\pm0.46$
  & $79.59\pm0.84\pm1.96$ & $62.19\pm0.70\pm0.58$\\
  \hline
  $\sigma_{train}=0.08$ ramp  & $88.88\pm0.41\pm0.55$ & $65.73\pm0.40\pm0.18$
  & $86.72\pm0.79\pm0.57$ & $\mathbf{63.82}\pm0.68\pm0.13$\\
  \hline
  $\rho=0.4$  & $89.96\pm3.51\pm4.21$ & $64.89\pm2.55\pm3.27$
  & $69.41\pm10.29\pm18.37$ & $50.63\pm7.31\pm12.01$\\
  $\rho=0.5$  & $81.18\pm4.28\pm6.01$ & $63.38\pm2.98\pm4.33$
  & $56.03\pm11.42\pm23.02$ & $44.95\pm8.19\pm16.80$\\
  \hline
  $\rho=0.4$ ramp & $93.27\pm0.40\pm0.88$ & $67.01\pm1.00\pm0.49$
  & $82.21\pm9.08\pm4.67$ & $58.21\pm5.84\pm3.10$\\
  $\rho=0.5$ ramp & $87.94\pm0.68\pm0.34$ & $\mathbf{67.22}\pm0.59\pm0.39$
  & $77.40\pm5.26\pm5.93$ & $58.88\pm3.62\pm4.39$\\
  \hline
\end{tabular}}
\resizebox{\textwidth}{!}{\begin{tabular}{c|cc|cc}
\hline

  & \multicolumn{2}{c|}{$\sigma=0.04$}
  & \multicolumn{2}{c}{$\sigma=0.05$} \\
 & Train Acc. & Test Acc.
 & Train Acc. & Test Acc.\\
 \hline
 $\sigma_{train}=0.02$   & $91.94\pm1.90\pm0.24$ & $53.54\pm1.21\pm0.31$
  & $72.87\pm4.70\pm0.96$ & $43.65\pm2.32\pm0.65$ \\
  $\sigma_{train}=0.03$  & $92.05\pm1.95\pm1.50$ & $54.44\pm1.24\pm0.76$
  & $74.45\pm4.65\pm3.58$ & $45.28\pm2.31\pm1.45$ \\
  $\sigma_{train}=0.04$   & $\mathbf{94.40}\pm1.85\pm0.50$ & $55.95\pm1.42\pm0.27$
  & $\mathbf{82.04}\pm4.37\pm1.11$ & $48.82\pm2.41\pm0.54$ \\
  $\sigma_{train}=0.05$  & $92.15\pm1.75\pm1.14$ & $55.81\pm1.23\pm0.24$
  & $78.81\pm4.08\pm2.63$ & $48.54\pm2.18\pm0.70$ \\
  $\sigma_{train}=0.06$  & $88.41\pm1.92\pm1.10$ & $57.20\pm1.08\pm0.37$
  & $76.95\pm3.77\pm2.05$ & $50.68\pm1.90\pm0.82$ \\
  $\sigma_{train}=0.08$  & $75.28\pm1.59\pm2.29$ & $59.01\pm1.16\pm0.79$
  & $68.47\pm2.83\pm2.79$ & $54.26\pm1.95\pm1.24$ \\
  \hline
  $\sigma_{train}=0.08$ ramp  & $82.71\pm1.46\pm0.65$ & $\mathbf{60.78}\pm1.13\pm0.21$
  & $75.90\pm2.60\pm0.79$ & $\mathbf{56.25}\pm1.75\pm0.15$\\
  \hline
  $\rho=0.4$  & $36.21\pm14.40\pm20.03$ & $28.22\pm10.12\pm14.47$
  & $12.82\pm10.65\pm8.20$ & $11.17\pm8.38\pm6.89$ \\
  $\rho=0.5$  & $41.69\pm4.44\pm14.84$ & $39.41\pm4.17\pm12.95$
  & $30.95\pm2.67\pm20.43$ & $29.60\pm2.47\pm19.11$\\
  \hline
  $\rho=0.4$ ramp & $60.04\pm2.30\pm1.83$ & $51.70\pm1.84\pm1.28$
  & $51.64\pm4.04\pm3.26$ & $45.48\pm1.82\pm0.93$\\
  $\rho=0.5$ ramp & $59.63\pm2.07\pm0.79$ & $52.24\pm1.57\pm0.49$
  & $52.62\pm3.20\pm0.95$ & $46.77\pm2.49\pm0.57$\\
  \hline
\end{tabular}}
\end{table}

\section{Dynamic Schedule Details}
\label{sec:schedules}


\subsection{Equations Describing Perturbation Schedules}
The exact equations governing the evolution of the perturbation strenght for a quadratic and linear schedule, respectively, are shown below:
\begin{align}
    \sigma_{t,quad}^2 = \sigma_{\max}^2\,\bigl(\tfrac{\min (t,T^*)}{T^*}\bigr)^{2}\\
    \sigma_{t,linear}^2 = \sigma_{\max}^2\,\bigl(\tfrac{\min (t,T^*)}{T^*}\bigr)
\end{align}
where $T^*$ is the total iterations in the warmup period, and $\sigma_{max}$  is the maximal achieved value of the perturbation strength.

\subsection{Finding the Optimal Perturbation Schedule}
To determine the optimal schedule hyperparameters (namely the maximal perturbation strength and number of warm-up iterations), we perform a grid search, shown in Tab. \ref{table:noise_schedule_rwp_sigma} and Tab. \ref{table:noise_schedule_sam_rho} for RWP and SAM, respectively. For both SAM and RWP, we find that the optimal perturbation strength is increased relative to the non-perturbed optimization. Additionally, we find that in the case of SAM, an aggressive ramping schedule is preferred, reaching terminal perturbation strength within the first 25\% of training iterations. This is consistent with the fact that SAM's optimal noise-robustness is achieved long before model convergence, and thereby requiring an accelerated schedule. In the case of RWP, optimal noise-robustness is always achieved at convergence- hence, the optimal ramp length is much longer (around 55\% of total training iterations).

\begin{table}
  \caption{Comparison of ResNet-18 models on $\sigma_{test}=0.07$ trained using SAM on Cifar-100 for varying perturbation strength $\rho$ when a quadratic ramp schedule is applied.}
  \centering
  \label{table:noise_schedule_sam_rho}
  \begin{tabular}{c|c|c|c}
    \hline
     &  $iter = 15000$  &   $iter = 20000$ &  $iter = 25000$ \\
    \hline
    $\rho_{max}=0.8$ & $ 57.64 \pm  3.01 \pm 1.80$  & $59.71 \pm 1.93 \pm 1.66$& $ 58.35 \pm 3.81 \pm 0.74$ \\
    $\rho_{max}=1.0$ & $ 57.54 \pm 2.34 \pm 0.52$ & $\mathbf{60.24} \pm 1.98 \pm 0.83$ & $ 59.13 \pm  2.40 \pm 0.68$  \\
    $\rho_{max}=1.2$ & $57.07 \pm 1.92 \pm 0.33$ & $58.83 \pm 1.78 \pm 1.14$ & $56.97 \pm  1.92 \pm 1.87$  \\
    \hline
  \end{tabular}
\end{table}

\begin{table}
  \caption{Comparison of ResNet-18 models on $\sigma_{test}=0.07$ trained using RWP on Cifar-100 for varying perturbation strength $\sigma_{train}$ when a quadratic ramp schedule is applied.}
  \centering
  \label{table:noise_schedule_rwp_sigma}
  \begin{tabular}{c|c|c|c}
    \hline
     &  $iter = 35000$  &   $iter = 45000$ &  $iter = 55000$ \\
    \hline
    $\sigma_{max}=0.08$ & $ 61.74 \pm 3.08 \pm 0.52$ & $63.32 \pm 2.25 \pm 0.74$ &  $62.55 \pm 2.71 \pm 1.61$ \\
    $\sigma_{max}=0.1$ & $ 63.03 \pm 2.41 \pm 0.93 $  & $ \mathbf{65.57} \pm 2.05 \pm 0.64$& $ 64.93 \pm 1.96 \pm 1.00$ \\
    $\sigma_{max}=0.12$ & $ 59.37 \pm 2.16 \pm 1.89$ & $62.50 \pm 1.87 \pm 0.96$ & $63.41 \pm 2.25 \pm 1.78$  \\
    \hline
  \end{tabular}
\end{table}

\subsection{Visualizing Ramped Perturbation Sharpness}

To understand the effect perturbation schedules have on sharpness evolution during training, we plot both gradient sharpness and average-direction sharpness of the quadratic schedule SAM and RWP, comparing against the constant-perturbation versions (Fig. \ref{fig:sharpness_schedule}). Here, we observe a coherent trend: in the early stages of training, the use of the ramp schedules reduces gradient sharpness significantly, after which the growing perturbations cause the gradient sharpness to surge. In the case of average-direction sharpness, shown in Fig. \ref{fig:avg_sharpness_schedule}, we see that while training using the schedules increases the training loss at convergence, sharpness is reduced compared to the constant perturbation baselines. This is especially noticeable in the case of SAM: we recall from Fig. \ref{fig:sharpness_07} that increasing constant-perturbation $\rho$ from 0.2 to 0.5 \textit{has no effect on reducing average-direction sharpness}. However, when using the quadratic perturbation schedule, a drastic reduction in sharpness can now be achieved through large perturbations.

\begin{figure}
    \begin{subfigure}{.5\textwidth}
        \centering
        \includegraphics[width=1.0\linewidth]{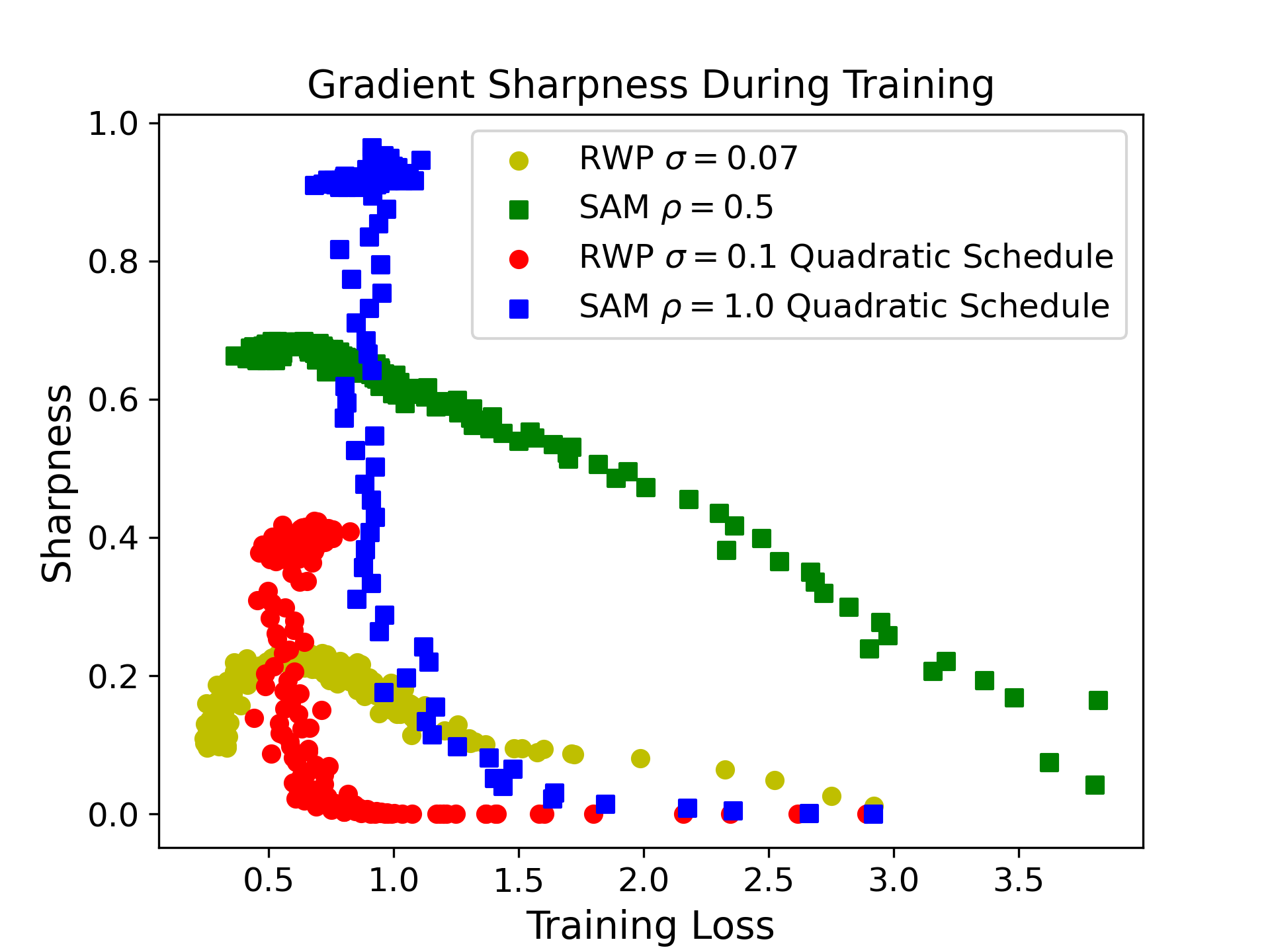}
        \caption{}
        \label{fig:grad_sharpness_schedule}
    \end{subfigure}%
    \begin{subfigure}{.5\textwidth}
      \centering
      \includegraphics[width=1.0\linewidth]{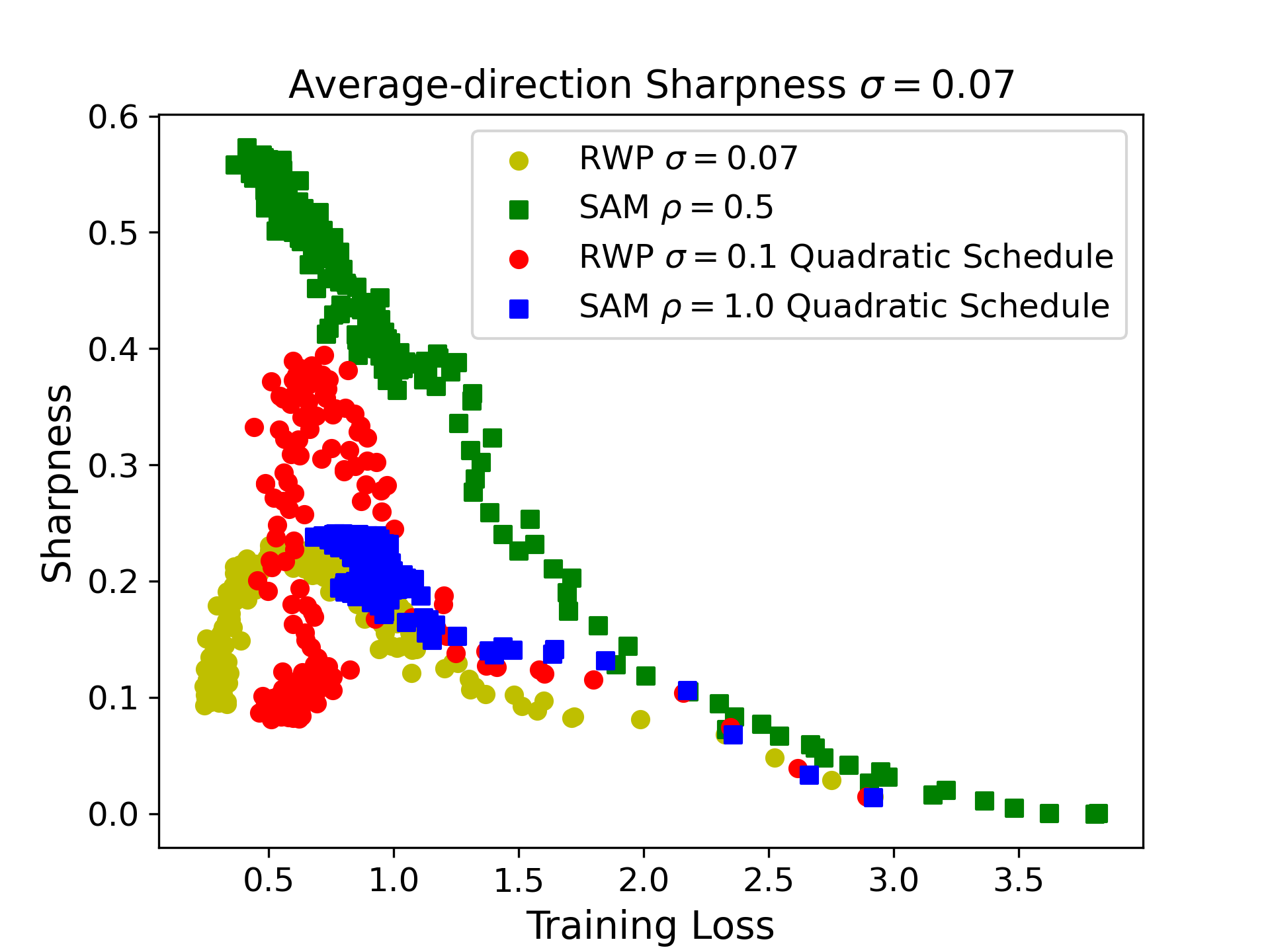}
      \caption{}
      \label{fig:avg_sharpness_schedule}
    \end{subfigure}
    \caption{Plots of (a) gradient sharpness and  (b) ascent-direction sharpness as a function of Cifar-100 training loss for a ResNet-18 trained using quadratic schedule SAM and RWP.}
    \label{fig:sharpness_schedule}
\end{figure}

\subsection{Alternative View of Landscape Broadening}

To provide another view of how the dynamic perturbation schedule affects the training dynamics, we plot the perturbed training loss as a function of epoch comparing the quadratically-ramped/constant schedules with equivalent terminal perturbation strengths for both RWP and SAM, shown in Fig. \ref{fig:perturbed_loss_ramp}. In these plots, we observe that for the dynamic perturbation schedules, even after the initial warm-up period (meaning the size of the perturbation has reached the terminal value), the perturbed training loss is \textit{lower} than that of the constant schedule. This difference is particularly stark for SAM: at the final epoch of the ramp, there is \~0.75 difference in perturbed loss between the two trajectories. This clearly illustrates the contrast in evolution of the two trajectories, as the dynamic schedules guide the optimization towards a flatter basin than would otherwise be encountered.

\begin{figure}
    \begin{subfigure}{.5\textwidth}
      \centering
      \includegraphics[width=1.0\linewidth]{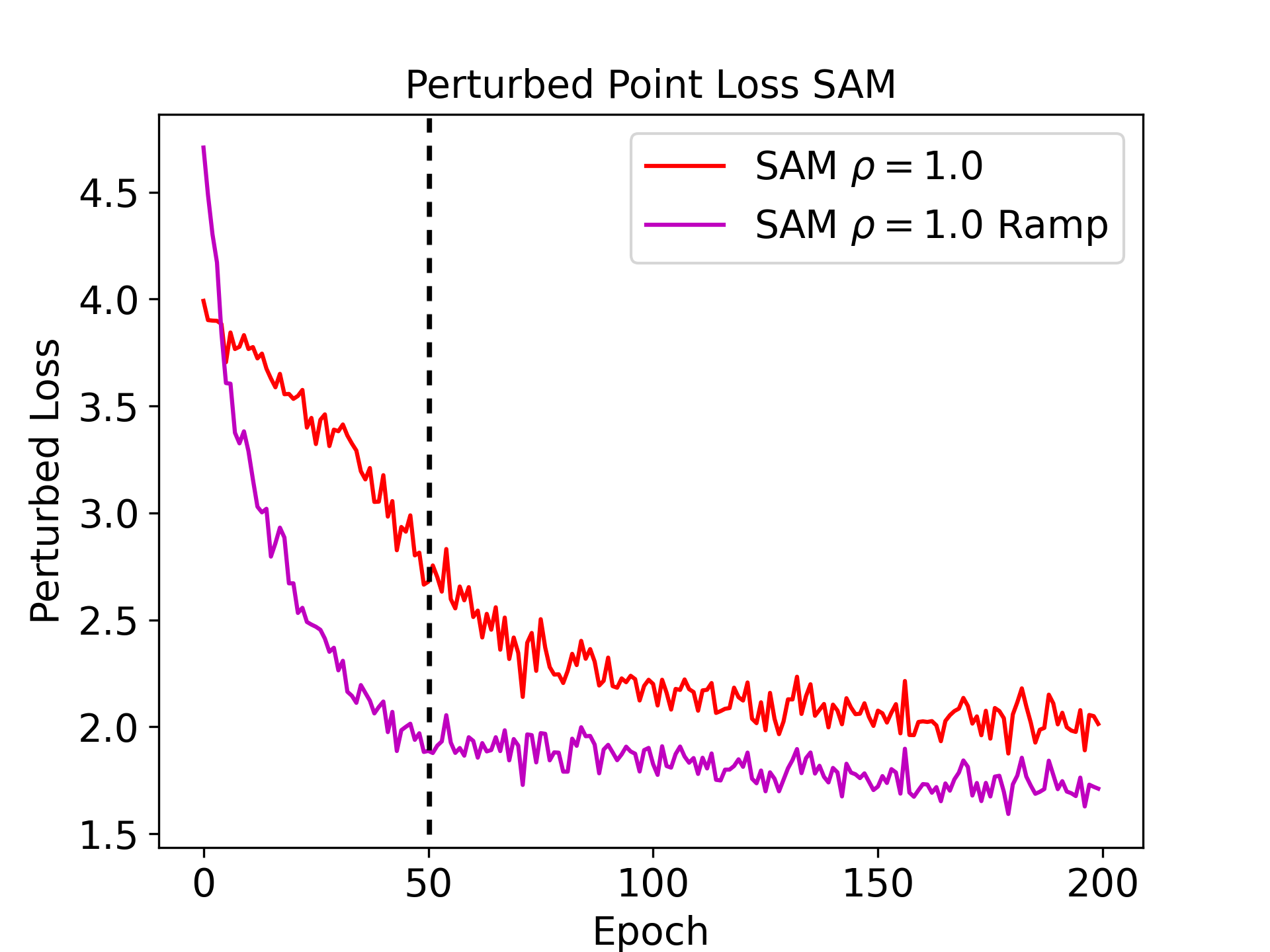}
      \caption{}
      \label{fig:perturbed_loss_ramp_sam}
    \end{subfigure}%
    \begin{subfigure}{.5\textwidth}
      \centering
      \includegraphics[width=1.0\linewidth]{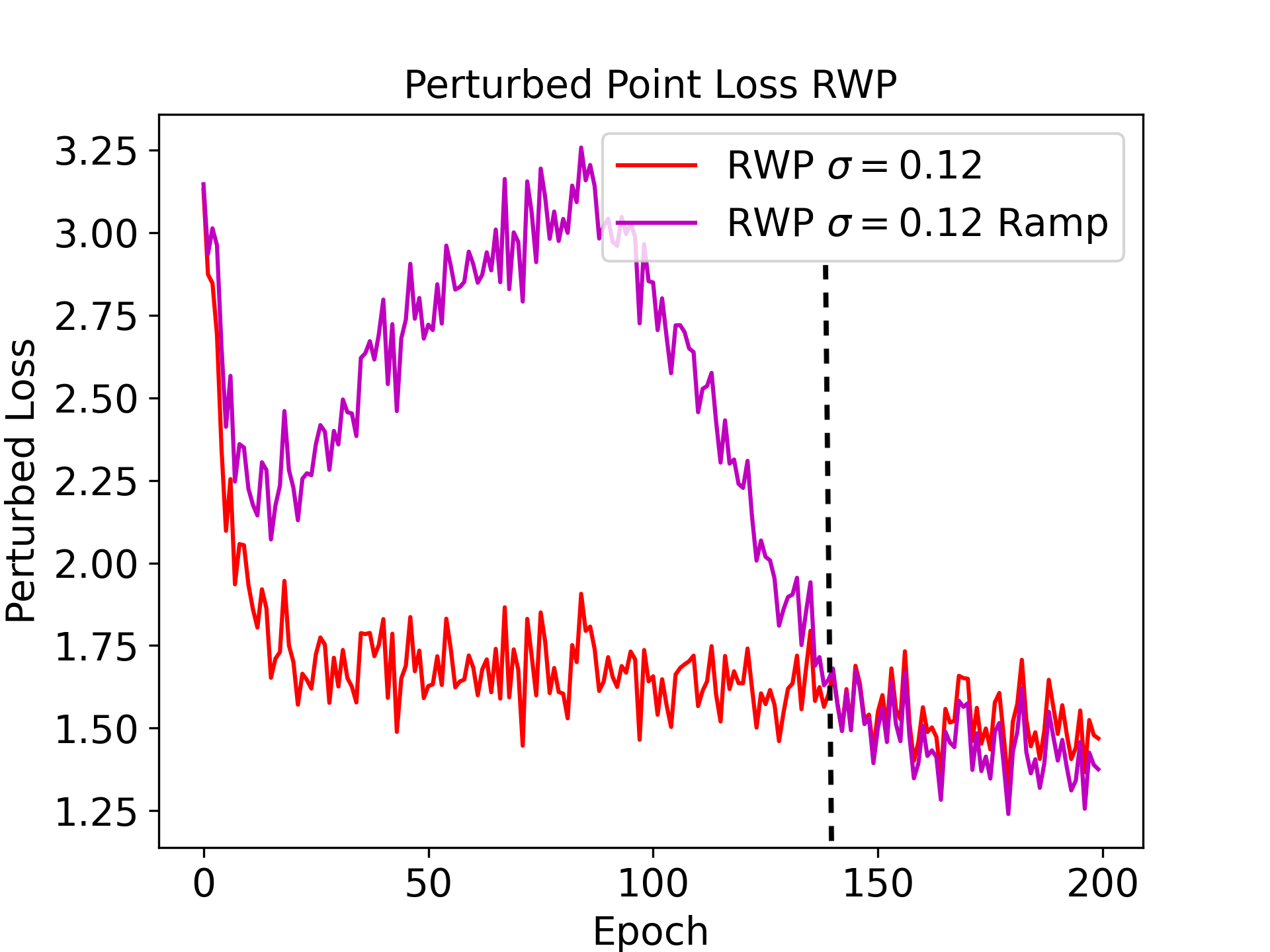}
      \caption{}
      \label{fig:perturbed_loss_ramp_rwp}
    \end{subfigure}
    \caption{Comparison of the perturbed loss $L(w_p)$ for ramped/constant schedules of (a) SAM and (b) RWP on Cifar-100. The perturbed loss is calculated at a constant perturbation distance (equal to the maximum perturbation), including for the dynamic schedules. Dotted black line denotes epoch at which ramped perturbation strength reaches strength equal to constant perturbation.}
    \label{fig:perturbed_loss_ramp}
\end{figure}

\end{document}